\documentclass[10pt]{article}
\usepackage{balance}
\usepackage[top=2.5cm, bottom=2.9cm, left=2.5cm,
right=2.5cm]{geometry}
\usepackage{dblfloatfix}
\usepackage[export]{adjustbox} 
\usepackage{microtype}
\usepackage{graphicx} 
\usepackage{pgfplots}
\pgfplotsset{compat=1.18}
\usepackage{tikz}
\usetikzlibrary{calc}
\usepgfplotslibrary{groupplots}

\usepackage{subcaption}
\usepackage{booktabs} 
\usepackage{times}  
\usepackage{helvet}  
\usepackage{courier}  
\usepackage[hyphens]{url}  
\usepackage{graphicx} 
\usepackage{natbib}  
\usepackage{caption} 
\usepackage{algorithm}
\usepackage{algorithmic}
\usepackage{bm}
\usepackage{enumitem,amsthm,amsmath,amsfonts,amssymb}
\usepackage{indentfirst}
\usepackage{float}
\usepackage{mathtools}
\usepackage{algorithmic}
\usepackage{algorithm}
\usepackage{xcolor}
\usepackage{graphicx}
\usepackage{subcaption}
\usepackage{tikz}
\usepackage{multirow}
\usepackage{multicol}
\usepackage{soul}
\usepackage{mathrsfs, fancyhdr}
\usepackage{booktabs}

\newcommand{\Z}{\mathbb{Z}}

\newcommand{\N}{\mathbb{N}}

\newcommand{\D}{\mathcal{D}}

\newcommand{\bA}{\mathbf{A}}
\newcommand{\E}{\mathbb{E}}

\newcommand{\R}{\mathbb{R}}
\newcommand{\F}{\mathcal{F}}

\newtheorem{cor}{Corollary}

\newtheorem{prop}{Proposition}
\newtheorem{rem}{Remark}

\def\bw{\mathbf{w}}
\def\O{\mathcal{O}}

\def\X{\mathcal{X}}
\def\Y{\mathcal{Y}}
\def\W{\mathcal{W}}

\def\Z{\mathcal{Z}}
\def\N{\mathcal{N}}
\def\L{\mathcal{L}}
\def\bx{\mathbf{x}}
\def\bz{\mathbf{z}}

\def\S{\mathcal{S}}
\def\bw{\mathbf{w}}
\def\bW{\mathbf{W}}
\def\bu{\mathbf{u}}

\def\A{\mathcal{A}}

\def\ba{\mathbf{a}}

\def\bu{\mathbf{u}}

\def\0{\mathbf{0}}

\def\bX{\mathbf{X}}

\def\by{\mathbf{y}}

\newcommand{\bSigma}{\mathbf{\Sigma}}
 
\newtheorem{theorem}{Theorem}
\newtheorem{lemma}[theorem]{Lemma}
\newtheorem{proposition}[theorem]{Proposition}
\newtheorem{corollary}[theorem]{Corollary}
\theoremstyle{definition}
\newtheorem{definition}{Definition}

\newtheorem{assumption}{Assumption}
\newtheorem{example}{Example}
\theoremstyle{definition}

\newtheorem{remark}{Remark}

\def\begeqn{\begin{equation}}
\def\endeqn{\end{equation}}
\def\begth{\begin{theorem}}
\def\endth{\end{theorem}}
\def\begprop{\begin{proposition}}
\def\endprop{\end{proposition}}
\def\begcor{\begin{corollary}}
\def\endcor{\end{corollary}}
\def\begdef{\begin{definition}}
\def\enddef{\end{definition}}
\def\beglemm{\begin{lemma}}
\def\endlemm{\end{lemma}}
\def\begexm{\begin{example}}
\def\endexm{\end{example}}
\def\begrem{\begin{remark}}
\def\endrem{\end{remark}}

\def\begdef{\begin{definition}}
\def\enddef{\end{definition}}

\def\bw{\mathbf{w}}
\def\bz{z}

\def\O{\mathcal{O}}

\def\A{\mathcal{A}}
\def\Z{\mathcal{Z}}
\def\R{\mathbb{R}}

\def\X{\mathcal{X}}
\def\Y{\mathcal{Y}}
\def\Z{\mathcal{Z}}
\def\W{\mathcal{W}}

\def\bD{\mathbf{D}}
\def\bV{\mathbf{V}}
\def\bfI{\mathbf{I}}

\def\ebb{\mathbb{E}}

\def\bK{\mathbf{K}}
\def\bG{\mathbf{G}}
\def\bg{\mathbf{g}}
\def\H{\mathcal{H}}

\def\bS{\mathbf{S}}

\def\bL{\mathbf{L}}
\def\tildeW{\widetilde{\bW}}
 \def\tildeo{\tilde{o}}
\def\tildeD{\widetilde{\bD}}
\def\tildeV{\widetilde{\bV}}
\usepackage{hyperref}
\usepackage{bm}
\usepackage{enumitem,amsthm,amsmath,amsfonts,amssymb}
\usepackage{indentfirst}
\usepackage{float}
\usepackage{natbib}
\usepackage{mathtools}
\usepackage{algorithmic}
\usepackage{algorithm}
\usepackage{xcolor}
\usepackage{graphicx}
\usepackage{subcaption}
\usepackage{tikz}
\usepackage{multirow}
\usepackage{multicol}
\usepackage{hyperref}
\usepackage{soul}
\usepackage{mathrsfs, fancyhdr}
\usepackage{amscd}
\usepackage{booktabs}

\def\bw{\mathbf{w}}
\def\O{\mathcal{O}}

\def\X{\mathcal{X}}
\def\Y{\mathcal{Y}}
\def\W{\mathcal{W}}

\def\bA{\mathbf{A}}
\def\Z{\mathcal{Z}}
\def\N{\mathcal{N}}
\def\L{\mathcal{L}}
\def\bx{\mathbf{x}}
\def\bz{\mathbf{z}}

\def\S{\mathcal{S}}
\def\bw{\mathbf{w}}
\def\bW{\mathbf{W}}
\def\bu{\mathbf{u}}

\def\A{\mathcal{A}}

\def\ba{\mathbf{a}}

\def\bu{\mathbf{u}}

\def\0{\mathbf{0}}

\def\bX{\mathbf{X}}

\def\by{\mathbf{y}}

\def\bG{\mathbf{G}}

\def\bw{\mathbf{w}}
\def\bz{z}

\def\O{\mathcal{O}}

\def\A{\mathcal{A}}
\def\Z{\mathcal{Z}}
\def\R{\mathbb{R}}

\def\X{\mathcal{X}}

\def\Y{\mathcal{Y}}
\def\Z{\mathcal{Z}}
\def\W{\mathcal{W}}

\def\bD{\mathbf{D}}
\def\bV{\mathbf{V}}
\def\bfI{\mathbf{I}}

\def\ebb{\mathbb{E}}

\def\blue{\textcolor{blue}}

\def\bK{\mathbf{K}}

\def\H{\mathcal{H}}

\def\bS{\mathbf{S}}

\def\bL{\mathbf{L}}

\usepackage{amsmath}
\usepackage{amssymb}
\usepackage{mathtools}
\usepackage{amsthm}
\usepackage{multirow}

\usepackage[capitalize,noabbrev]{cleveref}

\usepackage[textsize=tiny]{todonotes}

\begin{document}

\title{Minimax-Optimal Generalization Bounds for Smooth Deep Neural Networks Trained by (Stochastic) Gradient Descent\footnote{Corresponding author is Puyu Wang.}}
\author{\!Junyu Zhou$^1$\quad Puyu Wang$^{2*}$\quad Dennis Wagner$^{2}$\quad Yunwen Lei$^3$ \quad Marius Kloft$^2$ \quad Yiming Ying$^4$ \\ 
\smallskip \\
$^{1}$ Mathematical Institute for Machine Learning and Data Science, \\ Catholic University of Eichstätt-Ingolstadt\\ 
$^{2}$ Department of Computer Science, RPTU Kaiserslautern-Landau \\
$^{3}$ Department of Mathematics, The University of Hong Kong\\
$^{4}$ School of Mathematics and Statistics, The University of Sydney}

 \date{}

\maketitle

\begin{abstract}
Characterizing the optimization dynamics and statistical performance of over-parameterized deep neural networks (DNNs) remains a central challenge in understanding the remarkable success of deep learning. We establish quantitative bounds showing that kernel gradient descent in the reproducing kernel Hilbert space induced by the deterministic infinite-width neural tangent kernel approximates finite-width deep regression with smooth activations under gradient descent (GD) and stochastic gradient descent (SGD) training. The approximation gap is governed by the network width and training horizon, with an additional stochastic gradient error in the SGD case. This connection provides a general mechanism for transferring learning-theoretic guarantees from kernel methods to deep regression. As an application, under general source and effective dimension conditions, we show that both GD- and SGD-trained DNNs attain the minimax-optimal excess population risk rate, up to logarithmic factors, provided that the network width grows polynomially in the sample size. To the best of our knowledge, these are the first such guarantees for standard fully connected deep neural networks with smooth activations trained by GD and SGD. 
\end{abstract}

\bigskip

\parindent=0cm

\section{Introduction}
Over-parameterized deep neural networks (DNNs) trained by gradient descent (GD) or stochastic gradient descent (SGD) often achieve nearly zero training error while maintaining strong performance on unseen data, despite the highly nonconvex nature of their training objectives. Explaining this phenomenon requires understanding the learning process induced by gradient methods and how it shapes the statistical performance of the resulting predictors.
This has motivated a substantial body of work on the optimization and generalization of neural networks
\cite{bartlett2017spectrally,du2018gradient,du2019gradient,allen2019convergence,cao2020generalization,
taheri2025sharper,richards2021stability,zou2018stochastic}.

The neural tangent kernel (NTK), introduced by \cite{jacot2018neural}, provides a natural approach to understanding this learning process.
In the infinite-width limit, the evolution of a neural network under gradient-based training can be described by a kernel method in the reproducing kernel Hilbert space (RKHS) induced by the NTK.
This suggests that the behavior of a finite-width network may be understood through its deterministic infinite-width kernel counterpart.
A fundamental question therefore arises:
\begin{center}
\textit{Can a finite-width DNN with smooth activations trained by GD or SGD be quantitatively approximated by kernel GD induced by its deterministic infinite-width NTK?}
\end{center}

Beyond describing the learning dynamics, such a connection would provide a natural mechanism for transferring statistical guarantees from kernel methods to DNNs.
A representative example is the minimax-optimal excess population risk rate attained by kernel methods under standard source and effective dimension conditions.
It is therefore natural to ask whether smooth DNNs trained by GD or SGD can inherit this statistical performance from their kernel counterparts.

For shallow neural networks with smooth activations, minimax-optimal rates for both GD and SGD have been established under various source and capacity conditions, albeit with different network width requirements \citep{nitanda2021optimal,cao2024stochastic,nguyen2024many}.
These analyses proceed through the random RKHS induced by the finite-width NTK and require controlling its effective dimension. While this is possible in the shallow setting by exploiting the independent random feature structure, the argument does not extend directly to DNNs, whose tangent features exhibit intricate dependencies across layers.
A quantitative connection directly to the deterministic infinite-width kernel, if established, could bypass this obstacle and provide a route for transferring sharp kernel guarantees to deep networks.

For deep networks with smooth activations,
\cite{kohler2026rate,kohler2026ratesgd} obtained nearly minimax-optimal rates under different assumptions. Their estimators, however, are linear combinations of an increasing number of fixed-size fully connected subnetworks, rather than a single standard DNN. This raises a second fundamental question:
\begin{center}
\textit{Can a single fully connected DNN trained by GD or SGD attain the minimax-optimal excess population risk rate?}
\end{center}
In this work, we answer both questions affirmatively. Our main contributions are summarized below.
\begin{itemize}

\item \textbf{Kernel GD approximates finite-width deep regression under (S)GD training.}
We establish a quantitative connection between finite-width DNNs with smooth activations and kernel GD induced by the deterministic infinite-width NTK.
More precisely, after \(T\) iterations, the predictor produced by a DNN trained by (S)GD is approximated by the corresponding kernel GD predictor.
The approximation gap depends polynomially on the inverse width \(m^{-1}\) and the effective training horizon \(\eta T\), with an additional \(\widetilde\O(\eta)\) error due to stochastic gradients in the SGD case, where \(\eta\) denotes the step size. 
This connection enables the transfer of learning-theoretic guarantees from kernel methods to finite-width DNNs.

\item \textbf{Minimax-optimal generalization for DNNs.}
As a concrete application, we show that, under general source and effective dimension conditions, both GD- and SGD-trained DNNs attain the minimax-optimal excess population risk rate $\O\big(n^{-\frac{2\beta}{2\beta+\gamma}}\big)$ when $2\beta +\gamma >1$,    provided that the network width grows polynomially in the sample size $n$.  
Here, $\beta>0$ is the smoothness of the target function and $\gamma\in[0,1]$ measures the capacity of the RKHS. 
To the best of our knowledge, these are the first such guarantees for a single standard fully connected DNN with smooth activations trained by GD or SGD.

\item \textbf{Key technical advances.}
First, we prove uniform concentration of the finite-width deep NTK around its deterministic infinite-width limit, simultaneously over all layers and input pairs, with width dependence \( \mathcal O (m^{-1/2})\). This improves the \(m^{-1/6}\)-type dependence previously obtained for deep ReLU networks \cite{xu2024overparametrized} and serves as a key ingredient in establishing our quantitative connection between DNNs and infinite-width kernel GD.
Second, we establish a stability bound that controls the discrepancy between kernel GD predictors in different RKHSs through the uniform difference between their kernels.
This direct finite-width to infinite-width comparison bypasses the effective dimension analysis of the random finite-width RKHS, which is difficult to extend to DNNs because of cross-layer dependencies in their tangent features, and enables minimax guarantees for the limiting RKHS to be transferred to a single standard DNN.
\end{itemize}


\section{Related Work}
There exists a substantial body of literature dedicated to understanding the generalization ability of neural networks. We briefly review the most relevant work to ours.

\paragraph{NTK-Based Regression with Smooth Activations.}
Sharp NTK-based generalization guarantees for regression have so far been established mainly for shallow neural networks. 
For shallow networks with smooth activations, \cite{nitanda2021optimal} and \cite{nguyen2024many} removed spectral assumptions and established minimax-optimal excess population risk rates $\O\big(n^{-\frac{2\beta}{2\beta+\gamma}}\big)$ under general source and effective dimension conditions for one-pass SGD and GD, respectively.
\cite{cao2024stochastic} further reduced the required network width in the capacity-independent setting.
These smooth network analyses proceed through the random RKHS induced by the finite-width NTK and require controlling its effective dimension.
However, these analyses do not extend directly to DNNs due to the intricate cross-layer dependencies in their tangent features.

\paragraph{Uniform Convergence Approach.}
Another major line of work studies neural network generalization through uniform convergence. This approach controls the capacity of the hypothesis class using complexity measures such as Rademacher complexity and covering numbers, leading to algorithm-independent generalization bounds \cite{bartlett2017spectrally,frei2023random,golowich2018size, neyshabur2015norm,parhi2022near}.
Building on this framework, recent works established generalization guarantees for shallow and deep neural networks trained by GD \citep{braun2024convergence,drews2023analysis,drews2024universal, kohler2025rate}.
In particular, \cite{kohler2026rate,kohler2026ratesgd} obtained the nearly minimax-optimal rate $\mathcal{O}(n^{-\frac{2p}{2p+d}+\epsilon})$ for any $\epsilon > 0$, under the assumption that the target function is \((p,C)\)-smooth. Their estimators are a linear combination of \(K_n\) deep networks with smooth activations, where \(K_n\) depends on \(n\). In contrast, we analyze a single standard fully connected DNN.

\paragraph{Algorithmic Stability Approach.}
Algorithmic stability has recently emerged as a useful tool for studying the generalization of neural networks trained by gradient methods.
For regression, \cite{richards2021stability,lei2022stability, wang2025generalization} established excess risk bounds of order \(\mathcal O(n^{-1/2})\) for shallow networks with smooth activations trained by GD or SGD.
Their analyses exploit the smoothness and weak convexity of the empirical objective, with the weak-convexity parameter decreasing as the network width grows. More recently, \cite{taheri2025sharper} established stability-based generalization guarantees for deep neural networks with smooth activations in classification settings.

\section{Problem Settings}
We begin by introducing the problem setup. 
Let $\mathcal X$ and $\mathcal Y$ denote the input and output spaces, respectively, and let  $\Z=\X\times \Y$ be the sample space. 
Throughout the paper, we take $\mathcal X = \mathbb S^{d-1}$, the unit sphere in $\mathbb R^d$, and $\mathcal Y = [-1,1]$.
Let $\rho$ be an unknown probability distribution on $\mathcal Z$, and $S=\{z_i=(\bx_i,y_i)\}_{i=1}^n$ be a training dataset drawn independently from $\rho$.
Our goal is to learn a predictor $f:\X \rightarrow \R$ based on $S$. The population  risk  of $f$ is defined by $\L(f):=\frac{1}{2}\,\E_{(\bx,y)\sim \rho} [(y -f(\bx))^2 ],$ 
and the corresponding empirical risk is  $\L_S(f):=\frac{1}{2n} \sum_{i=1}^n ( y_i-f(\bx_i)  )^2.$ 
 
The generalization performance of a predictor $f$  is measured  by the \textit{excess population risk}
$\varepsilon_{risk} (f )=\L (f ) - \L(f_\rho),$ where $f_\rho:\X \rightarrow \R$ is a global minimizer of $\mathcal{L}(\cdot)$ over all measurable functions.
For least-squares regression, the regression function satisfies $f_\rho(\bx) = \ebb[y|\bx]$. Moreover, the excess population risk admits the representation
\citep{cucker2007learning}  
\[\varepsilon_{risk} (f ) = \frac{1}{2} \|  f - f_{\rho} \|_\rho^2,\]
where $\|f \|_\rho:=(\int_\X |f(\bx)|^2 d \rho_{\bx} (\bx))^{1/2}$ and $\rho_{\mathbf{x}}$ denotes the marginal distribution of $\rho$ on $\mathcal X$.
  

\subsection{Analyzed Model}
We focus on predictors \(f_{\bW}\) represented by a DNN and parameterized by \(\bW\in\W\). Specifically, we consider an \(L\)-layer fully connected neural network with width \(m\) in every hidden layer:
\begin{align}
\label{eq:NN}
f_{\bW}(\bx)
=
\ba^\top
\sqrt{\frac{c_\sigma}{m}}\,
\sigma\Big(
\bW^L
\cdots
\sqrt{\frac{c_\sigma}{m}}\,
\sigma\left(\bW^1\bx\right)
\Big),
\end{align}
where $\bx\in\X$, the activation function \(\sigma\) is applied componentwise and  $c_\sigma = ( \E_{Z\sim\mathcal N(0,1)}  [\sigma^2(Z) ])^{-1}$ is a normalization constant.
The trainable hidden-layer weights are collected as
\[\bW=\left(\bW^1,\ldots,\bW^L\right)\in\W:=\R^{m\times d}\times (\R^{m\times m} )^{L-1},\]
where $\bW^1\in\R^{m\times d}, \bW^\ell\in\R^{m\times m}$ for $\ell=2,\ldots,L$. Let $(\bw^\ell_r)^\top$ denote the $r$-th row of $\bW^\ell$ for $\ell\in[L]:=\{1,\ldots,L\}$.
The fixed output-layer weights are denoted by $\ba = (a_1,\ldots,a_m)^\top \in\R^m.$  For  simplicity, we assume $m$ is even and use the notations $\L(\bW)=\L(f_\bW)$, $\L_S(\bW)=\L_S(f_\bW)$, and define the loss $ l(\bW;z)=\frac{1}{2}(y-f_\bW(\bx))^2$.  

We introduce the following standard assumption on the activation function. 
\begin{assumption}[Activation function]\label{ass:activation}
The activation function \(\sigma:\R\to\R\) is twice continuously differentiable and satisfies  $|\sigma(a)|\le B_{\sigma}$, $|\sigma'(a)|\le B_{\sigma'}$ and $|\sigma''(a)|\le B_{\sigma''}$, where $B_{\sigma}, B_{\sigma'},B_{\sigma''}>0 .$
\end{assumption}
Common smooth activations, including the sigmoid and hyperbolic tangent activations, satisfy this assumption.  
Throughout the paper, \(C_\sigma\geq 1\) denotes a constant depending only on \(c_\sigma\), \(B_\sigma\), \(B_{\sigma'}\), and \(B_{\sigma''}\), while \(C\geq 1\) denotes an absolute constant.
Their values may vary from line to line.

\subsection{Learning Algorithms}\label{subsec:learning}
We study GD and SGD for minimizing the empirical risk.
For both algorithms, we adopt the commonly used symmetric initialization \citep{nitanda2021optimal,damian2022neural,
nguyen2024many,xu2024overparametrized}.

For the first \(L-1\) hidden layers, the weights are initialized independently as $\bw_r^1(0)  \overset{\text{i.i.d.}}{\sim} \mathcal{N}(0, \mathbf{I}_d) $ and $  \bw_r^\ell(0)  \overset{\text{i.i.d.}}{\sim} \mathcal{N}(0,   \mathbf{I}_m) $ for all $r\in[m]$ and $\ell\in\{2,\ldots,L-1\}$. 
The weights of the last hidden layer and the output layer are initialized symmetrically as follows:
\begin{align}
      &  \bw_r^L (0)\! \overset{\text{i.i.d.}}{ \sim }\! \mathcal{N} (0, \mathbf{I}_m), r\!\in\!\big\{1,\!\ldots\!, \!\frac{m}{2} \big\} \textit{ and }
     \bw_{r\!+\!\frac{m}{2}}^L (0)\!  =\!\bw_r^L (0) \nonumber\\
     &\quad\, a_r \overset{\text{i.i.d.}}{\sim} \mathcal{N}(0, 1), r\in\big\{1,\ldots, \frac{m}{2}\big\} \textit{ and } a_{r+\frac{m}{2}}=-a_r.\nonumber
    \end{align}
All Gaussian draws in the above construction are mutually independent. 
This initialization ensures that $f_{\mathbf W(0)}(\mathbf x)=0$ for all $\mathbf x\in\mathcal X$, which simplifies the comparison with the corresponding kernel dynamics.
 
During training, we fix $\ba$ and only train $\bW$. For a differentiable function \(F:\W\to\R\), we denote its gradient with respect to all hidden-layer weights by \(\nabla F(\bW)\), and its gradient with respect to \(\bW^\ell\) by \(\nabla_\ell F(\bW)\).
The definitions of GD and SGD are given as follows. 
\begin{definition}[Gradient Descent]\label{GD}
    Let $\bW(0)\in\W$ be an initialization and $\eta >0$ be the step size. 
    GD updates $\{\bW(k): k\in \mathbb N\}$ by
\begin{align}\label{eq:update-GD}
    \bW (k+1)&=\bW (k) -\eta  \nabla  \L_S(\bW(k)) .
\end{align}  
\end{definition}
\begin{definition}[Stochastic Gradient Descent]\label{def:SGD}
Let $\bW(0)\in\W$ be an initialization and $\eta >0$ be the step size.  
SGD updates $\{\bW(k): k\in \mathbb N\}$ by
\begin{equation}\label{eq:sgd-update}
    \bW(k+1)=\bW(k)-\eta \nabla  l(\bW(k);z_{i_k}), 
\end{equation} 
  where $\{i_k\}_{k\geq 0}$ are i.i.d. and uniformly distributed over $[n]$.
\end{definition}

\subsection{Kernel Regression and Regularity Conditions}
\label{subsec:kernel-setting}

Given the \(L\)-layer neural network \(f_{\bW}\) in \eqref{eq:NN} with the symmetric initialization $\bW(0)$ defined in Section~\ref{subsec:learning}, we define its  NTK at initialization by \[K^{m}(\bx,\bx')
=
\left\langle
\nabla f_{\bW(0)}(\bx),
\nabla f_{\bW(0)}(\bx')
\right\rangle, \] 
where \(\langle\cdot,\cdot\rangle\) denotes the Frobenius inner product.
By construction, \(K^m\) is positive semidefinite and depends on the random initialization.
Let $K:\X\times\X\to\R$ denote the deterministic
infinite-width limit of $K^{m}$.
The explicit recursive construction of the infinite-width kernel
$K$ is given in Appendix~\ref{appde:problem}.
Let \(\H_K\) be the RKHS associated with \(K\).
For any $\bx, \bx' \in \X$, we define the function  $K_\bx\in \H_K$ by $K_\bx(\bx') = K(\bx,\bx')$. 
Viewing the empirical risk $\L_S(\cdot)$ as a functional on $\H_K$,
we define GD in $\H_K$, initialized at $g_0^{\text{GD}}=0$, by
\begin{align}\label{eq:kernel_GD_HK}
    &g_{k+1}^{\text{GD}} =  g_k^{\text{GD}} -  \frac{\eta}{n}\sum\nolimits_{i=1}^n(g_k^{\text{GD}}(\bx_i)  -  y_i) K_{\bx_i}. 
\end{align}
Let $\bS:\H_K\hookrightarrow \mathcal L_{\rho_\bx}^2$ denote the canonical inclusion operator, where $\mathcal L_{\rho_\bx}^2
:=
\left\{
f:\X\to\R:\|f\|_\rho<\infty
\right\}$. Thus, \(\bS g\) denotes \(g\in\H_K\) viewed as an element of
\(\mathcal L_{\rho_\bx}^2\).

Define the integral operator $\bL:
\mathcal L_{\rho_\bx}^2
\to
\mathcal L_{\rho_\bx}^2$ by 
$ \bL f 
=
\int_\X
K(\cdot,\bx)f(\bx)
\,d\rho_\bx(\bx).$ 
Let
\(\{(\lambda_i,\Phi_i)\}_{i\geq1}\) denote its nonzero eigenpairs,
ordered such that $\lambda_1\geq\lambda_2\geq\cdots>0,$ 
where \(\{\Phi_i\}_{i\geq1}\) is orthonormal in
\(\mathcal L_{\rho_\bx}^2\).
For \(s\in\R\), define $\bL^s f
=
\sum_{i=1}^\infty
\lambda_i^s
\langle f,\Phi_i\rangle_\rho
\Phi_i.$  

We impose the following standard assumptions on the capacity of
\(\H_K\) and the regularity of the target regression function
\cite{nitanda2021optimal,nguyen2024many}.

\begin{assumption}[Effective dimension]
\label{ass:effec_dim}
For some $\gamma\!\in\![0,1]$ and $c_\gamma\!\ge\!1$, there holds 
     $ tr\big( \bL(\bL+\lambda\mathbf{I})^{-1} \big)=\sum_{i=1}^\infty\frac{\lambda_i}{\lambda_i+\lambda}\le c_\gamma\lambda^{-\gamma}  \text{ for all }  \lambda>0. $
\end{assumption}

\begin{assumption}[Source condition]
\label{ass:frho_smth}
There exist $\beta > 0$ and $B>0$, such that $\|\bL^{-\beta}f_{\rho}\|_\rho\le B$. 
\end{assumption}

The parameter \(\gamma\) characterizes the capacity of the RKHS.    
Since \(K\) is uniformly bounded, Assumption~\ref{ass:effec_dim} always holds with \(\gamma=1\), corresponding to the capacity-independent setting.
Assumption~\ref{ass:frho_smth} characterizes the smoothness (or regularity) of $f_\rho$. A larger value of $\beta$ indicates greater smoothness of $f_\rho$, making the assumption correspondingly more stringent. 
When \(\beta\geq1/2\), the source condition implies
\(f_\rho\in\H_K\).

\section{Overview of Our Results}
In this section, we present simplified versions of our main results that highlight the key contributions.
The corresponding formal statements are provided in
Section~\ref{sec:results}.
For clarity, we suppress logarithmic factors.
We write \(M\lesssim M'\) if \(M\leq cM'\) for some absolute constant \(c>0\), \(M\asymp M'\) if both \(M\lesssim M'\) and \(M'\lesssim M\), and use $a\vee b = \max\{a,b\}$. 

\subsection{Kernel Regression Approximates Deep Regression
under (S)GD Training}
Our central result establishes a \textit{quantitative} connection between deep regression and kernel regression.
Specifically, we compare the prediction dynamics of a DNN with the kernel GD dynamics in the RKHS induced by its infinite-width NTK.
For GD, the discrepancy is controlled by the network width and training horizon, while SGD incurs an additional \(\O(\eta)\) error due to the stochastic gradient noise.

The following informal theorem provides a unified summary of the GD and SGD approximation results.
\begin{theorem}[Informal: DNN–kernel approximation]
\label{thm:informal-kernel-approximation}
Consider an \(L\)-layer fully connected neural network of width \(m\). Suppose Assumption~\ref{ass:activation} holds, and $m \ge C_\sigma^L \mathrm{poly}(d,\eta T)$. 
Under appropriate step size conditions, the following statements hold.
\begin{enumerate}[
  label=(\arabic*),
  labelindent=0pt,
  leftmargin=*,
  align=left
]
    \item
    Let \(f^{\text{GD}}_{\bW(T)}\) be the predictor produced by GD and let \(g_T^{\text{GD}}\) be the  infinite-width kernel GD iterate defined by \eqref{eq:kernel_GD_HK}. Then, with high probability over the initialization and the draw of the training sample,
    \[
        \bigl\|
            f^{\text{GD}}_{\bW(T)}
            -
            \bS g_T^{\text{GD}}
        \bigr\|_\rho^2
       \lesssim
            \frac{C_\sigma^L d(\eta T)^4}{m} .
    \] 
    \item
    Let \(f^{\text{SGD}}_{\bW(T)}\) be the predictor produced by SGD. Then, with high probability over the initialization and the draw of the training sample,
    \[
        \E
        \big[
            \bigl\|
                f^{\text{SGD}}_{\bW(T)}
                -
                \bS g_T^{\text{GD}}
            \bigr\|_\rho^2
        \big]
        \lesssim 
            \frac{C_\sigma^L d(\eta T)^5}{m}
            +
            \eta ,
    \]
    where the expectation is taken over the algorithmic randomness of SGD. 
\end{enumerate}
\end{theorem}
Theorem~\ref{thm:informal-kernel-approximation} provides a general mechanism for transferring learning-theoretic guarantees from kernel methods to deep regression.
Its explicit width-dependent bound shows that, once the width is chosen to make the network--kernel discrepancy sufficiently small, the prediction dynamics of the DNN closely track those of kernel GD.
Consequently, suitable properties of kernel methods can be transferred to deep regression, providing a \textit{systematic route} to understanding DNNs under gradient-based training.
In the next subsection, we illustrate this principle by deriving minimax-optimal generalization bounds for smooth DNNs.

\begin{table*}[t]
{\small{ 
\begin{center}
\begin{tabular}{ |c|c|c|c|c|c| } 
 \hline
 \multirow{2}{*}{\textbf{Work}}&\multirow{2}{*}{\textbf{Layer}} &\multirow{2}{*}{\textbf{Method}}  &  \multirow{2}{*}{\textbf{Setting}}  &\multirow{2}{*}{\!\textbf{Excess risk} }  &  \multirow{2}{*}{\textbf{Over-parameterization}} \\
   &&&&& \\
    \hline 
 \multirow{2}{*}{ \cite{nitanda2021optimal}}& \multirow{2}{*}{Shallow}     &\multirow{2}{*}{One-pass SGD}      &  \multirow{2}{*}{$\beta\in [1/2,1]$}&  \multirow{2}{*}{$ {\O} (  n^{-\frac{2\beta}{2\beta+\gamma}} )$} & \multirow{2}{*}{$\exp(n)$}  \\
 &&&&& \\
   \hline
 \multirow{2}{*}{\cite{nguyen2024many}}&\multirow{2}{*}{Shallow}  &\multirow{2}{*}{GD}     &  \multirow{2}{*}{$\beta> 0$ and   $2\beta+\gamma >1$}&  \multirow{2}{*}{$ {\O}(  n^{-\frac{2\beta}{2\beta+\gamma}} )$} &\multirow{2}{*}{$d^{5}n^{ \frac{(2\beta)\vee (3 - 4\beta)}{2\beta+\gamma}}$}\\
    &&&&& \\
     \hline
 \multirow{2}{*}{\cite{kohler2026rate}}&\multirow{4}{*}{Deep}  &\multirow{2}{*}{GD}     &  \multirow{4}{*}{$(p,C)$-smooth}&  \multirow{4}{*}{$ {\O}(  n^{-\frac{2p}{2p+d} + \epsilon})$} &\multirow{4}{*}{ $n^{cL(2p+d)^4}$ }\\
   &&&&& \\
   \multirow{2}{*}{\cite{kohler2026ratesgd}}& &\multirow{2}{*}{SGD}     &  &   & \\
    &&&&& \\
  \hline 
 \multirow{4}{*}{Ours} &\multirow{4}{*}{Deep}  & \multirow{2}{*}{GD} &  \multirow{4}{*}{$\beta> 0$ and   $2\beta+\gamma >1$}&  \multirow{4}{*}{$ {\O}(  n^{-\frac{2\beta}{2\beta+\gamma}}  )$} & \multirow{2}{*}{$C_\sigma^Ldn^{\frac{2\beta+4}{2\beta+\gamma}}$}  \\
 &&&&& \\
 && \multirow{2}{*}{SGD} &  & & \multirow{2}{*}{$C_\sigma^Ldn^{\frac{2\beta+5}{2\beta+\gamma}}$}  \\
&&&&& \\
  \hline 
\end{tabular}
\end{center}
}\caption{Results for GD and SGD with smooth activations in regression.
We omit logarithmic factors, and \(\epsilon>0\) is a constant.
\cite{kohler2026rate,kohler2026ratesgd} analyze linear combinations of multiple smooth deep networks.  
\label{table:summary}} } 
\end{table*}

\subsection{Minimax-Optimal Generalization Bounds for Smooth DNNs}

As a concrete application of Theorem~\ref{thm:informal-kernel-approximation},
we transfer the minimax-optimal guarantees of kernel methods to smooth DNNs trained by GD and SGD. 

\begin{theorem}[Informal: Minimax-optimal rates for DNNs]
\label{thm:informal-minimax}
Suppose Assumptions~\ref{ass:activation}--\ref{ass:frho_smth}
hold and \(2\beta+\gamma>1\).
Then, with high probability over the initialization and the draw of the training sample, the following statements hold.
\begin{enumerate}[
  label=(\arabic*),
  labelindent=0pt,
  leftmargin=*,
  align=left
]
    \item
    For GD, assume $m \gtrsim C_\sigma^L d\, n^{\frac{2\beta+4}{2\beta+\gamma}}$. Choosing $\eta \asymp C_\sigma^{-L}$ and $  T \asymp n^{\frac{1}{2\beta+\gamma}}$   yields
    \[
    \varepsilon_{\mathrm{risk}}
    \bigl(f^{\text{GD}}_{\bW(T)}\bigr)
    \lesssim
    C_\sigma^{L(\beta\vee 1)}
    n^{-\frac{2\beta}{2\beta+\gamma}}.
    \]

    \item
    For SGD, assume $m \gtrsim C_\sigma^L d\, n^{\frac{2\beta+5}{2\beta+\gamma}}$. Choosing   $ \eta \asymp   C_\sigma^{-L}
    n^{-\frac{2\beta}{2\beta+\gamma}}, $ $
    T  \asymp  n^{\frac{2\beta+1}{2\beta+\gamma}} $
    yields
    \[
    \E 
    \big[
        \varepsilon_{\mathrm{risk}}
        \bigl(f^{\text{SGD}}_{\bW(T)}\bigr)
    \big]
    \lesssim  C_\sigma^{L }   n^{-\frac{2\beta}{2\beta+\gamma}},
    \]
    where the expectation is taken over the algorithmic randomness of SGD.
\end{enumerate}
\end{theorem}
Under the same source and effective dimension conditions, kernel GD and SGD are known to attain the minimax-optimal rate \(\O(n^{-\frac{2\beta}{2\beta+\gamma}})\) \cite{lin2017optimal}.
Our theorem establishes that, under an explicit polynomial width requirement, finite-width DNNs trained by GD or SGD also attain this minimax-optimal rate.
This extends the classical minimax-optimal theory of kernel methods to finite-width deep regression trained by gradient methods.

\paragraph{Why the minimax result is nontrivial.}
Existing minimax analyses for shallow networks with smooth activations typically proceed through the random RKHS induced by the finite-width NTK \cite{nitanda2021optimal,nguyen2024many}.
Extending this approach to DNNs would require controlling the effective dimension of the random finite-width NTK of a deep network, whose tangent features exhibit intricate dependencies across layers.

Our network--kernel connection bypasses this obstacle by directly comparing the trained DNN with kernel GD induced by the deterministic infinite-width NTK.
It thereby transfers known minimax guarantees for the limiting RKHS to deep networks without characterizing the effective dimension of the random finite-width RKHS.

\subsection{Comparison with Prior Work}\label{sec:comp}
In this subsection, we highlight the key differences between our analysis and existing theory. 
\paragraph{Network--kernel connection.}
Existing NTK-based analyses of neural regression with smooth activations have primarily focused on shallow networks \cite{nitanda2021optimal,nguyen2024many,cao2024stochastic}.
To the best of our knowledge, no prior work has established a quantitative approximation of a finite-width smooth DNN trained by GD or SGD through kernel GD induced by its deterministic infinite-width NTK.
Our result establishes such a connection, with an explicit approximation error depending on the network width and training horizon.

\paragraph{Minimax-optimal generalization.}
The closest minimax-optimal results for smooth neural networks are summarized in Table~\ref{table:summary}.
For shallow networks, \cite{nguyen2024many} obtained the minimax-optimal rate \(\O(n^{-\frac{2\beta}{2\beta+\gamma}})\) for GD under general source and effective dimension conditions with \(\beta>0\) and \(2\beta+\gamma>1\). For SGD, \cite{nitanda2021optimal} established the same optimal rate  for averaged one-pass SGD when \(\beta\in[1/2,1]\), under an exponential width requirement.
\cite{cao2024stochastic} reduced the required width to a polynomial order in the capacity-independent setting. However, these analyses do not extend directly to DNNs.  

For deep networks with smooth activations, \cite{kohler2026rate,kohler2026ratesgd} established nearly minimax-optimal rates \(\O(n^{-\frac{2p}{2p+d}+\epsilon})\) for   fixed \(\epsilon>0\) under the classical \((p,C)\)-smoothness assumption.
Their estimators are linear combinations of \(K_n\) fixed-size fully connected subnetworks, with \(K_n\) growing polynomially in \(n\).
In contrast, we analyze a single standard fully connected DNN and establish the   minimax-optimal rate under general source and effective dimension conditions, with an explicit polynomial width requirement.

\section{Main Results}\label{sec:results}
In this section, we establish a quantitative DNN--kernel approximation framework.
As an application, we derive excess risk bounds and minimax-optimal rates for DNNs.

\subsection{A Kernel-to-Network Transfer Framework}
\subsubsection{Gradient Descent}
Let $\{\bW(k)\}_{k=1}^T$ denote the sequence of iterations generated by GD. 
The following theorem directly compares the predictor produced by GD
with its infinite-width kernel counterpart.

\begin{theorem}[DNN--kernel approximation for GD]\label{pro:f-g_T}
Let $\delta\in(0,1)$.
Suppose Assumption~\ref{ass:activation} holds.
Assume that $\eta \le C_\sigma^{-L}$, $1 \le \eta T \le n\big(C_\sigma^L\log(n/\delta)\big)^{-1}$ and $m \gtrsim C_\sigma^L d(\eta T)^4\log^2({m}/{\delta})$.
    With probability at least $1-L^2\exp(-C_\sigma \log^2(m)) - \delta$ over initialization $(\ba,\bW(0))$ and the draw of the training sample, it holds that
    \[\|f_{\bW(T)}^{\text{GD}} - \bS g^{\text{GD}}_T\|_\rho^2 \lesssim  \frac{C_\sigma^L d(\eta T)^4\log^2(m)}{m}.\]
\end{theorem}
Theorem~\ref{pro:f-g_T} shows that the predictor generated by GD is quantitatively approximated by kernel GD associated with the deterministic infinite-width NTK.
For a fixed $\eta T$, the approximation error decreases with the network width, while longer training requires a correspondingly larger width.

Moreover, if kernel GD satisfies an excess population risk bound $\varepsilon_{\mathrm{risk}}(\bS g^{\text{GD}}_T)
\leq
\mathcal R(n,T),$ 
then Theorem~\ref{pro:f-g_T} yields
\[
\varepsilon_{\mathrm{risk}}
\bigl(f^{\text{GD}}_{\bW(T)}\bigr)
\lesssim
\mathcal R(n,T)
+
\frac{
C_\sigma^L d(\eta T)^4\log^2(m)
}{m}.
\]
Thus, statistical guarantees for kernel GD can be transferred to
finite-width DNNs up to an explicit network--kernel approximation
error. In the next subsection, we instantiate this framework with known
statistical guarantees for kernel GD to derive minimax-optimal rates
for finite-width DNNs.

\paragraph{Proof Sketch of GD.}
To estimate $\|f^{\text{GD}}_{\bW(T)}- \bS g_T^{\text{GD}} \|_\rho^2$, we introduce the first-order linear approximation of $f^{\text{GD}}_{\bW(k)}$ around initialization $\bW(0)$:   
\begin{align*}
    f^{\text{lin,GD}}_{\bW(k)}(\bx) = f_{\bW(0)}(\bx) + \big\langle \nabla f_{\bW(0)}(\bx), \bW(k) - \bW(0)\big\rangle_2.
\end{align*}
The superscript $\mathrm{GD}$ is used only to indicate the optimization trajectory, while the underlying network remains $f_{{\mathbf W}(k)}$.
At initialization, we simply write $f_{\mathbf W(0)}$.

Let $\H_m$ denote the RKHS with respect to the finite-width NTK $K^m(\bx,\bx')$.
Define $K^m_{\bx}\in\H_m$ by $K^m_{\bx}(\bx') = K^m(\bx,\bx')$ for all $\bx,\bx'\in\X$.
Analogous to \eqref{eq:kernel_GD_HK}, we can define the iteration of GD in $\H_m$ with $g_0^{m,\text{GD}} =  0$ as
\begin{align*}
    g_{k+1}^{m,\text{GD}} = g_k^{m,\text{GD}} - \frac{\eta}{n}\sum\nolimits_{i=1}^n(g_k^{m,\text{GD}}(\bx_i)  -  y_i) K_{\bx_i}^m.
\end{align*}
We consider the following decomposition  
\begin{align*}
     \big\| f^{\text{GD}}_{\bW(T)}& - \bS g_T^{\text{GD}}  \big\|_\rho^2 \lesssim   \big\| f^{\text{GD}}_{\bW(T)} - f^{\text{lin,GD}}_{\bW(T)}\big\|_\rho^2   + \big\|f^{\text{lin,GD}}_{\bW(T)} - \bS_mg^{m,\text{GD}}_T\big\|_\rho^2   + \big\|\bS_mg^{m,\text{GD}}_T -  \bS g_T^{\text{GD}}\big\|_\rho^2,
\end{align*}
where \(\bS_m\) is the canonical inclusion into
\(\mathcal L^2_{\rho_\bx}\).

\paragraph{(1) Control of  $ \| f^{\text{GD}}_{\bW(T)} - f^{\text{lin,GD}}_{\bW(T)}  \|_\rho^2$.}
The main technical ingredient is a uniform local Hessian bound in a neighborhood of the initialization.
Its proof requires a careful layerwise perturbation analysis of both the forward activations and the backward propagation matrices, uniformly over the input and all parameters in the neighborhood. 
The Hessian bound controls the second-order Taylor remainder and yields a local almost-convexity inequality for the empirical loss.
Using this almost-convexity together with the GD update rule, we prove
by induction that, for every \(k\leq T\),
\[
\|\bW(k)-\bW(0)\|_2^2\leq 4\eta k,
\quad
\mathcal L_S(\bW(k))
\leq
\mathcal L_S(\bW(k-1)).
\]
Thus, the GD trajectory remains inside the neighborhood where the
uniform Hessian bound applies.
A Taylor expansion around \(\bW(0)\) then gives
\[
\bigl\|
f^{\text{GD}}_{\bW(k)}-f_{\bW(k)}^{\text{lin,GD}}
\bigr\|^2_\rho\lesssim \bigl\|
f^{\text{GD}}_{\bW(k)}-f_{\bW(k)}^{\text{lin,GD}}
\bigr\|^2_\infty 
\lesssim
\frac{C_\sigma^L(\eta T)^3}{ m}.
\]
Proofs of the uniform local Hessian bound and the resulting trajectory and linearization error estimates are given in Appendices~\ref{app:Hessian} and~\ref{app:GD-term1}, respectively.

\paragraph{(2) Control of $ \|f^{\text{lin,GD}}_{\bW(T)} - \bS_mg^{m,\text{GD}}_T \|_\rho^2 $.}
The dynamics of the linearized network can be viewed as a perturbed
version of kernel GD associated with \(K^m\). 
Using the trajectory, Hessian, and linearization estimates,
we can  get 
\[\bigl\|
f_{\bW(k)}^{\text{lin,GD}}-\bS_mg_k^{m,\text{GD}}
\bigr\|^2_\rho
\lesssim
\frac{C_\sigma^L (\eta T)^4}{m} .
\]

\paragraph{(3) Control of
\(\|\bS_m g_T^{m,\text{GD}}-\bS g_T^{\text{GD}}\|_\rho^2\).}
The proof proceeds in two steps.

\emph{Step 1: Stability of kernel GD.}
We first establish that kernel GD is stable with respect to uniform
perturbations of the kernel, which may be of independent interest.
For two bounded kernels \(K^1\) and \(K^2\), let
\(\{g_k^1\}_{k\geq0}\) and \(\{g_k^2\}_{k\geq0}\) denote the
corresponding kernel GD iterates, initialized at
\(g_0^1=g_0^2=0\) and run with the same step size and training sample.
Then,
\[
\|g_T^1-g_T^2\|_\rho
\leq
\|g_T^1-g_T^2\|_\infty
\lesssim
(\eta T)^2\|K^1-K^2\|_\infty.
\]
Thus, kernel GD iterates remain close whenever their underlying kernels
are uniformly close. 

\emph{Step 2: Uniform concentration of the NTK.}
We first establish a more general layerwise concentration result under
non-symmetric Gaussian initialization.
More precisely,
\[
\|K^{m,\ell}-K^\ell\|_\infty
\lesssim
C_\sigma^L
\sqrt{\frac{d\log^2(m)}{m}} \  \text{for every }  \ell\in[L] .
\]
The bound holds simultaneously over all layers and uniformly over
\(\X\times\X\).
For the symmetric initialization used in our main analysis,
the tangent features of the first \(L-1\) layers vanish, and the same
argument applied to the paired last layer yields $\|K^m-K\|_\infty
\lesssim
C_\sigma^L
\big(\frac{d\log^2(m)}{m}\big)^{1/2}$. 
Unlike results formulated only on a fixed sample, these bounds do not
depend on the smallest eigenvalue of the empirical Gram matrix.
Moreover, for smooth activations, our layerwise uniform concentration result achieves an $m^{-1/2}$ rate up to logarithmic factors, compared with the $m^{-1/6}$-type rate for deep ReLU networks in \cite{xu2024overparametrized}.
The proof combines separate controls of the forward activations and
backward propagation terms with covering arguments.

Applying Step 1 with \(K^1=K^m\), \(K^2=K\), then using Step 2 
gives
\[
\|\bS_mg_T^{m,\text{GD}}-\bS g_T^{\text{GD}}\|_\rho^2
\lesssim
C_\sigma^{ L}(\eta T)^4
\frac{d\log^2(m)}{m}.
\]
See   Appendices~\ref{app:GD-term3} and \ref{app:NTK} for more details.

Finally, combining the estimates for the three terms in the above
decomposition yields the desired connection.

\subsubsection{Stochastic Gradient Descent}
We next establish a  connection between SGD and the kernel GD iterates.
\begin{theorem}\label{pro:connection-sgd}
    Let $\delta\in(0,1)$.
Suppose Assumption~\ref{ass:activation} holds.
Let $\bW(T)$ be the output of SGD. 
Assume that $\eta \le C_\sigma^{-L}(\log(T) + 1)\big)$, $1 \le \eta T \le n\big(C_\sigma^L\log(n/\delta)\big)^{-1}$ and $m \gtrsim C_\sigma^L d(\eta T)^5\log^2({m}/{\delta})$.
    With probability at least $1-L^2\exp(-C_\sigma \log^2(m)) - \delta$ over initialization $(\ba,\bW(0))$ and the draw of  the  training sample, it holds that 
    \begin{equation*}
        \E [ \|  f^{\text{SGD}}_{\bW(T)} - \bS g^{\text{GD}}_T   \|_\rho^2 ] \!\lesssim\!  \frac{C_\sigma^L\! d(\eta T)^5\!\log^2\!(m)}{m} +\! \eta\log(T),   
    \end{equation*}
    where the expectation is taken over the algorithmic randomness of SGD.
\end{theorem}
Theorem~\ref{pro:connection-sgd} directly compares the predictor
generated by network SGD with the full-batch kernel GD iterate
associated with the deterministic infinite-width NTK.
The first term is the finite-width approximation error analogous to
that in the GD case, whereas the second term captures the additional discrepancy caused by stochastic gradient updates.
 
\paragraph{Proof sketch of SGD.}
The proof follows the same comparison strategy as in the GD case, with one additional term accounting for the discrepancy between stochastic and full-batch kernel dynamics.
Specifically, we introduce the linearized network predictor \(f_{\bW(T)}^{\text{lin,SGD}}\), the finite-width kernel SGD iterate \(f_T^m\), and the finite-width kernel GD iterate \(g_T^{m,\text{GD}}\).
These intermediate objects yield the decomposition
\begin{align*}
       & \ebb \big[\big\|f^{\text{SGD}}_{\bW(T)} \!-\! \bS g^{\text{GD}}_T\big\|_\rho^2\big] \!\lesssim\!  \ebb \big[ \big\|f^{\text{SGD}}_{\bW(T)} \!-\! f^{\text{lin,SGD}}_{\bW(T)}\big\|_\rho^2 \!+\! \big\|f^{\text{lin,SGD}}_{\bW(T)}  \!\!-\!  \bS_m f^m_T\big\|_\rho^2 \! +\!\big\|\bS_m\!\big(f^m_T \!-\! g^{m,\text{GD}}_T\big)\big\|_\rho^2\big] \!+\! \big\|\bS_mg^{ m,\text{GD}}_T \!\! -\! \bS g^{\text{GD}}_T\big\|_\rho^2.
    \end{align*}
The first two and the last terms are controlled by arguments analogous
to those used for GD.
The additional term $\ebb [ \|\bS_m (f^m_T \!-\! g^{m,\text{GD}}_T ) \|_\rho^2 ]$ is bounded by comparing the kernel SGD and kernel
GD recursions and exploiting the conditional mean-zero property of the
stochastic gradient noise.
Combining these estimates proves the theorem.  The complete proof is given in Appendix~\ref{app:transfer-SGD}.

\subsection{Minimax-Optimal Generalization Bounds}

We now transfer kernel GD guarantees to finite-width DNNs using our DNN--kernel approximation results.

 \begin{figure*}[t]
\centering
\includegraphics[width=0.3\textwidth]{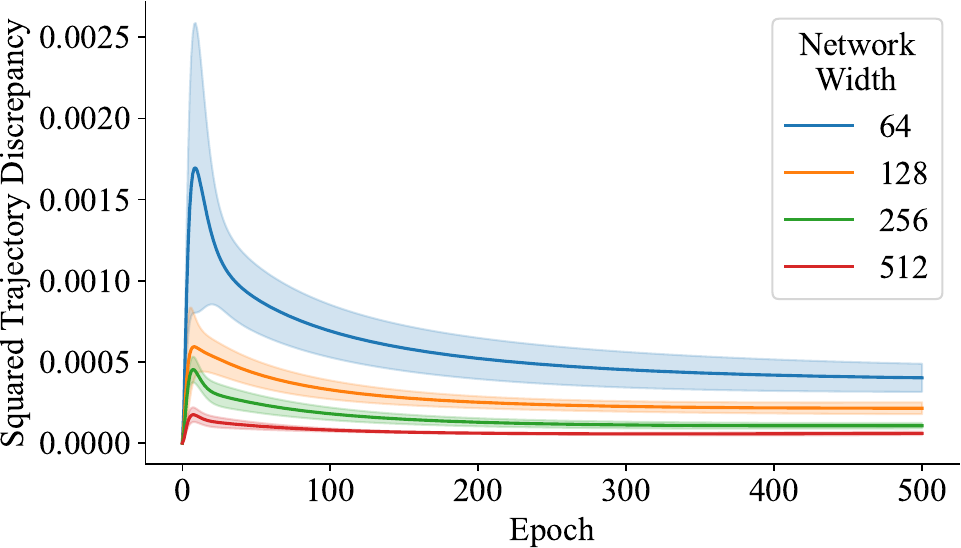} \hfill\includegraphics[width=0.3\textwidth]{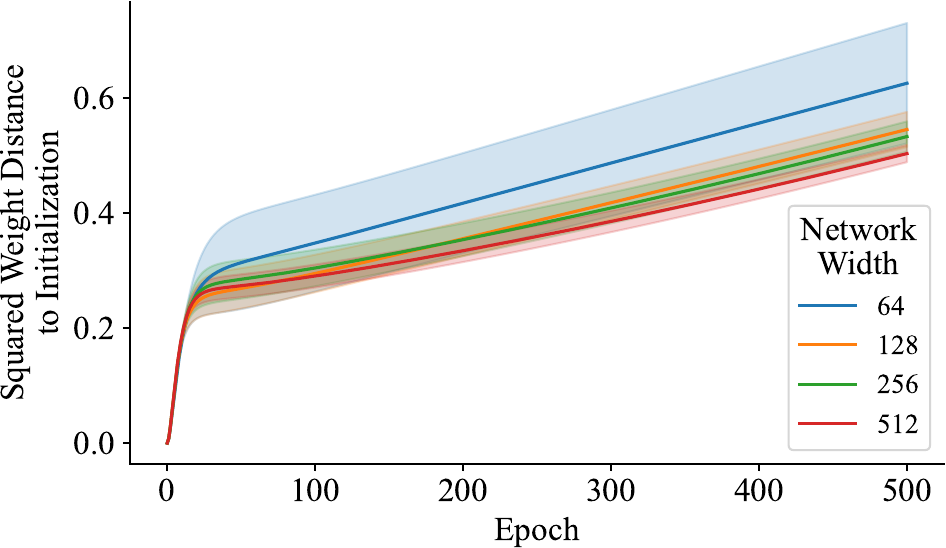}\hfill\includegraphics[width=0.3\textwidth]{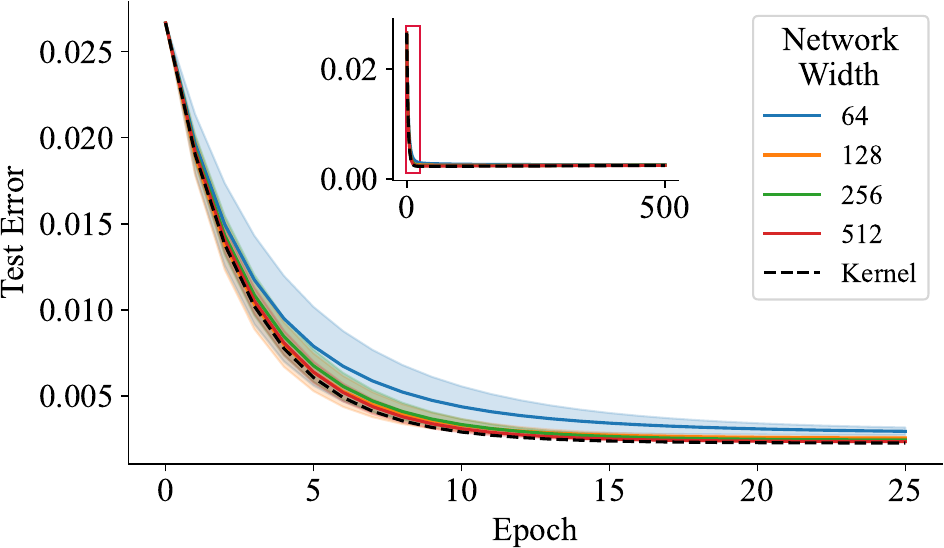}
\caption{
Squared prediction-trajectory discrepancy
\(\|f_{\mathbf W(k)}^{\mathrm{GD}}
-Sg_k^{\mathrm{GD}}\|_\rho^2\) (left),
squared distance to initialization
\(\|\mathbf W(k)-\mathbf W(0)\|_2^2\) (middle),
and test error (right) across network widths.
Results show mean and deviation over five runs.
}
\label{fig:experiments}
\end{figure*}

\paragraph{Gradient descent.}
Combining the optimality result for kernel GD from
\cite{lin2017optimal} with Theorem~\ref{pro:f-g_T}, we obtain the
following excess population risk rate.
A more general version is provided in
Appendix~\ref{app:minimax-GD}.
 
\begin{theorem}\label{cor:f-frho_gd}
    Suppose Assumptions~\ref{ass:activation}-\ref{ass:frho_smth} hold and $2\beta +\gamma>1$.
    For $\delta\in(0,1)$, assume $m\gtrsim C_\sigma^Ldn^{\frac{2\beta+4}{2\beta+\gamma}}\log^3\big(\frac{dLn}{\delta}\big)$ and $n \ge \max\{ \frac{12}{\delta}\big(\frac{18(2\beta+\gamma)}{2\beta+\gamma-1}\big)^{\frac{4\beta + 2\gamma}{2\beta + \gamma - 1}},  \eta^{-(2\beta+\gamma)}\}$. 
    Let $T = \lceil n^{\frac{1}{2\beta + \gamma}} \rceil$ and $\eta \asymp C_\sigma^{-L}$ be a constant. With probability at least $1 - \delta$ over initialization $(\ba,\bW(0))$ and the draw of  the training sample, it holds that
  \begin{center}   $\varepsilon_{risk}\big(f^{\text{GD}}_{\bW(T)}\big) \lesssim  C_\sigma^{L(\beta\vee1)}n^{-\frac{2\beta}{2\beta+\gamma}}\log^4 \big(\frac{n}{\delta} \big).$
  \end{center}
\end{theorem}
Theorem \ref{cor:f-frho_gd} implies that when the network width scales polynomially with \(n\) and \(d\),
GD-trained DNNs attain the minimax-optimal rate $\O(n^{-\frac{2\beta}{2\beta+\gamma}})$ up to logarithmic factors, with gradient complexity $\mathcal O (n^{1+\frac{1}{2\beta+\gamma}} ).$ 
Hence, finite-width DNNs inherit the classical statistical guarantees of kernel GD under the same source and effective dimension conditions.
A detailed comparison with existing minimax results has already been provided in Section~\ref{sec:comp} and Table~\ref{table:summary}.


\paragraph{Stochastic Gradient Descent} 
Combining Theorem~\ref{pro:connection-sgd} with the statistical
guarantees for kernel methods in
\cite{lin2017optimal}, we obtain the
following result.
A more general statement is provided in
Appendix~\ref{app:minimax-sgd}. 
\begin{theorem}\label{cor:f-frho_sgd}
    Suppose Assumptions~\ref{ass:activation}-\ref{ass:frho_smth} hold and $2\beta +\gamma>1$.
    For $\delta\in(0,1)$, assume $m\gtrsim C_\sigma^Ldn^{\frac{2\beta+5}{2\beta+\gamma}}\log^2\big(\frac{dLn}{\delta}\big)$ and $n \ge (C_\sigma^{L}\log(\frac{n}{\delta}))^{2\beta+\gamma}$.
    Let $T=\lceil n^{\frac{2\beta + 1}{2\beta + \gamma}} \rceil$ and $\eta=(C_\sigma^L \log(\frac{n}{\delta}))^{ -1} n^{ -\frac{2\beta}{ 2\beta+\gamma}} $. With probability at least $1  - \delta$ over initialization $(\ba, \bW(0))$ and the draw of  the training sample, it holds that
  \begin{center}
      $\ebb \big[\varepsilon_{risk}\big(f^{\text{SGD}}_{\bW(T)}\big)\big] \lesssim  C_\sigma^Ln^{-\frac{2\beta}{2\beta+\gamma}}\log^{2+2\beta} \big(\frac{n}{\delta} \big),$
  \end{center}
   where the expectation is taken over the algorithmic randomness of SGD.
\end{theorem}
Theorem~\ref{cor:f-frho_sgd} shows that SGD-trained DNNs attain the
same minimax-optimal rate $\O ( n^{-\frac{2\beta}{2\beta+\gamma}} )$ 
up to logarithmic factors, provided that the network width grows polynomially with \(n\).
The required stochastic gradient complexity is $\O (
n^{\frac{2\beta+1}{2\beta+\gamma}}
 ).$ 
Thus, the optimal statistical guarantees of kernel gradient methods
extend to finite-width DNNs trained by SGD.

\section{Experiments}
\label{sec:experiments}
In this section, we focus on GD to numerically validate the finite-width approximation
predicted by our theory, comparing network GD with the corresponding
infinite-width kernel GD across different widths.

We consider a controlled regression problem on the unit sphere.
The inputs
\(\bx\sim\mathcal U(S^4)\)
are sampled uniformly from \(S^4\), and the target function is $y
=
0.5x_1
+
0.3x_2x_3
+
0.04x_4^2
-
0.04.$ 
We train networks of the form \eqref{eq:NN} with widths $m\in\{64,128,256,512\}.$ 
The normalization constant \(c_\sigma\) is evaluated numerically, and
the deterministic infinite-width kernel \(K\) is approximated using
Gauss--Hermite quadrature.
For both network GD and kernel GD, we use the step size $\eta
\!=\!
0.2 
\lambda_{\max}
 (
n_{\mathrm{train}}^{-1}\mathbf K_{\mathrm{train}}
 )^{-1},$ 
where \(\mathbf K_{\mathrm{train}}\) is the kernel Gram matrix on  training data and \(n_{\mathrm{train}}\) is the sample size.

Figure~\ref{fig:experiments} reports three quantities over five
independent runs.
The left panel shows the discrepancy between the prediction
trajectories of network GD and infinite-width kernel GD.
After a short initial transient, the discrepancy decreases during
training and becomes consistently smaller as the network width grows.
The middle panel tracks the squared distance between the GD iterates and their initialization.
After a short initial transient, it grows approximately linearly with the number of epochs and remains smaller for wider networks, consistent with Lemma~\ref{lem:wk-w0}, which gives \(\|\mathbf W(k)-\mathbf W(0)\|_2^2\le 4\eta k\).
Finally, the right panel shows that the test errors of wider
networks increasingly approach that of kernel GD.
Overall, the experiment supports the predicted finite-width to kernel
approximation and its improvement with increasing width.

\section{Conclusion}
In this paper, we establish a key connection between the learning dynamics of DNNs and those in the RKHS by leveraging over-parameterization. Building on this connection, we demonstrate that DNNs with smooth activations trained via GD/SGD can achieve the minimax-optimal excess risk rate $\O(n^{-\frac{2\beta}{2\beta+\gamma}})$, provided the network width satisfies $m \gtrsim \text{Poly}(n, d)$. Our results show that GD and SGD with DNNs can attain generalization performance comparable to classical gradient methods in the kernel setting. 

An important direction for future work is to improve the dependence on the network depth and extend the analysis to other architectures, such as convolutional and residual networks. Our analysis is restricted to the NTK regime, where the network remains close to its linearization around initialization. Developing generalization guarantees beyond this regime, particularly in settings with substantial feature learning, remains an important open problem.  

\bibliography{learning}
 \bibliographystyle{plain}

\newpage
\appendix
\onecolumn

\begin{center}\Large \bf Appendix for ``Minimax-Optimal Generalization Bounds for Smooth Deep Neural Networks Trained by (Stochastic) Gradient Descent''\end{center}

\section{Problem Reformulation}\label{appde:problem}

We first introduce some useful notations and express the problem in a simple form. 
An $L$-layer DNN with weights $(\bW^1,\ldots,\bW^L)$ has the form
\begin{align*}
    f_\bW(\bx)=\ba^\top \sqrt{\frac{c_\sigma}{m}} \sigma\Big( \bW^L \cdots \sqrt{\frac{c_\sigma}{m}}\sigma\big( \bW^1 \bx\big)\Big).   
\end{align*}
Recall that we denote the gradient of $f_\bW$ by
\begin{align*}
    \nabla_\ell f_{\bW}(\bx) = \frac{\partial f_{\bW}(\bx)}{\partial \bW^\ell} \quad \text{for }\ell\in[L], \quad \text{and} \quad\nabla f_{\bW}(\bx) = \big(\nabla_1 f_{\bW}(\bx),\ldots,\nabla_L f_{\bW}(\bx)\big).
\end{align*}
Let $\nabla^2 f_{\bW}$ denote the Hessian of $f_\bW$ with respect to the weights $(\bW_1,\ldots,\bW_L)$.

For $\ell\in[L]$, $r\in[m]$ and $\bx\in\X$, we denote by $\bw^\ell_r(0)^\top$ the $r$-th row vector of the initialization $\bW^\ell(0)$.
Define the output of $\ell$-th layer at initialization $o^\ell_0(\bx)$  iteratively as
\begin{align}\label{eq:def_o0}
    o^\ell_0(\bx) = \sqrt{\frac{c_\sigma}{m}}\sigma\Big(\bW^\ell(0)o^{\ell-1}_0(\bx)\Big) \text{ with } o^0_0(\bx) = \bx,
\end{align}
and the diagonal matrix $\bD^\ell_0(\bx)$ as
\begin{align}\label{eq:def_D0}
    \bD^\ell_0(\bx) = \text{diag}\big(\sigma'(\bw^\ell_r(0)^\top o^{\ell-1}_0(\bx))\big)_{r=1}^m \in \R^{m\times m}.
\end{align}
For any $k,s\in\mathbb{N}$ with $k\le s$ and matrices $\bA_k,\bA_{k+1},\ldots,\bA_s$ with suitable dimensions, we denote $\prod_{h=k}^s\bA_h = \bA_s\cdots\bA_k$.
If $k > s$, we use the conventional notation $\prod_{h=k}^s\bA_h = \bfI.$

Note $f_{\bW(0)}(\bx) = \ba^\top o^L_0(\bx)$, from chain rule we know 
\begin{align*}
    \frac{\partial f_{\bW(0)}(\bx)}{\partial \bw^\ell_r(0)^\top} &= \frac{\partial \ba^\top o^L_0(\bx)}{\partial o^L_0(\bx)}\frac{\partial o^L_0(\bx)}{\partial o^{L-1}_0(\bx)}\cdots \frac{\partial o^{\ell+1}_0(\bx)}{\partial o^{\ell}_0(\bx)}\frac{\partial o^{\ell}_0(\bx)}{\partial \bw^\ell_r(0)^\top} = \ba^\top\Big[\prod_{h=\ell+1}^L\sqrt{\frac{c_\sigma}{m}}\bD^h_0(\bx)\bW^h(0)\Big]\sqrt{\frac{c_\sigma}{m}}\bD^\ell_0(\bx)\mathbf{O}^{\ell-1}_r(\bx),
\end{align*}
where $\mathbf{O}^{\ell-1}_r(\bx)$ is a matrix such that its $r$-th row is $(o^{\ell-1}_0(\bx))^\top$ and other rows are $\mathbf{0}$. 
We further define
\begin{align}\label{eq:def_V0}
    \bV^\ell_{L,0}(\bx) = \bigg(\Big[\prod_{h=\ell+1}^L\sqrt{\frac{c_\sigma}{m}}\bD^h_0(\bx)\bW^h(0)\Big]\sqrt{\frac{c_\sigma}{m}}\bD^\ell_0(\bx)\bigg)^\top.
\end{align}

Then, it holds that
\begin{align*}
    \nabla_\ell f_{\bW(0)}(\bx) = \bV^\ell_{L,0}(\bx)\ba \big(o^{\ell-1}_0(\bx)\big)^\top \text{ for all }\ell\in[L].
\end{align*}
The following lemma shows that under symmetric initialization, the gradients with respect to first $L-1$-layers vanish.
\begin{lemma}\label{lem:aDW=0}
    Under the symmetric initialization, it holds that
    \begin{align*}
        \ba^\top \bD_0^L(\bx)\bW^L(0) = 0 \quad \text{and} \quad \nabla_\ell f_{\bW(0)}(\bx) = 0 \ \text{ for all }\ \ell\in[L-1],\bx\in\X.
    \end{align*}
\end{lemma}
\begin{proof}
    Since $a_r = -a_{r+m/2}$ and $\bw^L_r(0) = \bw^L_{r+m/2}(0)$ for all $r\in[m/2]$, it holds that
    \begin{align*}
        \ba^\top \bD_0^L(\bx)\bW^L(0) &= \sum_{r=1}^m a_r\sigma'\big(\bw^L_r(0)^\top o_0^{L-1}(\bx)\big)\bw^L_r(0)^\top\\
        &= \sum_{r=1}^{m/2} a_r\sigma'\big(\bw^L_r(0)^\top o_0^{L-1}(\bx)\big)\bw^L_r(0)^\top + \sum_{r=m/2+1}^m a_r\sigma'\big(\bw^L_r(0)^\top o_0^{L-1}(\bx)\big)\bw^L_r(0)^\top\\
        &= \sum_{r=1}^{m/2} a_r\sigma'\big(\bw^L_r(0)^\top o_0^{L-1}(\bx)\big)\bw^L_r(0)^\top + \sum_{r=1}^{m/2} (-a_r)\sigma'\big(\bw^L_r(0)^\top o_0^{L-1}(\bx)\big)\bw^L_r(0)^\top = 0.
    \end{align*}
    Therefore, for any $\ell\in[L-1]$ and $\bx\in\X$, it holds that
    \begin{align*}
        &\nabla_\ell f_{\bW(0)}(\bx) = \bV^\ell_{L,0}(\bx)\ba \big(o^{\ell-1}_0(\bx)\big)^\top \\
        & = \bigg(\Big[\prod_{h=\ell+1}^{L-1}\sqrt{\frac{c_\sigma}{m}}\bD^h_0(\bx)\bW^h(0)\Big]\sqrt{\frac{c_\sigma}{m}}\bD^\ell_0(\bx)\bigg)^\top \sqrt{\frac{c_\sigma}{m}}\bW^L(0)^\top \bD^L_0(\bx)\ba\big(o_0^{\ell-1}(\bx)\big)^\top= 0,
    \end{align*}
    where the last identity used the equality $\ba^\top \bD_0^L(\bx)\bW^L(0) = 0.$
    This completes the proof of the lemma.
\end{proof}
Then, we know the finite-width NTK reduces to the inner product of the last-layer tangent feature, i.e.,
\begin{align*}
    K^m(\bx,\bx') &= \langle \nabla f_{\bW(0)}(\bx),\nabla f_{\bW(0)}(\bx')\rangle = \langle \nabla_L f_{\bW(0)}(\bx),\nabla_L f_{\bW(0)}(\bx')\rangle\\
    &= \langle o^{L-1}_0(\bx),o^{L-1}_0(\bx')\rangle \, \ba^\top \bV^L_{L,0}(\bx)^\top\bV^L_{L,0}(\bx')\ba
\end{align*}

We will prove later that with high probability, $K^{m}$ converges to a deterministic NTK $K:\X\times\X\to\R$ uniformly as the width $m\to\infty$, where $K$ is defined by 
\begin{align}\label{eq:K}
    K(\bx,\bx') = c_\sigma\ebb[\sigma(U^{L-1}(\bx))\sigma(U^{L-1}(\bx'))] q^L(\bx,\bx'),
\end{align}
where  $ \{(U^{\ell}(\bx),U^{\ell}(\bx'))\}_{\ell=1}^{L} $ are pairs of bivariate normal variables defined iteratively by $(U^{\ell}(\bx),U^{\ell}(\bx'))\sim \N(0,$ $\Sigma^{\ell-1}(\bx,\bx'))$ with  
\begin{align*}
\Sigma^{\ell-1} (\bx,\!\bx') =
    c_\sigma\left(\begin{aligned}
        &\ebb[\sigma^2(U^{\ell - 1} (\bx) ) ] && \ebb[\sigma(U^{\ell-1}(\bx) )\sigma(U^{\ell-1}(\bx') ) ]\\
         \ebb[& \sigma (U^{ \ell-1} (\bx) ) \sigma( U^{ \ell-1} (\bx') ) ] &&\ \  \ebb[\sigma^2(U^{\ell - 1}(\bx') ) ]
    \end{aligned}
     \right)\!, \end{align*}
     and
     \begin{align*}
    \Sigma^{0}(\bx,\bx') = \left(\begin{aligned}
        &1 &&\langle \bx,\bx'\rangle\\
        \langle \bx&,\bx'\rangle &&   1
    \end{aligned}
    \!\right) 
\end{align*}
and $q^\ell(\bx,\bx') =  c_\sigma\ebb[\sigma'(U^{\ell}(\bx))\sigma'(U^{\ell}(\bx'))]$, and use the notation $\ebb[\sigma(U^0(\bx))\sigma(U^0(\bx')) ] = c_\sigma^{-1}\langle \bx,\bx'\rangle $.
Note that for all $\bx,\bx'\in\X$ and $\ell\in[L]$, $\ebb[\sigma(U^\ell(\bx))\sigma(U^\ell(\bx')) ]$ and $\ebb[\sigma'(U^{\ell}(\bx))\sigma'(U^{\ell}(\bx'))]$ are deterministic, and does not involve any randomness.

We claim that
    \begin{align}\label{eq:EU=1}
        c_\sigma\ebb[\sigma^2(U^\ell(\bx))] = 1 \text{ for all }\ell\in[L-1] \text{ and }\bx\in\X.
    \end{align}
    We prove this by induction on the layer $\ell$.
    Note $U^1(\bx)\sim\N(0,1)$.
    Recall that we define $c_\sigma = (\ebb_{x\sim\N(0,1)}[\sigma^2(x)])^{-1}$, it holds that $c_\sigma\ebb[\sigma^2(U^1(\bx))] = 1$.
    Suppose this claim holds for the layer $\ell$, we will show that it also holds for the layer $\ell+1$. Indeed, the induction assumption implies that $c_\sigma\ebb[\sigma^2(U^\ell(\bx))] = 1$.
    Together with the definition of $U^{\ell+1}(\bx)$ we know $U^{\ell+1}(\bx)\sim\N\big(0,c_\sigma\ebb[\sigma^2(U^\ell(\bx)))]\big) = \N(0, 1)$.
    Then, $c_\sigma\ebb[\sigma^2(U^{\ell+1}(\bx))] = c_\sigma\ebb_{Z\sim\N(0,1)}[\sigma^2(Z)] = 1$.
    This proves \eqref{eq:EU=1}.

Now, we can show that the NTK $K$ is uniformly bounded.
    Indeed, for any $\bx,\bx'\in\X$, we know
    \begin{align}\label{eq:bound_K}
        \big|K(\bx,\bx')\big| &=  \Big|c_\sigma\ebb[\sigma(U^{L-1}(\bx))\sigma(U^{L-1}(\bx'))] c_\sigma\ebb[\sigma'(U^{L}(\bx))\sigma'(U^{L}(\bx'))]\Big|\nonumber\\
        &\le c_\sigma\Big[\sqrt{c_\sigma\ebb[\sigma^2(U^{L-1}(\bx))]} \sqrt{c_\sigma \ebb[\sigma^2(U^{L-1}(\bx'))]}  \big|\ebb[\sigma'(U^{L}(\bx))\sigma'(U^{L}(\bx'))]\big|\Big]\nonumber\\
        &=c_\sigma\big|\ebb[\sigma'(U^{L}(\bx))\sigma'(U^{L}(\bx'))]\big| \le c_\sigma B_{\sigma'}^2 \le C_\sigma,
    \end{align}
    where in the first inequality we have used Cauchy-Schwarz inequality, in the second identity we have used \eqref{eq:EU=1}.

    Let $\|\cdot\|_2$ denote the Euclidean norm for vectors and the Frobenius norm for matrices.
    From Theorem 4.21 in \cite{steinwart2008support}, we know $K^m$ uniquely defines a RKHS $\H_m$ given by \[\H_m = \{f: \X\to\R\ :\ \exists \ \bW \in \W \text{ such that } f(\bx) = \langle \bW, \nabla f_{\bW(0)}(\bx) \rangle \},\] 
    with norm defined, for any $f\in \H_m$,  by
    \[\|f\|_{\H_m} =   \inf  \{ \|\bW\|_2  :   \bW  \in  \W  \\ \text{ with } f(\bx)  =  \langle \bW,  \nabla f_{\bW(0)} (\bx) \rangle   \}.  \]

\section{Proofs of Gradient Descent}\label{app:GD}
In this section, we prove the main results for GD.

\subsection{Proofs of Uniform Local Hessian Bounds around Initialization}\label{app:Hessian}
In this subsection, we prove the uniform local Hessian bound by analyzing layerwise perturbations of the
forward activations and backward propagation matrices.
Let $\bD^\ell(\bx), o^\ell(\bx), \bV^\ell_{L}$ and $\widetilde{\bD}^\ell(\bx), \tilde{o}^\ell(\bx), \widetilde{\bV}^\ell_{L}$ be defined as \eqref{eq:def_D0}, \eqref{eq:def_o0} and \eqref{eq:def_V0} with $\bW^\ell(0)$ replaced by $\bW^\ell$ and $\widetilde{\bW}^\ell$, respectively.

The following useful lemma shows that the initial weights $\bW^\ell(0)$ are bounded by $\O(\sqrt{m})$ with high probability.  
\begin{lemma}[Theorem 4.4.5 in \cite{vershynin2018high}]\label{lemma:oprt_norm}
    Suppose $m\gtrsim d$.
    With probability at least $1-(L+1)\exp(-cm)$ over initialization $(\ba,\bW(0))$, there exists an absolute constant $c_0>1$ such that for any $\ell\in[L]$, it holds that
\begin{align}\label{eq:optr_nrm_W}
    \|\bW^\ell(0)\|_{op} \le c_0\sqrt{m} \ \text{ and } \ \|\ba\|_2 \le c_0\sqrt{m} 
\end{align}
\end{lemma}

\begin{lemma}\label{lem:tildeo-o}
    Let $R>0.$
    Assume $m\gtrsim \max\{R^2,d\}$.
    With probability at least $1-CL\exp(-cm)$ over the initialization $\bW(0)$, it holds for any $\bx\in\X$ ,$\ell\in[L]$, and any parameters $\tildeW=(\tildeW^1,\ldots,\tildeW^L)$ and $\bW=(\bW^1,\ldots,\bW^L)$ satisfying $\max_{\ell\in[L]}\{\|\bW^\ell - \bW^\ell(0)\|_{op}, \|\tildeW^\ell - \bW^\ell(0)\|_{op}\} \le R$ that
    \begin{align*}
        \|\tildeo^\ell(\bx) - o^\ell(\bx)\|_2 \le   \frac{C_\sigma^\ell\sum_{k=1}^\ell\|\tildeW^k - \bW^k\|_{op}}{\sqrt{m}}.
    \end{align*}
\end{lemma}
\begin{proof}
    For $\ell\in[L]$ and $\bx\in\X$, we have
    \begin{align*}
        &\|\tildeo^\ell(\bx) - o^\ell(\bx)\|_2  = \sqrt{\frac{c_\sigma}{m}}\big\|\sigma\big(\tildeW^\ell\tildeo^{\ell-1}(\bx)\big) - \sigma\big(\bW^\ell o^{\ell-1}(\bx)\big)\big\|_2 \le \sqrt{\frac{c_\sigma}{m}}B_{\sigma'}\big\|\tildeW^\ell\tildeo^{\ell-1}(\bx) - \bW^\ell o^{\ell-1}(\bx)\|_2\\
        &= \sqrt{\frac{c_\sigma}{m}}B_{\sigma'} \big\|\tildeW^\ell \big(\tildeo^{\ell-1}(\bx) - o^{\ell-1}(\bx)\big) + \big(\tildeW^\ell - \bW^\ell\big)o^{\ell-1}(\bx) \big\|_2 \\
        &\le \sqrt{\frac{c_\sigma}{m}}B_{\sigma'} \big((\|\tildeW^\ell - \bW^\ell(0)\|_{op} + \|\bW^\ell(0)\|_{op})\|\tildeo^{\ell-1}(\bx) - o^{\ell-1}(\bx)\|_2 + \|\tildeW^\ell - \bW^\ell\|_{op} \|o^{\ell-1}(\bx)\|_2\big)\\
        &\le  \sqrt{\frac{c_\sigma}{m}}B_{\sigma'} \big((R + c_0\sqrt{m})\|\tildeo^{\ell-1}(\bx) - o^{\ell-1}(\bx)\|_2 + \sqrt{c_\sigma}B_\sigma\|\tildeW^\ell - \bW^\ell\|_{op}\big)\\
        &= \sqrt{c_\sigma}B_\sigma\big(c_0+\frac{R}{\sqrt{m}}\big)\|\tildeo^{\ell-1}(\bx) - o^{\ell-1}(\bx)\|_2 + c_\sigma B_\sigma B_{\sigma'}\frac{\|\tildeW^\ell - \bW^\ell\|_{op}}{\sqrt{m}}\\
        &\le \sqrt{c_\sigma}B_\sigma C\|\tildeo^{\ell-1}(\bx) - o^{\ell-1}(\bx)\|_2 + c_\sigma B_\sigma B_{\sigma'}\frac{\|\tildeW^\ell - \bW^\ell\|_{op}}{\sqrt{m}},
    \end{align*}
    where the first inequality used the boundedness $\sup_{a\in\R}|\sigma'(a)|\le B_{\sigma'}$, the second-to-last inequality used Lemma~\ref{lemma:oprt_norm}, and the assumption $\|\tildeW^\ell - \bW^\ell(0)\|_{op} \le R$ for all $\ell\in[L]$, and the bound $\sup_{\bx\in\X}\|o^\ell(\bx)\|_2 \le \sqrt{c_\sigma} B_\sigma$, which follows from $\sup_{a\in\R}|\sigma(a)| \le B_\sigma$ in Assumption \ref{ass:activation}, and the last inequality used the assumption $m \gtrsim R^2.$

    Applying the above inequality iteratively over $\ell$, it holds for any $\bx\in\X$ that
    \begin{align*}
        \|\tildeo^\ell(\bx) - o^\ell(\bx)\|_2 \le \sum_{k=1}^\ell c_\sigma B_\sigma B_{\sigma'}\big(\sqrt{c_\sigma}B_\sigma C\big)^{\ell-k} \frac{\|\tildeW^k - \bW^k\|_{op}}{\sqrt{m}}\le  C_\sigma^\ell \sum_{k=1}^\ell \frac{\|\tildeW^k - \bW^k\|_{op}}{\sqrt{m}}.
    \end{align*}
    This completes the proof of the lemma.
\end{proof}

\begin{lemma}\label{lem:tildeD-D}
    Suppose that we are under the same conditions as in Lemma~\ref{lem:tildeo-o}.
    Then, it holds that
    \begin{align*}
        \|\tildeD^\ell(\bx) - \bD^\ell(\bx)\|_{2} \le C_\sigma^\ell \sum_{k=1}^{\ell} \|\tildeW^k - \bW^k\|_{op}.
    \end{align*}
\end{lemma}
\begin{proof}
    Recall that for each $r\in[m]$ and $\ell\in[L]$, $(\tilde{\bw}^\ell_r)^\top$ and $(\bw^\ell_r)^\top$ denote the $r$-th row of $\tildeW^\ell$ and $\bW^\ell$, respectively.
    It holds for any $\ell\in[L]$ and $\bx\in\X$ that
    \begin{align*}
        &\|\tildeD^\ell(\bx) - \bD^\ell(\bx)\|_{2}  = \Big(\sum_{r=1}^m\big(\sigma'\big((\tilde{\bw}^\ell_r)^\top \tildeo^{\ell-1}(\bx)\big) - \sigma'\big((\bw^\ell_r)^\top o^{\ell-1}(\bx)\big)\big)^2\Big)^{\frac{1}{2}}\\
        &\le \|\sigma''\|_\infty\Big(\sum_{r=1}^m \big((\tilde{\bw}^\ell_r)^\top \tildeo^{\ell-1}(\bx) - (\bw^\ell_r)^\top o^{\ell-1}(\bx)\big)^2\Big)^{\frac{1}{2}} = \|\sigma''\|_\infty \big\|\tildeW^\ell\tildeo^{\ell-1}(\bx) - \bW^\ell o^{\ell-1}(\bx)\big\|_2\\
        &\le B_{\sigma''}\big(\|\tildeW^\ell\|_{op}\|\tildeo^{\ell-1}(\bx) - o^{\ell-1}(\bx)\|_2 + \|\tildeW^\ell - \bW^\ell\|_{op} \|o^{\ell-1}(\bx)\|_2\big)\\
        &\le B_{\sigma''}\big((R + c_0\sqrt{m})\|\tildeo^{\ell-1}(\bx) - o^{\ell-1}(\bx)\|_2 + \|\tildeW^\ell - \bW^\ell\|_{op} \|o^{\ell-1}(\bx)\|_2\big)\\
        &\le B_{\sigma''} \Big(\frac{R+c_0\sqrt{m}}{\sqrt{m}}C_\sigma^{\ell-1} \sum_{k=1}^{\ell-1} \|\tildeW^k - \bW^k\|_{op} + \sqrt{c_\sigma}B_\sigma\|\tildeW^\ell - \bW^\ell\|_{op}\Big)\\
        &\le B_{\sigma''} \Big(CC_\sigma^{\ell-1} \sum_{k=1}^{\ell-1} \|\tildeW^k - \bW^k\|_{op} + \sqrt{c_\sigma}B_\sigma\|\tildeW^\ell - \bW^\ell\|_{op}\Big)\\
        &\le B_{\sigma''} \big(CC_\sigma^{\ell-1} + \sqrt{c_\sigma}B_\sigma\big) \sum_{k=1}^{\ell} \|\tildeW^k - \bW^k\|_{op} \le C_\sigma^\ell \sum_{k=1}^{\ell} \|\tildeW^k - \bW^k\|_{op},
    \end{align*}
        where the first inequality used the Lipschitzness of $\sigma'$ (see Assumption~\ref{ass:activation}), the third inequality used the estimate $\|\tildeW^\ell\|_{op}\le \|\tildeW^\ell - \bW^\ell(0)\|_{op} + \|\bW^\ell(0)\|_{op}\le R + c_0\sqrt{m}$, which follows from the assumption and Lemma~\ref{lemma:oprt_norm}, the fourth inequality used Lemma~\ref{lem:tildeo-o} and the fact $\sup_{\bx\in\X}\|o^\ell(\bx)\|_2 \le \sqrt{c_\sigma} B_\sigma$ by noting $\sup_{a\in\R}|\sigma(a)| \le B_\sigma$ (see Assumption \ref{ass:activation}), and in the third-to-last inequality we used the assumption $m\gtrsim R^2$, and in the last inequality we absorbed all constants in $C_\sigma^\ell.$

        This completes the proof of the lemma.
\end{proof}

\begin{lemma}\label{lem:tildeVa-Va}
    Let $\delta\in(0,1)$ and $R>0.$
    Assume $m\gtrsim \max\{R^2,d\}$.
    With probability at least $1-CL\exp(-cm) - \delta$ over the initialization $(\ba, \bW(0))$, it holds for any $\bx\in\X$ ,$\ell\in[L]$, and any parameters $\tildeW=(\tildeW^1,\ldots,\tildeW^L)$ and $\bW=(\bW^1,\ldots,\bW^L)$ satisfying $\max_{\ell\in[L]}\{\|\bW^\ell - \bW^\ell(0)\|_{op}, \|\tildeW^\ell - \bW^\ell(0)\|_{op}\} \le R$ that
    \begin{align*}
        \big\|\big(\tildeV^\ell_L(\bx) - \bV^\ell_L(\bx)\big)\ba\big\|_2 \le C_\sigma^L (R+1)\sqrt{\frac{\log(m/\delta)}{m}} \sum_{\ell=1}^L \|\tildeW^\ell - \bW^\ell\|_{op}.
    \end{align*}
\end{lemma}
\begin{proof}
    \def\tildeG{\widetilde{\bG}}
    For $\ell\in[L]$, parameter $\bW^\ell$ and $\bx\in\X$, we define $\bG^\ell(\bx) := \sqrt{\frac{c_\sigma}{m}}\bD^\ell(\bx)\bW^\ell.$
    Then, we know that
    $$\bV^\ell_L(\bx) = \sqrt{\frac{c_\sigma}{m}}\bD^\ell(\bx)\Big(\prod_{k=\ell+1}^L\bG^k(\bx)\Big)^\top,$$
    and we have the following estimate
    \begin{align}\label{eq:boundGell}
        \|\bG^\ell(\bx)\|_{op} \!\le\! \sqrt{\frac{c_\sigma}{m}}\|\bD^\ell(\bx)\|_{op}\|\bW^\ell\|_{op} \!\le\! \sqrt{\frac{c_\sigma}{m}} B_{\sigma'}\big(\|\bW^\ell \!-\!\bW^\ell(0)\|_{op} \!+\! \|\bW^\ell(0)\|_{op}\big) \!\le\!  \sqrt{c_\sigma}B_{\sigma'}\frac{R\!+\!c_0\sqrt{m}}{\sqrt{m}} \le C_\sigma,
    \end{align}
    where we have used the assumption $\|\bW^\ell-\bW^\ell(0)\|_{op}\le R$ and $m\gtrsim R^2.$
    
    Similarly, we can define $\tildeG^\ell(\bx)$ for $\ell\in[L]$, parameter $\tildeW^\ell$ and $\bx\in\X,$ and we also have the bound $\|\tildeG^\ell(\bx)\|_{op}\le C_\sigma.$
    
    For the case $\ell\in[L-1]$ and $\bx\in\X$, it holds that
    \begin{align}
        &\big\|\big(\bV^\ell_L(\bx) - \tildeV^\ell_L(\bx)\big)\ba\big\|_2\nonumber\\
        &= \Big\|\ba^\top \Big[\prod_{k=\ell+1}^L \bG^k(\bx) \Big] \sqrt{\frac{c_\sigma}{m}} \bD^\ell(\bx) - \ba^\top \Big[\prod_{k=\ell+1}^L \tildeG^k(\bx) \Big] \sqrt{\frac{c_\sigma}{m}}\tildeD^\ell(\bx)\Big\|_2 \nonumber\\
        &\le \sqrt{\frac{c_\sigma}{m}}\Big(\Big\|\ba^\top \Big[\prod_{k=\ell+1}^L \bG^k(\bx)\Big]\big(\bD^\ell(\bx) - \tildeD^\ell(\bx)\big)\Big\|_2 + \Big\|\ba^\top\Big[\prod_{k=\ell+1}^L \bG^k(\bx) - \prod_{k=\ell+1}^L \tildeG^k(\bx)\Big]\tildeD^\ell(\bx)\Big\|_2\Big)\nonumber\\
        &\le \sqrt{\frac{c_\sigma}{m}}\Big(\Big\|\ba^\top \Big[\prod_{k=\ell+1}^L \bG^k(\bx)\Big]\Big\|_2\big\|\bD^\ell(\bx) - \tildeD^\ell(\bx)\big\|_{op} + \Big\|\ba^\top\Big[\prod_{k=\ell+1}^L \bG^k(\bx) - \prod_{k=\ell+1}^L \tildeG^k(\bx)\Big]\Big\|_2\big\|\tildeD^\ell(\bx)\big\|_{op}\Big) \nonumber\\
        &\le \sqrt{\frac{c_\sigma}{m}}\bigg(C_\sigma^\ell \Big(\sum_{k=1}^{\ell} \|\tildeW^k - \bW^k\|_{op}\Big)\Big\|\ba^\top \Big[\prod_{k=\ell+1}^L \bG^k(\bx)\Big]\Big\|_2 + B_{\sigma'}\Big\|\ba^\top\Big[\prod_{k=\ell+1}^L \bG^k(\bx) - \prod_{k=\ell+1}^L \tildeG^k(\bx)\Big]\Big\|_2 \bigg) \label{eq:V-tildeV},
    \end{align}
    where in the last inequality we have used Lemma~\ref{lem:tildeD-D}.

    Now, we turn to control $\|\ba^\top [\prod_{k=\ell+1}^L \bG^k(\bx)]\|_2$ and $\|\ba^\top[\prod_{k=\ell+1}^L \bG^k(\bx) - \prod_{k=\ell+1}^L \tildeG^k(\bx)]\|_2$, respectively.
    For any $\ell\in[L-1]$ and $\bx\in\X$, it holds that
    \begin{align}
        &\quad \Big\|\ba^\top \Big[\prod_{k=\ell+1}^L \bG^k(\bx)\Big]\Big\|_2 = \sqrt{\frac{c_\sigma}{m}}\Big\|\ba^\top \bD^L(\bx)\bW^L \Big[\prod_{k=\ell+1}^{L-1} \bG^k(\bx)\Big]\Big\|_2\nonumber\\
        &= \sqrt{\frac{c_\sigma}{m}}\Big\|\ba^\top \big(\bD^L(\bx)\bW^L - \bD^L_0(\bx) \bW^L(0)\big)\Big[\prod_{k=\ell+1}^{L-1} \bG^k(\bx)\Big]\Big\|_2\nonumber\\
        &\le \sqrt{\frac{c_\sigma}{m}}\Big\|\ba^\top \Big(\bD^L(\bx)\big(\bW^L - \bW^L(0)\big) + \big(\bD^L(\bx) -\bD^L_0(\bx)\big) \bW^L(0)\Big)\Big[\prod_{k=\ell+1}^{L-1} \bG^k(\bx)\Big]\Big\|_2\nonumber\\
        &\le \sqrt{\frac{c_\sigma}{m}}\big(\big\|\ba^\top \bD^L(\bx)\big(\bW^L - \bW^L(0)\big)\big\|_2 + \big\|\ba^\top \big(\bD^L(\bx) -\bD^L_0(\bx)\big) \bW^L(0)\big\|_2\big)\prod_{k=\ell+1}^{L-1} \|\bG^k(\bx)\|_{op}\nonumber\\
        &\le \sqrt{\frac{c_\sigma}{m}}\Big(\|\ba\|_2\|\bD^L(\bx)\|_{op}\|\bW^L - \bW^L(0)\|_{op} + \|\ba\|_\infty\|\bD^L(\bx) - \bD^L_0(\bx)\|_2\|\bW^L(0)\|_{op} \Big) \prod_{k=\ell+1}^{L-1} \|\bG^k(\bx)\|_{op}\nonumber\\
        &\le \sqrt{\frac{c_\sigma}{m}}\Big(c_0\sqrt{m} B_{\sigma'}\|\bW^L - \bW^L(0)\|_{op} + c_0\sqrt{m} \|\ba\|_\infty C_\sigma^L \sum_{k=1}^L\|\bW^k - \bW^k(0)\|_{op} \Big) \prod_{k=\ell+1}^{L-1} \|\bG^k(\bx)\|_{op}\nonumber\\
        &\le \sqrt{c_\sigma}\big(c_0 B_{\sigma'}R + c_0\|\ba\|_\infty C_\sigma^L LR \big)C_\sigma^L \le C_\sigma^L(\|\ba\|_\infty + 1)R, \label{eq:bound_aG}
    \end{align}
    where the second identity used Lemma~\ref{lem:aDW=0}, the third inequality used the bound $\|\ba^\top(\bD^L(\bx) - \bD^L_0(\bx))\|_2\le\|\ba\|_\infty\|\bD^L(\bx) - \bD^L_0(\bx)\|_2$ by noting that $\|\bu^\top \bD\|_2 \le \|\bu\|_\infty\|\bD\|_2$ for any vector $\bu$ and diagonal matrix $\bD$, the third-to-last inequality follows from Lemma~\ref{lemma:oprt_norm} and Lemma~\ref{lem:tildeD-D}, and in the second-to-last inequality we have used \eqref{eq:boundGell} and the assumption $\|\bW^k - \bW^k(0)\|_{op}\le R$ for all $k\in[L].$

    Similarly, we can also show that
    \begin{align}
        \Big\|\ba^\top \Big[\prod_{k=\ell+1}^L \tildeG^k(\bx)\Big]\Big\|_2 \le C_\sigma^L(\|\ba\|_\infty + 1)R. \label{eq:bound_atildeG}
    \end{align}

    Note that for all $\ell\in\{2,\ldots,L\}$ and $\bx\in\X$ we have
    \begin{align*}
        &\quad \Big\|\ba^\top\Big[\prod_{k=\ell}^L \bG^k(\bx) - \prod_{k=\ell}^L \tildeG^k(\bx)\Big]\Big\|_2\\
        &=\Big\|\ba^\top\Big[\prod_{k=\ell+1}^L \bG^k(\bx) - \prod_{k=\ell+1}^L \tildeG^k(\bx)\Big] \bG^\ell(\bx)\Big\|_2 + \Big\|\ba^\top \Big[\prod_{k=\ell+1}^L \tildeG^k(\bx)\Big]\big(\bG^\ell(\bx) - \tildeG^\ell(\bx)\big)\Big\|_2 \\
        &\le \Big\|\ba^\top\Big[\prod_{k=\ell+1}^L \bG^k(\bx) - \prod_{k=\ell+1}^L \tildeG^k(\bx)\Big]\Big\|_2 \big\|\bG^\ell(\bx)\big\|_{op} + \Big\|\ba^\top \Big[\prod_{k=\ell+1}^L \tildeG^k(\bx)\Big]\Big\|_2 \big\|\bG^\ell(\bx) - \tildeG^\ell(\bx)\big)\big\|_{op}\\
        &\le C_\sigma \Big\|\ba^\top\Big[\!\prod_{k=\ell+1}^L \bG^k(\bx) \!-\!\!\! \prod_{k=\ell+1}^L \tildeG^k(\bx)\Big]\Big\|_2  +   C_\sigma^L(\|\ba\|_\infty +  1)R \sqrt{\frac{c_\sigma}{m}}\big(\|\bD^\ell(\bx)  -  \tildeD^\ell(\bx)\|_{op}\|\bW^\ell\|_{op}\nonumber\\&\quad   +  \|\tildeD^\ell(\bx)\|_{op}\|\bW^\ell  - \tildeW^\ell\|_{op}\big)\\
        &\le C_\sigma \Big\|\ba^\top\Big[\prod_{k=\ell+1}^L \bG^k(\bx) - \prod_{k=\ell+1}^L \tildeG^k(\bx)\Big]\Big\|_2 + \frac{C_\sigma^L(\|\ba\|_\infty + 1)R}{\sqrt{m}}\big(C_\sigma^\ell(R+c_0\sqrt{m}) \sum_{k=1}^\ell \|\bW^k-\tildeW^k\|_{op} \nonumber\\&\quad + B_{\sigma'}\|\bW^\ell - \tildeW^\ell\|_{op}\big)\\
        &\le C_\sigma \Big\|\ba^\top\Big[\prod_{k=\ell+1}^L \bG^k(\bx) - \prod_{k=\ell+1}^L \tildeG^k(\bx)\Big]\Big\|_2 + C_\sigma^L(\|\ba\|_\infty + 1)R\sum_{k=1}^\ell \|\bW^k-\tildeW^k\|_{op},
    \end{align*}
    where the second inequality used \eqref{eq:boundGell} and \eqref{eq:bound_atildeG}, and the second-to-last inequality used Lemma~\ref{lem:tildeD-D}, the estimates $\|\bW^\ell\|_{op}\le \|\bW^\ell - \bW^\ell(0)\|_{op}+ \|\bW^\ell(0)\|_{op} \le R + c_0\sqrt{m}$ and $\|\tildeD^\ell(\bx)\|_{op}\le\|\sigma'\|_\infty\le B_{\sigma'}$, and the last inequality used the assumption $m\gtrsim R^2.$

    Applying the above inequality iteratively over $\ell$, it holds that
    \begin{align*}
        \Big\|\ba^\top\Big[\prod_{k=\ell}^L \bG^k(\bx) - \prod_{k=\ell}^L \tildeG^k(\bx)\Big]\Big\|_2 & \le \sum_{k=\ell}^LC_\sigma^{L + k-\ell}(\|\ba\|_\infty + 1)\sum_{h=1}^k\|\bW^h - \tildeW^h\|_{op} \\& \le C_\sigma^L(\|\ba\|_\infty + 1)R\sum_{k=1}^L\|\bW^k-\tildeW^k\|_{op}.
    \end{align*}

    Plugging the above inequality and \eqref{eq:bound_aG} into \eqref{eq:V-tildeV}, we have
    \begin{align*}
        &\big\|\big(\bV^\ell_L(\bx) - \tildeV^\ell_L(\bx)\big)\ba\big\|_2 \\ &\le \sqrt{\frac{c_\sigma}{m}}\bigg(C_\sigma^L (\|\ba\|_\infty + 1)R \Big(\sum_{k=1}^{\ell} \|\tildeW^k - \bW^k\|_{op}\Big) + B_{\sigma'}C_\sigma^L (\|\ba\|_\infty + 1)R\sum_{k=1}^L\|\bW^k-\tildeW^k\|_{op}\bigg)\\
        &\le \frac{C_\sigma^L}{\sqrt{m}}(\|\ba\|_\infty + 1)R\sum_{k=1}^L\|\bW^k-\tildeW^k\|_{op}.
    \end{align*}

    It remains to control $\|\ba\|_\infty$.
    Recall that $a_r\sim N(0,1)$ for all $r\in[m].$
    For any $t>0$, by the union bound and the standard Gaussian tail bound, we have
    \begin{align*}
        \mathbb P\big(\|\mathbf a\|_\infty \ge t\big) = \mathbb P\Big(\max_{r\in[m]} |a_r| \ge t\Big) \le \sum_{r=1}^m \mathbb P\big(|a_r|\ge t\big) \le 2m \exp\Big(-\frac{t^2}{2}\Big).
    \end{align*}
    Take $t=\sqrt{2\log\frac{2m}{\delta}}$ yields that with probability at least $1-\delta$,
    \begin{align}\label{eq:bound_a_infty}
        \|\mathbf a\|_\infty \le \sqrt{2\log\frac{2m}{\delta}}.
    \end{align}
    Plugging this bound into the previous estimate shows that for the case $\ell\in[L-1]$, we have
    \begin{align*}
        \big\|\big(\tildeV^\ell_L(\bx) - \bV^\ell_L(\bx)\big)\ba\big\|_2 \le C_\sigma^L R\sqrt{\frac{\log(m/\delta)}{m}}\sum_{k=1}^{L} \|\tildeW^k - \bW^k\|_{op}.
    \end{align*}

    For the case $\ell=L,$ from the above inequality and Lemma~\ref{lem:tildeD-D} we have
    \begin{align*}
        \big\|\big(\tildeV^\ell_L(\bx) - \bV^\ell_L(\bx)\big)\ba\big\|_2 &= \sqrt{\frac{c_\sigma}{m}}\big\|\ba^\top\big(\tildeD^L(\bx) - \bD^L(\bx)\big)\big\|_2 \le \sqrt{\frac{c_\sigma}{m}} \|\ba\|_\infty \|\tildeD^L(\bx) - \bD^L(\bx)\|_2\\
        &\le C_\sigma^L\sqrt{\frac{\log(m/\delta)}{m}}\sum_{k=1}^{L} \|\tildeW^k - \bW^k\|_{op}.
    \end{align*}
    Combining the above two cases for $\ell$ completes the proof of the lemma.
\end{proof}

\begin{lemma}\label{lem:Hessian}
    Let $\delta\in(0,1)$ and $R>0.$
    Assume $m\gtrsim \max\{R^2, d\}$.
    With probability at least $1-CL\exp(-cm) - \delta$ over the initialization $(\ba, \bW(0))$, it holds for any $\bx\in\X$ ,$\ell\in[L]$, and any parameters $\tildeW=(\tildeW^1,\ldots,\tildeW^L)$ and $\bW=(\bW^1,\ldots,\bW^L)$ satisfying $\max_{\ell\in[L]}\{\|\bW^\ell - \bW^\ell(0)\|_{op}, \|\tildeW^\ell - \bW^\ell(0)\|_{op}\} \le R$ that
    \begin{align*}
        \big\|\nabla f_{\bW}(\bx)\big\|_2 \le C_\sigma^L \quad\text{and}\quad \big\|\nabla f_{\tildeW}(\bx) - \nabla f_{\bW}(\bx)\big\|_2 \le C_\sigma^L (R+1)\sqrt{\frac{\log(m/\delta)}{m}}\|\tildeW - \bW\|_2.
    \end{align*}
    As a corollary, for any parameter $\bW=(\bW^1,\ldots,\bW^L)$ that $\max_{\ell\in[L]}\{\|\bW^\ell - \bW^\ell(0)\|_2\} \le R$, we have
    \begin{align*}
        \sup_{\bx\in\X}\|\nabla^2f_{\bW}(\bx)\|_{op} \le C_\sigma^L (R+1)\sqrt{\frac{\log(m/\delta)}{m}}.
    \end{align*}
\end{lemma}
\begin{proof}
    We first control the norm of $\nabla f_{\bW}(\bx)$.
    Recall that in the proof of Lemma~\ref{lem:tildeVa-Va}, we defined $$\bG^\ell(\bx) := \sqrt{\frac{c_\sigma}{m}}\bD^\ell(\bx)\bW^\ell,$$
    and hence we have
    \begin{align*}
        \bV^\ell_L(\bx) = \sqrt{\frac{c_\sigma}{m}}\bD^\ell(\bx)\Big(\prod_{k=\ell+1}^L\bG^k(\bx)\Big)^\top.
    \end{align*}
    Then, it holds for any $\ell\in[L]$ and $\bx\in\X$ that
    \begin{align*}
         \big\|\nabla_\ell f_{\bW}(\bx)\big\|_2 &= \Big\|\bV^\ell_{L}(\bx)\ba \big(o^{\ell-1}(\bx)\big)^\top\Big\|_2\le \big\|\bV^{\ell}_{L}(\bx)\ba\big\|_2 \|o^{\ell-1}(\bx)\|_2 \\
        & \le \prod_{k=\ell+1}^L\|\bG^k(\bx)\|_{op}\sqrt{\frac{c_\sigma}{m}}\|\bD^\ell(\bx)\|_{op}\|\ba\|_2 \sqrt{c_\sigma} B_\sigma\le (C_\sigma)^{L-\ell}\sqrt{\frac{c_\sigma}{m}}B_{\sigma'}c_0\sqrt{m} B_\sigma \le C_\sigma^L,
    \end{align*}
    where the second inequality used the bound $\sup_{\bx\in\X}\|o^\ell(\bx)\|_2 \le \sqrt{c_\sigma}\|\sigma\|_\infty\le \sqrt{c_\sigma} B_\sigma$, and in the third inequality we have used Lemma~\ref{lemma:oprt_norm}, \eqref{eq:boundGell}, and the bound $\|\bD^\ell(\bx)\|_{op}\le\|\sigma'\|_\infty\le B_{\sigma'}$.
    Therefore, we have
    \begin{align*}
        \big\|\nabla f_{\bW}(\bx)\big\|_2 = \Big(\sum_{\ell=1}^L\big\|\nabla_\ell f_{\bW}(\bx)\big\|_2^2\Big)^{\frac{1}{2}} \le \sum_{\ell=1}^L\big\|\nabla_\ell f_{\bW}(\bx)\big\|_2 \le LC_\sigma^L \lesssim C_\sigma^L.
    \end{align*}

    Now, we turn to bounding the variation of $\big\|\nabla_\ell f_{\bW}(\bx)\big\|_2$ with respect to parameter $\bW$.
    \begin{align*}
        &\quad \big\|\nabla_\ell f_{\tildeW}(\bx) - \nabla_\ell f_{\bW}(\bx)\big\|_2\\
        &= \big\|\tildeV^\ell_{L}(\bx)\ba \big(\tildeo^{\ell-1}(\bx)\big)^\top - \bV^\ell_{L}(\bx)\ba \big(o^{\ell-1}(\bx)\big)^\top\big\|_2\\
        &\le \big\|\big(\tildeV^\ell_{L}(\bx) - \bV_L^\ell(\bx)\big)\ba \big(\tildeo^{\ell-1}(\bx)\big)^\top\big\|_2 + \big\|\bV^\ell_{L}(\bx)\ba \big(\tildeo^{\ell-1}(\bx) - o^{\ell-1}(\bx)\big)^\top\big\|_2\\
        &\le \big\|\big(\tildeV^\ell_{L}(\bx) - \bV_L^\ell(\bx)\big)\ba\big\|_2 \|\tildeo^{\ell-1}(\bx)\|_2 + \Big\|\ba\prod_{k=\ell+1}^L\bG^k(\bx)\Big\|_2 \sqrt{\frac{c_\sigma}{m}}\|\bD^\ell(\bx)\|_{op} \big\|\tildeo^{\ell-1}(\bx) - o^{\ell-1}(\bx)\big\|_2\\
        &\le C_\sigma^L(R+1)\sqrt{\frac{\log(m/\delta)}{m}} \sum_{\ell=1}^L \|\tildeW^\ell - \bW^\ell\|_{op} + C_\sigma^L(\|\ba\|_\infty + 1)R \sqrt{\frac{c_\sigma}{m}}B_{\sigma'} C_\sigma^\ell \sum_{k=1}^\ell \|\tildeW^k - \bW^k\|_{op}\\
        &\le C_\sigma^L(R+1)\sqrt{\frac{\log(m/\delta)}{m}} \sum_{\ell=1}^L \|\tildeW^\ell - \bW^\ell\|_{op},
    \end{align*}
    where the second-to-last inequality we have used Lemmas~\ref{lem:tildeo-o} and~\ref{lem:tildeVa-Va}, and \eqref{eq:bound_aG}, and in the last inequality we have used \eqref{eq:bound_a_infty}.

    Therefore, for any $\bx\in\X$, we have
    \begin{align*}
        &\quad \, \big\|\nabla f_{\tildeW}(\bx) \!-\! \nabla f_{\bW}(\bx)\big\|_2\\
        &= \Big(\sum_{\ell=1}^L\big\|\nabla_\ell f_{\tildeW}(\bx) \!-\! \nabla_\ell f_{\bW}(\bx)\big\|_2^2\Big)^{\frac{1}{2}} \!\le\! \sqrt{L}\max_{\ell\in[L]}\big\|\nabla_\ell f_{\tildeW}(\bx) - \nabla_\ell f_{\bW}(\bx)\big\|_2\\
        &\le \sqrt{L}C_\sigma^L(R+1)\sqrt{\frac{\log(m/\delta)}{m}} \sum_{\ell=1}^L \|\tildeW^\ell - \bW^\ell\|_{op} \le \sqrt{L}C_\sigma^L(R+1)\sqrt{\frac{\log(m/\delta)}{m}} \sum_{\ell=1}^L \|\tildeW^\ell - \bW^\ell\|_2\\
        &\le LC_\sigma^L(R+1)\sqrt{\frac{\log(m/\delta)}{m}} \Big(\sum_{\ell=1}^L \|\tildeW^\ell - \bW^\ell\|_2^2\Big)^{\frac{1}{2}} \le C_\sigma^L(R+1)\sqrt{\frac{\log(m/\delta)}{m}}\|\tildeW - \bW\|_2,
    \end{align*}
    where the second-to-last inequality follows from $\sum_{\ell=1}^L b_\ell \le L(\sum_{\ell=1}^L b_\ell^2)^{1/2}$ for any $b_\ell \ge 0$ and the last inequality absorbed $LC_\sigma^L$ into the generic factor $C_\sigma^L$, since $C_\sigma\ge1$ may differ from line to line.
    
    This completes the proof of the theorem.
\end{proof}

\subsection{A Kernel-to-Network Transfer Framework of GD}
Recall that \(\bW(k)\) denotes the \(k\)-th GD iterate generated by \eqref{eq:update-GD}, and \(f_{\bW(k)}^{\text{GD}}\) and $\nabla f_{\bW(k)}^{\text{GD}}$ are the corresponding network predictor and gradient.
Note that the superscript $\mathrm{GD}$ is used only to indicate the optimization trajectory, while the underlying network and the gradient remain $f_{\bW(k)}$ and $\nabla f_{\bW(k)}$, respectively.
At initialization, we simply write $f_{\mathbf W(0)}$ and $\nabla f_{\mathbf W(0)}$.
Define the first-order linear approximation of $f^{\text{GD}}_{\bW(k)}$ around initialization $\bW(0)$ by
\begin{align}\label{eq:f-lin}
    f^{\text{lin,GD}}_{\bW(k)}(\bx) := f_{\bW(0)}(\bx) + \big\langle \nabla f_{\bW(0)}(\bx), \bW(k) - \bW(0)\big\rangle_2\ \ \text{for all} \ \bx\in\X.
\end{align}
Recall that we introduced the finite-width NTK $K^m$ and denoted the associated RKHS by $\H_m$.
Define $K^m_{\bx}\in\H_m$ by $K^m_{\bx}(\bx') = K^m(\bx,\bx')$ for all $\bx,\bx'\in\X$.
Analogous to \eqref{eq:kernel_GD_HK}, we can define the iteration of GD in $\H_m$ as
\begin{align}\label{eq:GD-Km}
    g_{k+1}^{m,\text{GD}} = g_k^{m,\text{GD}} - \frac{\eta}{n}\sum\nolimits_{i=1}^n(g_k^{m,\text{GD}}(\bx_i)  -  y_i) K_{\bx_i}^m \quad \text{with} \quad g_0^{m,\text{GD}} =  0.
\end{align}
With these intermediate objects, we can further decompose the first error term $\|f_{\bW(T)}-  \bS g_T\|_\rho^2$ into three contributions
\begin{align*}
    \big\|f_{\bW(T)}^{\text{GD}}-  \bS g_T^{\text{GD}}\big\|_\rho^2 \lesssim \big\|f_{\bW(T)}^{\text{GD}} - f^{\text{lin,GD}}_{\bW(T)}\big\|_\rho^2 + \big\|f^{\text{lin,GD}}_{\bW(T)} - \bS_mg^{m,\text{GD}}_T\big\|_\rho^2 + \big\|\bS_mg^{m,\text{GD}}_T - \bS g_T^{\text{GD}}\big\|_\rho^2.
\end{align*}
We will bound these three terms in the sequel.
\subsubsection{(1) Control of $ \|f_{\bW(T)}^{\text{GD}} - f^{\text{lin,GD}}_{\bW(T)} \|_\rho^2$}\label{app:GD-term1}
\begin{lemma}\label{lem:almost-convex}
    Let $\delta\in(0,1).$
    Assume $m\ge \max\{R^2, d\}$. 
    With probability at least $1 - CL\exp(-m/2) - \delta$ over initialization $(\ba, \bW(0))$, it holds for any $z = (\bx, y)\in\Z$ and any parameters $\tildeW,\bW$ with $\|\bW^\ell - \bW^\ell(0)\|_{op} \le R$ and $\|\widetilde{\bW}^\ell -\bW^\ell(0)\|_{op} \le R$ that
    \begin{align}\label{eq:almost_linear}
        \big|f_{\widetilde{\bW}}(\bx) - f_{\bW}(\bx) - \big\langle \nabla f_{\bW}(\bx), \widetilde{\bW} - \bW\big\rangle \big| \le C_\sigma^L(R+1)\sqrt{\frac{\log(m/\delta)}{m}}\big\|\widetilde{\bW}-\bW\big\|^2_2,
    \end{align}
    and
    \begin{align}\label{eq:almost_convex}
    l(\widetilde{\bW};z ) -  l(\bW;z ) \ge\big\langle \nabla l(\bW;z), \widetilde{\bW}-\bW  \big\rangle - |f_{\bW}(\bx) - y|\cdot \epsilon \ \text{ with }\ \epsilon \le C_\sigma^L(R+1)\|\widetilde{\bW}-\bW\|^2_2\sqrt{\log(m/\delta)/m}.
    \end{align}
\end{lemma}
\begin{proof}
    We start by proving the first part.
    Fix arbitrary $\bx\in\X$.
    From Taylor's theorem we know there exists $\bW^* = \tau\widetilde{\bW} + (1-\tau)\bW$ with some $\tau\in[0,1]$, it holds that
    \begin{align*}
        \big|f_{\widetilde{\bW}}(\bx) - f_{\bW}(\bx) - \big\langle \nabla f_{\bW}(\bx), \widetilde{\bW} - \bW\big\rangle \big| &\le\frac{\big\|\nabla^2f_{\bW^*}(\bx)\big\|_{op}\big\| \widetilde{\bW}-\bW\big\|^2_2}{2} \le C_\sigma^L(R+1)\sqrt{\frac{\log(m/\delta)}{m}}\big\|\widetilde{\bW}-\bW\big\|^2_2,
    \end{align*}
    where the last inequality used Lemma \ref{lem:Hessian} by noting $\|(\bW^*)^\ell - \bW^\ell(0)\|_{op} \le \tau\|\widetilde{\bW}^\ell -\bW^\ell(0)\|_{op} + (1-\tau)\|\bW^\ell - \bW^\ell(0)\|_{op}\le R$.
    Note that the above inequality holds for any $\bx\in\X$.
    The proof of the first part is complete.  

    We turn to show that the loss $l$ is locally almost convex near the initialization. 
    From the convexity of $l(\bW;z)$ (with respect to $f_{\bW}$), we know
\begin{align*}
     l(\widetilde{\bW};z ) -  l(\bW ;z )&\ge \frac{\partial l(\bW ;z)}{\partial f_{\bW }} \big(f_{\widetilde{\bW}}(\bx) - f_{\bW }(\bx)\big). 
\end{align*}
Then, according to the chain rule $\frac{\partial l(\bW ;z)}{\partial f_{\bW }}\cdot \nabla f_\bW(\bx) = \nabla l(\bW(k);z)$ and the fact $\frac{\partial l(\bW ;z)}{\partial f_{\bW }} = (f_\bW(\bx) - y)$, we get
\begin{align}\label{eq:weaklyconvex}
    l(\widetilde{\bW};z ) -  l(\bW;z )\nonumber &\ge  \frac{\partial l(\bW ;z)}{\partial f_{\bW }} \big(f_{\widetilde{\bW}}(\bx) -f_{\bW }(\bx) - \big\langle \nabla f_\bW(\bx), \widetilde{\bW}-\bW  \big\rangle + \big\langle \nabla f_\bW(\bx), \widetilde{\bW}-\bW  \big\rangle\big)\nonumber\\
    &= \frac{\partial l(\bW ;z)}{\partial f_{\bW }} \big\langle \nabla f_\bW(\bx), \widetilde{\bW}-\bW  \big\rangle + \big(f_{\bW }(\bx)- y \big)  \Big(f_{\widetilde{\bW}}(\bx) - f_{\bW }(\bx) - \big\langle \nabla f_\bW(\bx), \widetilde{\bW}-\bW  \big\rangle\big)\nonumber\\
    &\ge\big\langle \nabla l(\bW(k);z), \widetilde{\bW}-\bW  \big\rangle - \big|f_{\bW }(\bx)- y\big|\Big|f_{\widetilde{\bW}}(\bx) - f_{\bW }(\bx) - \big\langle \nabla f_\bW(\bx), \widetilde{\bW}-\bW  \big\rangle\big|. 
\end{align}

Plugging \eqref{eq:almost_linear} into \eqref{eq:weaklyconvex} yields
\begin{align*} 
    l(\widetilde{\bW};z ) -  l(\bW;z ) \ge\big\langle \nabla l(\bW(k);z), \widetilde{\bW}-\bW  \big\rangle - |f_{\bW}(\bx) - y|\cdot \epsilon
\end{align*}
with $\epsilon \le C_\sigma^L(R+1)\|\widetilde{\bW}-\bW\|^2_2\sqrt{\log(m/\delta)/m}.$
The second part of the lemma is proved.
\end{proof}


Let $H$ be a separable Hilbert space.
For $f\in H$, we define the operator $f\otimes f:H\to H$ by $(f\otimes f) g \mapsto \langle f, g\rangle_H f$.
\begin{lemma}\label{lem:wk-w0}
Suppose Assumption~\ref{ass:activation} holds.
 Let $\delta\in(0,1)$ and $\{\bW(k)\}_{k=1}^T$ be produced by GD iterates with   $T$ iterations based on $S$. 
 Assume $m\gtrsim C_\sigma^Ld(\eta T)^3\log(m/\delta)$ and $\eta \le C_\sigma^{-L} $ as a constant. 
     Then, with probability at least $1 - CL\exp(-m/2) - \delta$ over initialization $(\ba, \bW(0))$, for any $k=0,\ldots,T$, it holds that \[\|\bW(k)-\bW(0) \|_2^2  \le 4\eta k \ \text{ and } \ \L_S(\bW(k)) \le \L_S(\bW(k-1)) \le  \cdots \le \L_S(\bW(0)). \] 
\end{lemma}

\begin{proof}
The lemma is proved by induction.
It's obvious that $\|\bW(k) - \bW(0)\|_2^2 \le 0$ and $\L_S(\bW(k)) \le \L_S(\bW(k-1)) \le \cdots \le \L_S(\bW(0))$ hold with $k=0$.
Assume, for all $t\in [k]$ with $k\le T-1$, $\| \bW(t) - \bW(0) \|^2_2 \le 4\eta t$ and $\L_S(\bW(t)) \le \L_S(\bW(t-1)) \le \cdots \le \L_S(\bW(0))$ hold.
We will show that  $\|\bW(k+1) - \bW(0) \|^2_2 \le 4\eta(k+1)$ and $\L_S(\bW(k+1)) \le \L_S(\bW(k))$.

    From the update rule \eqref{eq:update-GD}, we know
    \begin{align}\label{eq:wk-w0}
       &\big\|\bW(k+1) - \bW(0)  \big\|_2^2\nonumber \\
       &= \big\|\bW(k) - \bW(0) - \bW(k) + \bW(k+1)  \big\|_2^2\nonumber\\
       &= \big\|\bW(k) - \bW(0)\big\|_2^2 + \big\|\bW(k+1) - \bW(k)\big\|_2^2 + 2\langle \bW(k) - \bW(0), \bW(k+1) - \bW(k) \big\rangle \nonumber\\ 
       &= \big\|\bW(k) - \bW(0)\big\|_2^2 +  \eta^2 \Big\|\frac{1}{n}\sum_{i=1}^n\big[(f^{\text{GD}}_{\bW(k)}(\bx_i) - y_i)\nabla f^{\text{GD}}_{\bW(k)}(\bx_i)\big]\Big\|_2^2  +  \frac{2\eta}{n} \Big\langle \bW(0) - \bW(k),  \sum_{i=1}^n \nabla l(\bW(k);z_i)    \Big\rangle\nonumber\\
       &\le \big\|\bW(k) - \bW(0)\big\|_2^2 + \frac{2\eta^2\L_S(\bW(k))}{n}\sum_{i=1}^n\big\|\nabla f^{\text{GD}}_{\bW(k)}(\bx_i)\big\|_2^2  + \frac{2\eta}{n} \Big\langle \bW(0) -  \bW(k),  \sum_{i=1}^n \nabla l(\bW(k);z_i)   \Big\rangle \nonumber\\
       &\le \big\|\bW(k) - \bW(0)\big\|_2^2 + C_\sigma^L\eta^2\L_S(\bW(k)) + \frac{2\eta}{n} \Big\langle \bW(0) - \bW(k),  \sum_{i=1}^n \nabla l(\bW(k);z_i)   \Big\rangle,
    \end{align}
    where in the first inequality we have used the Cauchy-Schwarz inequality, and the last inequality follows from Lemma~\ref{lem:Hessian} with $R = 2\sqrt{\eta T}$ by noting induction assumption $\|\bW(k) - \bW(0)\|_2\le 2\sqrt{\eta k} \le 2\sqrt{\eta T}$, which is guaranteed by the condition $m \gtrsim C_\sigma^Ld(\eta T)^3\log(m/\delta)$.

According to \eqref{eq:almost_convex} in Lemma \ref{lem:almost-convex} with $R = 2\sqrt{\eta T}$, $\bW = \bW(k)$, $\widetilde{\bW} = \bW(0)$ and $\|\bW(k) - \bW(0)\|_2 \le R$, and the induction assumption $\L_S(\bW(k)) \le \cdots \le \L_S(\bW(0)) = \frac{1}{2n}\sum_{i=1}^n y_i^2 \le 1/2$, we know
\begin{align*}
    &  \frac{2\eta}{n}\!\Big\langle \bW(0)- \bW(k),  \sum_{i=1}^n \nabla l(\bW(k);z_i)  \Big\rangle 
     \le 2\eta\big(\L_S(\bW(0)) - \L_S(\bW(k))\big) + 2\eta \epsilon\sum_{i=1}^n\frac{|f^{\text{GD}}_{\bW(k)}(\bx_i) - y_i|}{n}\\
    &\le 2\eta\big(\L_S(\bW(0)) - \L_S(\bW(k))\big) + 2\eta\epsilon\sqrt{2\L_S(\bW(k))} \le 2\eta\big(\L_S(\bW(0)) - \L_S(\bW(k))\big) + 2\eta\epsilon
\end{align*}
with $\epsilon \le C_\sigma^L(R+1)\|\bW(k) - \bW(0)\|^2_2\sqrt{\log(m/\delta)/m} \le C_\sigma^L(\eta T)^{3/2}\sqrt{\log(m/\delta)/m}$ by the induction assumption, and in the second inequality we have used the Cauchy-Schwarz inequality.

Plugging the above inequality  into \eqref{eq:wk-w0}, we get
\begin{align*}
     \big\|\bW(k+ 1) - \bW(0)\big\|_2^2 
    &\le \big\|\bW(k) - \bW(0)\big\|_2^2 + C_\sigma^L\eta^2 \L_S(\bW(k))  + 2\eta\big(\L_S(\bW(0)) - \L_S(\bW(k))\big) + 2\eta \epsilon\\
    &\le \big\|\bW(k) - \bW(0)\big\|_2^2 + \big(C_\sigma^L\eta^2 - 2\eta\big) \L_S(\bW(k)) + 2\eta \L_S(\bW(0)) + 2\eta \epsilon \\
    &\le \big\|\bW(k) - \bW(0)\big\|_2^2 + \eta + 2\eta \epsilon \le \big\|\bW(k) - \bW(0)\big\|_2^2 + 3\eta \le 4\eta k + 3\eta < 4\eta(k+1),
\end{align*}
where in the third inequality we have used $\eta \le C_\sigma^{-L} $ and $ 2\L_S(\bW(0)) = \frac{1}{n}\sum_{i=1}^ny_i^2 \le 1$ by noting $f_{\bW(0)} = 0$, and in the fourth inequality we have used $\epsilon\le1$ by condition $m\gtrsim C_\sigma^L d(\eta T)^3\log(m/\delta)$.

Next, we turn to show that $\L_S(\bW(k + 1)) \le \L_S(\bW(k))$.
By Taylor's theorem, there exists a $\widetilde{\bW}(k) = \tau \bW(k+1) + (1-\tau)\bW(k)$ with some $\tau\in[0,1]$ such that
\begin{align*}
    &\L_S(\bW(k+1)) - \L_S(\bW(k))\\
    &\le \big\langle\nabla \L_S(\bW(k)), \bW(k+1) - \bW(k)\big\rangle + \frac{1}{2}\big\|\nabla^2 \L_S(\widetilde{\bW}(k))\big\|_{op}\big\|\bW(k+1) - \bW(k)\big\|^2_2\\
    &= -\eta\big\|\nabla \L_S(\bW(k))\big\|_{2}^2 + \frac{1}{2}\Big\|\frac{1}{n}\sum_{i=1}^n\nabla f_{\widetilde{\bW}(k)}(\bx_i)\otimes\nabla f_{\widetilde{\bW}(k)}(\bx_i)  + \big(f_{\widetilde{\bW}(k)}(\bx_i) - y_i\big)\nabla^2 f_{\widetilde{\bW}(k)}(\bx_i)\Big\|_{op} \\& \quad \, \times \eta^2\big\|\nabla \L_S(\bW(k))\big\|^2_2\\
    &\le -\eta\big\|\nabla \L_S(\bW(k))\big\|_{2}^2 + \eta^2\|\nabla \L_S(\bW(k))\|^2_2\frac{1}{2n}\sum_{i=1}^n\Big(\big\|\nabla f_{\widetilde{\bW}(k)}(\bx_i)\big\|_{2}^2 + \big\|\nabla^2 f_{\widetilde{\bW}(k)}(\bx_i)\big\|_{op}\big|f_{\widetilde{\bW}(k)}(\bx_i) - y_i\big| \Big)\\
    &\le -\eta\big\|\nabla \L_S(\bW(k))\big\|_{2}^2 + \eta^2\|\nabla \L_S(\bW(k))\|^2_2\Big(C_\sigma^L + C_\sigma^L \frac{1}{n}\sum_{i=1}^n\big|f_{\widetilde{\bW}(k)}(\bx_i) - y_i\big|\Big)
\end{align*}
where the last inequality used Lemma~\ref{lem:Hessian} by noting that $\|\tildeW^\ell(k) - \bW^\ell(0)\|_{op} \le \tau \|\bW^\ell(k+1) - \bW^\ell(0)\|_{op} + (1-\tau)\|\bW^\ell(k) - \bW^\ell(0)\|_{op} \le R = 2\sqrt{\eta T}$ for all $\ell\in[L]$ and the assumption on $m$.

Define $\tildeo^L_k(\bx), o^L_k(\bx)$ the same way as $o^L_0(\bx)$ with $\bW(0)$ replaced by $\tildeW(k)$ and $\bW(k)$, respectively.
It holds that
\begin{align*}
    \frac{1}{n}\sum_{i=1}^n\big|f_{\widetilde{\bW}(k)}(\bx_i) - y_i\big| & \le \frac{1}{n}\sum_{i=1}^n\Big(\big|f_{\widetilde{\bW}(k)}(\bx_i) -f^{\text{GD}}_{\bW(k)}(\bx_i)\big| +\big| f^{\text{GD}}_{\bW(k)}(\bx_i) - y_i\big|\Big)\\
    &\le \sup_{\bx\in\X}\big|\ba^\top (\tildeo^L_k(\bx) - o^L_k(\bx))\big| + \frac{1}{n}\sum_{i=1}^n\big|f^{\text{GD}}_{\bW(k)}(\bx_i) - y_i\big|\\
    &\le \|\ba\|_2 \frac{C_\sigma^L\sum_{\ell=1}^L\|\tildeW^\ell(k) - \bW^\ell(k)\|_{op}}{\sqrt{m}} + \sqrt{2\L_S(\bW(k))}\\
    &\le \|\ba\|_2 \frac{C_\sigma^L\sum_{\ell=1}^L\|\nabla_\ell f^{\text{GD}}_{\bW(k)}\|_{op}}{\sqrt{m}} + 1 \le C_\sigma^L\frac{\|\ba\|_2}{\sqrt{m}} + 1 \le C_\sigma^L,
\end{align*}
where the third inequality used the Cauchy-Schwarz inequality and Lemma~\ref{lem:tildeo-o} by noting that both $\|\tildeW(k) - \bW(0)\|_{op}$ and $\|\bW(k) - \bW(0)\|_{op}$ are controlled by $\max\{\|\bW(k+1) - \bW(0)\|_{op}, \|\bW(k) - \bW(0)\|_{op}\}$, which is less than $R=2\sqrt{\eta T} \le \sqrt{m},$ and the fourth inequality used $\|\tildeW^\ell(k) - \bW^\ell(k)\|_{op} \le \|\bW^\ell(k+1) - \bW^\ell(k)\|_{op} = \|\nabla_\ell f_{\bW(k)}\|_{op}$ and the induction assumption $\L_S(\bW(k))\le \L_S(\bW(0)) \le 1/2,$ and the last inequality used Lemma~\ref{lemma:oprt_norm}.

Combining the above estimates, we have
\begin{align*}
    \L_S(\bW(k+1)) - \L_S(\bW(k)) &\le -\eta\big\|\nabla \L_S(\bW(k))\big\|_{2}^2 + C_\sigma^L\eta^2\|\nabla \L_S(\bW(k))\|^2_2\\
    &= \eta(C_\sigma^L\eta - 1)\big\|\nabla \L_S(\bW(k))\big\|_{2}^2 \le 0,
\end{align*}
where the last inequality follows from $\eta \le C_\sigma^{-L}.$
This completes the proof of the lemma.
\end{proof}


Recall that $f^{\text{lin,GD}}_{\bW(k)}(\bx) = f_{\bW(0)}(\bx) + \langle \nabla f_{\bW(0)}(\bx), \bW(k) - \bW(0)\rangle_2$ for any $\bx\in\X$.
The following lemma controls the gap between $f^{\text{GD}}_{\bW(k)}$ and its linear approximation $f^{\text{lin,GD}}_{\bW(k)}$ uniformly.
\begin{lemma}\label{prop:f-flin_gd}
Let $\delta\in(0,1).$
    Suppose Assumption~\ref{ass:activation} holds.
    Let $\{\bW(k)\}_{k=1}^T$ be produced by GD iterations with $T$ iterations based on $S$. 
    Assume $\eta \le C_\sigma^{-L}$, $m\gtrsim C_\sigma^Ld(\eta T)^3\log(m/\delta)$ and $\eta T\ge1$. 
    Then, with probability at least $1 - CL\exp(-m/2) -\delta$ over initialization $(\ba, \bW(0))$, for any $k=0,\ldots,T$ it holds that
    \[\big\|f^{\text{ GD}}_{\bW(k)} - f^{\text{lin,GD}}_{\bW(k)}\big\|_\rho^2   \le \big\|f^{\text{GD}}_{\bW(k)} - f^{\text{lin,GD}}_{\bW(k)}\big\|_\infty^2 \le C_\sigma^L(\eta T)^3\frac{\log(m/\delta)}{m} . \]
\end{lemma}
\begin{proof}
    The proof is obtained by combining Lemma \ref{lem:wk-w0} and \eqref{eq:almost_linear} in Lemma \ref{lem:almost-convex} with $\widetilde{\bW} = \bW(k)$, $\bW = \bW(0)$ and $R = 2\sqrt{\eta T}$.
\end{proof}

\subsubsection{(2) Control of $ \|f^{\text{lin,GD}}_{\bW(T)} - \bS_mg^{m,\text{GD}}_T \|_\rho^2$}\label{app:GD-term2} 
\begin{lemma}\label{lemma:<2}[Lemma 3 in \cite{carratino2018learning}]
    Let $\lambda > 0$, $\Gamma \in \mathbb{N}$ and $\delta \in (0,1)$. Let $\zeta_1,\dots,\zeta_\Gamma$ be independent and identically distributed random vectors bounded by $\kappa > 0$. Let $Q_\Gamma = \frac{1}{\Gamma}\sum_{i=1}^\Gamma \zeta_i \otimes \zeta_i$ and $Q$ be the expectation of $Q_\Gamma$.
    Then, for any $\lambda \geq \frac{9\kappa^2}{\Gamma} \log \frac{\Gamma}{\delta}$, with probability at least $1-\delta$ over sampling, there holds
$$\| (Q + \lambda \bfI)^{1/2}(Q_\Gamma + \lambda \bfI)^{-1/2}\|_{op}^{2} = \|(Q_\Gamma + \lambda \bfI)^{-1/2} (Q + \lambda \bfI)^{1/2}\|_{op}^{2} \leq 2.$$
\end{lemma}

The following lemma gives an estimate of the term $\big\|  f^{\text{lin,GD}}_{\bW(T)}  - \mathbf{S}_m g^{m,\text{GD}}_T \big\|_\rho^2$.
\begin{lemma}\label{prop:flin-gkm_gd}
    Let $\delta\in(0,1).$
    Suppose Assumption~\ref{ass:activation} holds.
    Let $\{\bW(k)\}_{k=1}^T$ be produced by GD iterations with $T$ iterations based on $S$. 
    Assume $\eta \le C_\sigma^{-L} $, $  \eta T \le n(C_\sigma^L\log(2n/\delta))^{-1}$ and $m\gtrsim C_\sigma^Ld(\eta T)^3\log(m/\delta)$. 
    Then, with probability at least $1 - L\exp(-Cm) - \delta$ over initialization $(\ba, \bW(0))$, for any $k=0,\ldots,T$ it holds that
    \begin{align*}
        \big\|f^{\text{lin,GD}}_{\bW(k)} - \bS_mg^{m,\text{GD}}_k\big\|_\rho \le C_\sigma^L(\eta T)^2\sqrt{\frac{\log(m/\delta)}{m}} 
    \end{align*}
    and 
    \begin{align*}
        \big\|f^{\text{lin,GD}}_{\bW(T)} - \bS_mg^{m,\text{GD}}_T\big\|_\rho^2 \le C_\sigma^L(\eta T)^4\frac{\log(m/\delta)}{m}.
    \end{align*}
\end{lemma}
\begin{proof}
Let $F_k = f^{\text{lin,GD}}_{\bW(k) } - g_k^{m,\text{GD}} ,\ \epsilon_k = f_{\bW(k+1)}^{\text{lin,GD}} - f_{\bW(k)}^{\text{lin,GD}} + \frac{\eta }{n}  \sum_{i=1}^n (f_{\bW(k) }^{\text{lin,GD}}(\bx_i) - y_i)K^m_{\bx_i} \in\H_m$ for any $k\in[T-1]$.
We define the self-adjoint positive operator $\widehat{\bSigma}_m = \frac{1}{n}\sum_{i=1}^n K^m_{\bx_i}\otimes K^m_{\bx_i}:\H_m\to\H_m$. 
From the update rule of $ g_{k}^{m,\text{GD}}$ (see \eqref{eq:GD-Km}), it holds that
\begin{align}\label{eq:recursive_FT}
    F_{k+1} &= \Big(f_{\bW(k)}^{\text{lin,GD}}  -  \frac{\eta }{n}  \sum_{i=1}^n (f_{\bW(k) }^{\text{lin,GD}}(\bx_i)  - y_i)K^m_{\bx_i}   + \epsilon_k\Big) - \Big(g^{m,\text{GD}}_k  -  \frac{\eta }{n}  \sum_{i=1}^n (g^{m,\text{GD}}_{ k }(\bx_i) -  y_i)K^m_{\bx_i}\Big) \nonumber\\
        &= \big(f_{\bW(k)}^{\text{lin,GD}} - g^{m,\text{GD}}_k\big) - \frac{\eta}{n}\sum_{i=1}^n\big(f_{\bW(k)}^{\text{lin,GD}}(\bx_i) - g^{m,\text{GD}}_k(\bx_i)\big)K^m_{\bx_i} + \epsilon_k \nonumber\\
        &= F_k - \frac{\eta }{n}  \sum_{i=1}^n \langle F_k,K^m_{\bx_i}\rangle_{\H_m} K^m_{\bx_i} + \epsilon_k = F_k - \frac{\eta }{n}  \sum_{i=1}^n  K^m_{\bx_i}\otimes K^m_{\bx_i} F_k + \epsilon_k = \big(\mathbf{I} - \eta \widehat{\bSigma}_m\big)F_k  + \epsilon_k, 
\end{align}
where the third equality follows from the fact $F_k = f^{\text{lin,GD}}_{\bW(k) } - g_k^{m,\text{GD}} \in \H_m$ and the reproducing kernel property that $f_{\bW(k)}^{\text{lin,GD}}(\bx_i) - g^{m,\text{GD}}_k(\bx_i) = \langle f_{\bW(k)}^{\text{lin,GD}}- g^{m,\text{GD}}_k,K^m_{\bx_i}\rangle_{\H_m} = \langle F_k,K^m_{\bx_i}\rangle_{\H_m}$.

Applying \eqref{eq:recursive_FT} recursively, we get
\begin{equation}\label{eq:recursive-F}
    F_{k+1}= \sum_{s=0}^{k}\big(\mathbf{I} - \eta \widehat{\bSigma}_m\big)^s\epsilon_{k-s}.
\end{equation}

Define the expectation of $\widehat{\bSigma}_m$ by $\bSigma_m := \ebb[\widehat{\bSigma}_m] = \int_\X K^m_\bx\otimes K^m_\bx d\rho_\bx(\bx):\H_m\to\H_m$.
Mercer's Theorem \citep{steinwart2008support} implies $\|\bS_mf\|_\rho = \|\bSigma_m^{\frac{1}{2}}f\|_{\H_m}$ for any $f\in \H_m$.
From Lemma~\ref{lem:Hessian} we know $\|K^m_{\bx_i}\|_{\H_m} = \sqrt{K^m(\bx_i,\bx_i)} \le \sqrt{\|K^m\|_\infty} \le \sup_{\bx\in\X}\|\nabla f_{\bW(0)}(\bx)\|_2 \le C_\sigma^L$.
Therefore, Lemma \ref{lemma:<2}  with $\zeta_i = K^m_{\bx_i}$, $\Gamma = n$ and $\kappa = C_\sigma^L$ yields $\|(\bSigma_m+\lambda \mathbf{I})^{\frac{1}{2}}(\widehat{\bSigma}_m+\lambda \mathbf{I})^{-\frac{1}{2}}\|_{op}\le 2$ with probability at least $1-\delta/2$ over the sampling if $\lambda > \frac{C_\sigma^L}{n}\log(\frac{2n}{\delta})$.
Then, according to \eqref{eq:recursive-F}, we get
\begin{align}\label{eq:bound_h_T}
    &\|\bS_mF_{k}\|_{\rho} = \big\|\bSigma_m^{\frac{1}{2}}F_k\big\|_{\H_m} \le \big\|(\bSigma_m+\lambda \mathbf{I})^{\frac{1}{2}}F_k\big\|_{\H_m}\nonumber\nonumber\\
    &\le \big\|(\bSigma_m+\lambda \mathbf{I})^{\frac{1}{2}}(\widehat{\bSigma}_m+\lambda \mathbf{I})^{-\frac{1}{2}}\big\|_{op}\big\|(\widehat{\bSigma}_m+\lambda \mathbf{I})^{\frac{1}{2}}F_k\big\|_{\H_m}
    \le 2\big\|\widehat{\bSigma}_m^{\frac{1}{2}}F_k\big\|_{\H_m} + 2\sqrt{\lambda}\big\|F_k\big\|_{\H_m}\nonumber\\
    &=2\eta^{-\frac{1}{2}}\bigg\|\sum_{s=0}^{k-1}\big(\eta \widehat{\bSigma}_m\big)^\frac{1}{2}\big(\mathbf{I} - \eta \widehat{\bSigma}_m\big)^s\epsilon_{k-s-1}\bigg\|_{\H_m} + 2\sqrt{\lambda}\bigg\|\sum_{s=0}^{k-1}\big(\mathbf{I} - \eta \widehat{\bSigma}_m\big)^s\epsilon_{k-s-1}\bigg\|_{\H_m}\nonumber\\
    &\le 2\eta^{-\frac{1}{2}}\sum_{s=0}^{k-1}\Big\|\big(\eta \widehat{\bSigma}_m\big)^\frac{1}{2}\big(\mathbf{I} - \eta \widehat{\bSigma}_m\big)^s\Big\|_{op}\big\|\epsilon_{k-s-1}\big\|_{\H_m}+2\sqrt{\lambda}\sum_{s=0}^{k-1}\Big\|\big(\mathbf{I} - \eta \widehat{\bSigma}_m\big)^s\Big\|_{op}\big\|\epsilon_{k-s-1}\big\|_{\H_m}.
\end{align}

For any $a\in[0,1)$ and any $s\in\mathbb{N}$, it can be easily computed that $\sup_{t\in[0,1]}t^a(1-t)^s \le \big(\frac{a}{a+s}\big)^a$.
Here, we take notation $0^0 = 1$.
From Lemma \ref{lem:Hessian} and $\eta \le C_\sigma^{-L}$, we know $\eta\|\widehat{\bSigma}_m\|_{op} \le \frac{\eta}{n}\sum_{j=1}^n\|K^m_{\bx_i}\otimes K^m_{\bx_i}\|_{op} = \frac{\eta}{n}\sum_{j=1}^n\|K^m_{\bx_i}\|^2_{\H_m} \le \eta \|K^m\|_\infty \le \eta \sup_{\bx\in\X}\|\nabla f_{\bW(0)}(\bx)\|_2^2 \le \eta C_\sigma^L \le 1$. Then, there holds
\begin{align*}
    &\sum_{s=0}^{k-1}\big\|\big(\eta \widehat{\bSigma}_m\big)^a\big(\mathbf{I} - \eta \widehat{\bSigma}_m\big)^s\big\|_{op} \le \sum_{s=0}^{k-1}\sup_{t\in[0,1]}t^a(1-t)^s \le \sum_{s=0}^{k-1}\Big(\frac{a}{a+s}\Big)^a= 1 + a^a\sum_{s=1}^{k-1}\Big(\frac{1}{a+s}\Big)^a\\
    &\le1 +a^a\sum_{s=1}^{k-1}\int_{s-1}^{s}\Big(\frac{1}{a+x}\Big)^adx = 1 +a^a\int_0^{k-1}\Big(\frac{1}{a+x}\Big)^adx \\
    &\le 1 + \frac{a^a}{1-a}\big((k+a-1)^{1-a} - a^{1-a}\big) \le 1 + \frac{(k+a-1)^{1-a}}{1-a}.
\end{align*}
Combining \eqref{eq:bound_h_T} and the above inequality with $a=\frac{1}{2}$ and $a=0$, respectively,  we have
\begin{align}\label{eq:estimate_F_T}
    \big\|\bS_mF_k\big\|_\rho & \le \bigg(2\eta^{-\frac{1}{2}}\!\sum_{s=0}^{k-1}\Big\|\big(\eta \widehat{\bSigma}_m\big)^\frac{1}{2}\big(\mathbf{I} - \eta \widehat{\bSigma}_m\big)^s\Big\|_{op}\!\!+ 2\sqrt{\lambda}\sum_{s=0}^{k-1}\Big\|\big(\mathbf{I} - \eta \widehat{\bSigma}_m\big)^s\Big\|_{op}\bigg)\!\max_{s\in[k-1]}\|\epsilon_{s}\|_{\H_m}\nonumber\\
    &\le \Big(2\eta^{-\frac{1}{2}}\big(1 + 2\sqrt{k}\big) + 2\sqrt{\lambda} k \Big)\max_{s\in[k-1]}\|\epsilon_{s}\|_{\H_m}.
\end{align}

It remains to  estimate $\|\epsilon_k\|_{\H_m}$. From the definition of $f^{\text{lin,GD}}_{\bW }$ and the update rule of GD \eqref{eq:update-GD}, there holds
    \begin{align*}
        \epsilon_k(\bx) &= f^{\text{lin,GD}}_{\bW(k+1) }(\bx)  - f^{\text{lin,GD}}_{\bW(k)} (\bx) + \frac{\eta}{n}\sum_{i=1}^n\big(f^{\text{lin,GD}}_{\bW(k)} (\bx_i) - y_i\big)K^m_{\bx_i}(\bx)\nonumber\\
        &= \Big\langle \nabla f_{\bW(0)}(\bx), \bW(k+1)- \bW(k) \Big\rangle + \frac{\eta}{n}\sum_{i=1}^n\big(f^{\text{lin,GD}}_{\bW(k)} (\bx_i) - y_i\big)\Big\langle \nabla f_{\bW(0)}(\bx_i), \nabla f_{\bW(0)}(\bx) \Big\rangle \nonumber\\
        & = \frac{\eta}{n}\sum_{i=1}^n\bigg[ - \big(f^{\text{GD}}_{\bW(k)}(\bx_i) -  y_i\big)\Big\langle \nabla f^{\text{GD}}_{\bW(k)}(\bx_i), \nabla f_{\bW(0)}(\bx) \Big\rangle + \big(f^{\text{lin,GD}}_{\bW(k)}(\bx_i)  -  y_i\big) \Big\langle \nabla f_{\bW(0)}(\bx_i), \nabla f_{\bW(0)}(\bx) \Big\rangle\bigg] \nonumber\\
        &= \bigg\langle \frac{\eta}{n}\sum_{i=1}^n  \Big[\big(y_i - f^{\text{GD}}_{\bW(k)}(\bx_i)\big)\nabla f^{\text{GD}}_{\bW(k)}(\bx_i) + \big(f^{\text{lin,GD}}_{\bW(k)}(\bx_i) - y_i\big)\nabla f_{\bW(0)}(\bx_i)\Big] , \nabla f_{\bW(0)}(\bx)\bigg\rangle \nonumber\\
        & =: \Big\langle \Delta(k), \nabla f_{\bW(0)}(\bx) \Big\rangle,
    \end{align*}
    where the second equality is due to $K^m_{\bx_i}(\bx) = K^m(\bx_i,\bx)$, the third equality is according to the update rule \eqref{eq:update-GD}, and in the last equality we define
    \begin{align*}
         \Delta(k) : &= \frac{\eta}{n}\sum_{i=1}^n  \Big[\big(y_i - f^{\text{GD}}_{\bW(k)}(\bx_i)\big)\nabla f^{\text{GD}}_{\bW(k)}(\bx_i) + \big(f^{\text{lin,GD}}_{\bW(k)}(\bx_i) - y_i\big)\nabla f_{\bW(0)}(\bx_i)\Big]\\
        &= \frac{\eta}{n} \sum_{i=1}^n   \Big[\big(y_i  -  f^{\text{GD}}_{\bW(k)}(\bx_i)\big)\Big(\nabla f^{\text{GD}}_{\bW(k)}(\bx_i)  - \nabla f_{\bW(0)}(\bx_i)\Big) +  \big(f^{\text{lin,GD}}_{\bW(k)}(\bx_i)  -  f^{\text{GD}}_{\bW(k)}(\bx_i)\big)\nabla f_{\bW(0)}(\bx_i)\Big].
    \end{align*}
    
    Note that for any $f\in \H_m$,  $
    \|f\|_{\H_m} = \inf\big\{\|\bW\|_2 : \bW \in \W \text{ with } f(\bx) = \big\langle \bW, \nabla f_{\bW(0)}(\bx)\big\rangle\big\}.$ We control $\|\epsilon_k\|_{\H_m}$ as follows
    \begin{align}\label{eq:epsilon1}
        &\|\epsilon_k\|_{\H_m} \le \|\Delta(k)\|_2\nonumber\\
        &\le \frac{\eta}{n}\sum_{i=1}^n  \Big[\big|y_i - f^{\text{GD}}_{\bW(k)}(\bx_i)\big|\Big\|\nabla f^{\text{GD}}_{\bW(k)}(\bx_i) - \nabla f_{\bW(0)}(\bx_i)\Big\|_2 + \big|f^{\text{lin,GD}}_{\bW(k)}(\bx_i) - f^{\text{GD}}_{\bW(k)}(\bx_i)\big| \Big\|\nabla f_{\bW(0)}(\bx_i)\Big\|_2\Big] \nonumber\\
        &\le \eta\Big( \sqrt{2\L_S(\bW(k))} \sup_{\bx\in\X}\Big\|\nabla f^{\text{GD}}_{\bW(k)}(\bx) - \nabla f_{\bW(0)}(\bx)\Big\|_2 + \frac{1}{n}\sum_{i=1}^n\big|f^{\text{lin,GD}}_{\bW(k)}(\bx_i) - f^{\text{GD}}_{\bW(k)}(\bx_i)\big| \Big\|\nabla f_{\bW(0)}(\bx_i)\Big\|_2\Big) \nonumber\\
        &\le \eta \Big( \sqrt{2\L_S(\bW(k))} \sup_{\bx\in\X}\Big\|\nabla f^{\text{GD}}_{\bW(k)}(\bx) - \nabla f_{\bW(0)}(\bx)\Big\|_2 + \big\|f^{\text{lin,GD}}_{\bW(k)} - f^{\text{GD}}_{\bW(k)}\big\|_\infty\sup_{\bx\in\X}\Big\|\nabla f_{\bW(0)}(\bx)\Big\|_2\Big),
    \end{align}
    where in the last second inequality we have used Cauchy-Schwarz inequality.

    From Lemma \ref{lem:Hessian}, we know $\sup_{\bx\in\X}\|\nabla f_{\bW(0)}(\bx)\|_2 \le C_\sigma^L$.
    Note we choose $R = 2\sqrt{\eta T}$.
    From Lemma \ref{lem:wk-w0} we know $\L_S(\bW(k)) \le \L_S(\bW(0)) \le 1/2$
    and $\|\bW(k)  - \bW(0)\|_2 \le R$ for any $k\in[T]$.
    Further, from Lemma \ref{lem:Hessian} with $R = 2\sqrt{\eta T}$ and the fundamental theorem of calculus, it holds that 
    \begin{align*}
        &\big\|\nabla f^{\text{GD}}_{\bW(k)}(\bx) - \nabla f_{\bW(0)}(\bx)\big\|_2  = \Big\|\int_0^1 \nabla^2 f_{\bW(0) + t(\bW(k) - \bW(0))}(\bx) \big(\bW(k) - \bW(0)\big) dt\Big\|_2\\
        &\le \sup_{\|\bW-\bW(0)\|_2\le R}\big\|\nabla^2 f_{\bW}(\bx)\big\|_{op}\big\|\bW(k)-\bW(0)\big\|_2 \le C_\sigma^L\sqrt{\frac{(\eta T)^2\log(m/\delta)}{m}}
    \end{align*}
    According to Lemma~\ref{prop:f-flin_gd}, we know $\|f^{\text{GD}}_{\bW(k)} - f_{\bW(k)}^{\text{lin,GD}}\|_\infty \le  C_\sigma^L(\eta T)^{\frac{3}{2}}\sqrt{\log(m/\delta)/m}$.
    Plugging the above observations into \eqref{eq:epsilon1}, we know with probability at least $1 - CL\exp(-m/2) - \delta/2$ over initialization $(\ba,\bW(0))$, there holds
    \begin{align*}
     \|\epsilon_k\|_{\H_m}\le  C_\sigma^L\eta^\frac{5}{2} T^{\frac{3}{2}}\sqrt{\frac{\log(m/\delta)}{m}}.
    \end{align*}
    
Putting the estimate of $\|\epsilon_s\|_{\H_m}$ back into \eqref{eq:estimate_F_T} and setting $\lambda = (\eta T)^{-1}$, with a little abuse of notation (we regard $f^{\text{lin,GD}}_{\bW(k)}$ as a function in $\mathcal{L}^2_{\rho_\bx}$ in the following first term), the following inequality holds with probability at least $1-CL\exp(-m/2) - \delta$ over the initialization $(\ba,\bW(0))$ and sampling
\begin{align*}
    \|f^{\text{lin,GD}}_{\bW(k) } - \bS_mg_k^{m,\text{GD}}\|_\rho&=\|\bS_mF_k\|_\rho \le 7\eta^{-\frac{1}{2}}\sqrt{T}\max_{s\in[T-1]}\|\epsilon_s\|_{\H_m} \le C_\sigma^L(\eta T)^2\sqrt{\frac{\log(m/\delta)}{m}}.
\end{align*}
This completes the proof of the lemma.
\end{proof}

\subsubsection{(3) Control of $\|\bS_mg^{m,\text{GD}}_T - \bS g_T^{\text{GD}} \|_\rho^2$}\label{app:GD-term3}

Based on the kernel GD defined in \eqref{eq:kernel_GD_HK} with kernel $K$, for any $k\in\mathbb{N}$ we can define two outputs of the kernel GD $\hat{g}_k$ and $\tilde{g}_k$, which follow the same update rule as \eqref{eq:kernel_GD_HK} with kernel $K$ replaced by $\widehat{K}$ and $\widetilde{K}$, respectively.
i.e., $\hat{g}_{k+1} = \hat{g}_k - \frac{\eta}{n}\sum\nolimits_{i=1}^n(\hat{g}_k(\bx_i)  -  y_i) \widehat{K}_{\bx_i}, \tilde{g}_{k+1} = \tilde{g}_k - \frac{\eta}{n}\sum\nolimits_{i=1}^n(\tilde{g}_k(\bx_i)  -  y_i) \widetilde{K}_{\bx_i}$ with $\hat{g}_0 = \tilde{g}_0 = 0.$
We denote $\hat{\bg}_k = (\hat{g}_k(\bx_1),\ldots,\hat{g}_k(\bx_n))^\top\in\R^n$, $\tilde{\bg}_k = (\tilde{g}_k(\bx_1),\ldots,\tilde{g}_k(\bx_n))^\top\in\R^n$ and $\by = (y_1,\ldots,y_n)^\top \in\R^n$.
We denote $\widehat{\bK} = (\widehat{K}(\bx_i,\bx_j))_{i,j=1}^n, \widetilde{\bK} = (\widetilde{K}(\bx_i,\bx_j))_{i,j=1}^n \in \R^{n\times n}$ as Gram matrices with kernels $\widehat{K}$ and $\widetilde{K}$, respectively.

\begin{lemma}\label{lem:two_kernel_GD}\label{lem:key-finding}
    For $k\in[T]$, let $\hat{g}_k$ and $\tilde{g}_k$ be the outputs of Kernel GD with kernels $\widehat{K}$ and $\widetilde{K}$, respectively.
    Suppose $\|\widehat{K}\|_\infty, \|\widetilde{K}\|_\infty\le F$ for some $F \ge 1$.
    Assume $\eta T\ge 1$ and $\eta \le F^{-1}$.
    Then, for all $k\in[T]$, it holds that
    \begin{align*}
        \big\|\hat{g}_{k} - \tilde{g}_{k}\big\|_\rho \le \big\|\hat{g}_{k} - \tilde{g}_{k}\big\|_\infty \le 3F(\eta T)^2\|\widehat{K}-\widetilde{K}\|_\infty.
    \end{align*}
\end{lemma}

\begin{proof}
    We first estimate $\|\hat{\bg}_{k+1} - \tilde{\bg}_{k+1}\|_2$ for $k\in[T-1]$, which measures the distance between $\hat{g}_{k+1}$ and $\tilde{g}_{k+1}$ on the sample.
    According to the definitions of $\hat{g}_k$, $\tilde{g}_k$, the gram matrices $\widehat{\bK},\widetilde{\bK}$, and the fact $\widehat{K}_\bx(\bx') = \widehat{K}(\bx,\bx'), 
    \widetilde{K}_\bx(\bx') = \widetilde{K}(\bx,\bx')$ for any $\bx,\bx'\in\X$, it holds that
    \begin{align}\label{eq:update-Gmk-Fk}
        \hat{\bg}_{k+1} = \hat{\bg}_k - \frac{\eta}{n}\widehat{\bK}(\hat{\bg}_k-\by) \text{ and } \tilde{\bg}_{k+1} = \tilde{\bg}_k - \frac{\eta}{n}\widetilde{\bK}(\tilde{\bg}_k-\by).
    \end{align}
    Then, we have
    \begin{align}
        \hat{\bg}_{k+1} - \tilde{\bg}_{k+1} &= \hat{\bg}_k - \tilde{\bg}_k - \frac{\eta}{n}\big(\widehat{\bK}(\hat{\bg}_k-\by) - \widetilde{\bK}(\tilde{\bg}_k - \by)\big)\nonumber\\
        &= \hat{\bg}_k - \tilde{\bg}_k - \frac{\eta}{n}\big(\widehat{\bK}(\hat{\bg}_k-\tilde{\bg}_k) - (\widetilde{\bK} - \widehat{\bK})(\tilde{\bg}_k - \by)\big) = \Big(\bfI - \frac{\eta}{n}\widehat{\bK}\Big)(\hat{\bg}_k - \tilde{\bg}_k) + \frac{\eta}{n}(\widetilde{\bK} - \widehat{\bK})(\tilde{\bg}_k - \by)\nonumber.
    \end{align}
    Applying the above equality recursively, we have
    \begin{align}\label{eq:Gmk-Fk}
        \big\|\hat{\bg}_{k+1} - \tilde{\bg}_{k+1}\big\|_2 &= \Big\|\frac{\eta}{n}\sum_{s=0}^k\Big(\bfI - \frac{\eta}{n}\widehat{\bK}\Big)^s(\widetilde{\bK} -\widehat{\bK})(\tilde{\bg}_{k-s} - \by)\Big\|_2 \le \frac{\eta}{n}\sum_{s=0}^k\Big\|\bfI - \frac{\eta}{n}\widehat{\bK}\Big\|_{op}^s\|\widetilde{\bK} -\widehat{\bK}\|_{op}\|\tilde{\bg}_{k-s} - \by\|_2.
    \end{align}

    Since $\|\widehat{K}\|_\infty \le F$,
    Then, for any $\boldsymbol{\alpha} = (\alpha_1,\ldots,\alpha_n)^\top\in\R^n$ with $\|\boldsymbol{\alpha}\|_2 = 1$, there holds $\boldsymbol{\alpha}^\top\widehat{\bK}\boldsymbol{\alpha} = \|\sum_{i=1}^n\alpha_i\widehat{K}_{\bx_i}\|^2_{\H_m}\le $ $(\sum_{i=1}^n|\alpha_i|\|\widehat{K}_{\bx_i}\|_{\H_m})^2 \le (\sum_{i=1}^n|\alpha_i|\|\widehat{K}\|_{\infty}^\frac{1}{2})^2\le F(\sum_{i=1}^n|\alpha_i|)^2 \le F n$.
    This implies that $\|\widehat{\bK}\|_{op} \le F n$.
    Since $\eta \le F^{-1}$ and $\widehat{\bK}$ is PSD, we know $\|\bfI - \frac{\eta}{n}\widehat{\bK}\|_{op}\le1$.

    We can control  $\|\widehat{\bK}-\widetilde{\bK}\|_{op}$ as follows by
    \begin{align*}
        \|\widehat{\bK}-\widetilde{\bK}\|_{op} &= \sup_{\|\boldsymbol{\alpha}\|_2 = 1} |\boldsymbol{\alpha}^\top (\widehat{\bK} - \widetilde{\bK})\boldsymbol{\alpha}| = \sup_{\|\boldsymbol{\alpha}\|_2 = 1} \Big|\sum_{i,j=1}^n \alpha_i\alpha_j\big(\widehat{K}(\bx_i,\bx_j) - \widetilde{K}(\bx_i,\bx_j)\big)\Big|\\
        &\le  \|\widehat{K}-\widetilde{K}\|_\infty\sup_{\|\boldsymbol{\alpha}\|_2 = 1}\sum_{i,j=1}^n |\alpha_i\alpha_j| = \|\widehat{K}-\widetilde{K}\|_\infty\sup_{\|\boldsymbol{\alpha}\|_2 = 1}\Big(\sum_{i=1}^n |\alpha_i|\Big)\Big(\sum_{j=1}^n|\alpha_j|\Big)\\
        &\le n\|\widehat{K}-\widetilde{K}\|_\infty.
    \end{align*}    
    Further, from \eqref{eq:update-Gmk-Fk}, we know $\tilde{\bg}_{k} = (\bfI - \frac{\eta}{n}\widetilde{\bK})\tilde{\bg}_{k-1} + \frac{\eta}{n}\widetilde{\bK}\by$.
    Recursively applying this equation, we get $\tilde{\bg}_{k} = \frac{\eta}{n}\sum_{s=0}^{k-1}(\bfI - \frac{\eta}{n}\widetilde{\bK})^s\widetilde{\bK}\by$.
    Analogously to the estimate of $\|\bfI - \frac{\eta}{n}\widehat{\bK}\|_{op}$, we can show that $\|\widetilde{\bK}\|_{op} \le Fn$ and $\|\bfI - \frac{\eta}{n}\widetilde{\bK}\|_{op}\le 1$ by noting $\eta \le F^{-1}$ and $\|\widetilde{K}\|_\infty\le F$.
    Then, it holds that
    \begin{align}\label{eq:tilde_g_norm}
        \|\tilde{\bg}_k\|_2 &\le \Big\|\sum_{s=0}^{k-1}\Big(\bfI - \frac{\eta}{n}\widetilde{\bK}\Big)^s\frac{\eta}{n}\widetilde{\bK}\Big\|_{op}\|\by\|_2 \le  \sqrt{n}\sup_{t\in[0,1]}\Big|\sum_{s=0}^{k-1}(1-t)^st\Big| =\sqrt{n}\sup_{t\in[0,1]}\big(1 - (1-t)^k\big) \le \sqrt{n}.
    \end{align}
    
    Plugging the above estimates back into \eqref{eq:Gmk-Fk}, we have
    \begin{align}\label{eq:hat_g-tilde_g}
        \big\|\hat{\bg}_{k+1} - \tilde{\bg}_{k+1}\big\|_2 &\le \frac{\eta}{n}\sum_{s=0}^k\|\widehat{\bK}-\widetilde{\bK}\|_{op}\|\big(\|\tilde{\bg}_{k-s}\|_2 + \|\by\|_2\big) \le 2\eta T\sqrt{n}\|\widehat{K}-\widetilde{K}\|_\infty.
    \end{align}

    Now, we turn to control $\|\hat{\bg}_{k+1} - \tilde{\bg}_{k+1}\|_\infty$.
    For any $\bx\in\X$ and $k\in[T-1]$, we have
    \begin{align*}
        &|\hat{g}_{k+1}(\bx) - \tilde{g}_{k+1}(\bx)|\nonumber\\
        &\!=\! \Big|\hat{g}_{k}(\bx) \!-\! \tilde{g}_{k}(\bx) \!-\! \frac{\eta}{n}\!\sum_{i=1}^n\!\big[\big(\hat{g}_k(\bx_i) \!-\! \tilde{g}_k(\bx_i)\big)\widehat{K}(\bx_i,\bx) \!+\!\big(\tilde{g}_k(\bx_i) \!-\! y_i\big)\big(\widehat{K}(\bx_i,\bx) \!-\! \widetilde{K}(\bx_i,\bx)\big)\big]\Big| \nonumber\\
        &\!\le\! \big|\hat{g}_{k}(\bx) - \tilde{g}_{k}(\bx)\big| + \frac{\eta}{n}\sum_{i=1}^n\Big(\|\widehat{K}\|_\infty\big|\hat{g}_k(\bx_i) - \tilde{g}_k(\bx_i)\big| + \|\widehat{K}-\widetilde{K}\|_\infty\big|\tilde{g}_k(\bx_i) - y_i\big|\Big)\nonumber\\
        &\!\le\!\big|\hat{g}_{k}(\bx) - \tilde{g}_{k}(\bx)\big| + \frac{\eta}{\sqrt{n}}\big(\|\widehat{K}\|_\infty\|\hat{\bg}_k-\tilde{\bg}_k\|_2 + \|\widehat{K}-\widetilde{K}\|_\infty\|\tilde{\bg}_k-\by\|_2\big),
    \end{align*}
    where the last inequality used Cauchy-Schwarz inequality.
    
    Plugging  \eqref{eq:tilde_g_norm} and \eqref{eq:hat_g-tilde_g} back into the above inequality and noting $\bx\in\X$ is arbitrary, we have
    \begin{align*}
        \|\hat{g}_{k+1} - \tilde{g}_{k+1}\|_\infty \le \|\hat{g}_{k} - \tilde{g}_{k}\|_\infty + 3F\eta^2T\|\widehat{K}-\widetilde{K}\|_\infty.
    \end{align*}
    Applying the above inequality recursively and noting that $\hat{g}_{0} = \tilde g_{0}$, we have
    \begin{align*}
        \|\hat{g}_{k+1} - \tilde{g}_{k+1}\|_\infty \le 3F(\eta T)^2\|\widehat{K}-\widetilde{K}\|_\infty.
    \end{align*}
By further noting that
    \begin{align*}
        \big\|\hat{g}_k - \tilde{g}_k\big\|_\rho^2 &=  \int_\X |\hat{g}_k(\bx) - \tilde{g}_k(\bx)|^2 d\rho_\X(\bx)  \le \|\hat{g}_k - \tilde{g}_k\|_\infty^2.
    \end{align*}
  The proof of this lemma is complete.
\end{proof}

To control this term, we establish a uniform concentration result for the NTK of a DNN under non-symmetric Gaussian initialization.

More precisely, for every $\ell \in [L]$, the entries of
$\mathbf{W}^{\ell}(0)$ and the output layer $\ba$ are independently drawn from
$\mathcal{N}(0,1)$.
Under this initialization, the layerwise gradients $\nabla_{\ell}f_{\mathbf{W}(0)}(\mathbf{x})$
are generally nonzero also for $\ell\in[L-1]$.
We define the finite-width NTK at each layer $\ell\in[L]$ as
\begin{align}
    K^{m,\ell}(\bx, \bx') &= \big\langle \nabla_\ell f_{\bW(0)}(\bx), \nabla_\ell f_{\bW(0)}(\bx')\big\rangle  = \langle o^{\ell-1}_0(\bx),o^{\ell-1}_0(\bx')\rangle \, \ba^\top \bV^\ell_{L,0}(\bx)^\top\bV^\ell_{L,0}(\bx')\ba .\nonumber
\end{align}
In the following theorem, we show that $K^{m,\ell}$ converges to an infinite-width NTK $K^\ell:\X\times\X\to\R$ uniformly over $\X\times\X$, which is given by
\begin{align*}
    K^\ell(\bx,\bx') = c_\sigma\ebb[\sigma(U^{\ell-1}(\bx))\sigma(U^{\ell-1}(\bx'))] \prod_{h=\ell}^{L}q^h(\bx,\bx'),
\end{align*}
where the definition of $U^\ell(\bx)$ and $q^h(\bx)$ can be found in Appendix~\ref{appde:problem}.

The result is technically nontrivial and achieves a substantially sharper dependence on the network width than previously available bounds on DNN with ReLU activation.
Since its proof is lengthy and technically involved, we defer both the proof and a detailed comparison with existing NTK concentration results to Appendix~\ref{app:NTK}.
\begin{theorem}\label{thm:concentration}
    Let $\delta\in(0,1)$. 
 Suppose Assumption~\ref{ass:activation} holds, and $m \gtrsim C_\sigma^L d\log({m}/{\delta})$. 
    With probability at least $1-L^2\exp(-C_\sigma \log^2(m)) - \delta$ over initialization $(\ba,\bW(0))$, for all $\ell\in[L]$, it holds that
    \begin{align*}
        \|K^{m,\ell} - K^\ell \|_\infty \lesssim C_\sigma^L\sqrt{\frac{d\log^2(m)}{m}}.
    \end{align*}
\end{theorem}


Although Theorem~19 is stated under non-symmetric Gaussian
initialization, the same uniform concentration bound holds under the
symmetric initialization adopted for GD and SGD.
Indeed, the proof follows the same covering and concentration argument,
with only a simpler modification caused by the paired structure of the
last layer.
Hence, with probability at least $1-L^2\exp\big(-C_\sigma d\log^2(m)\big)-\delta$ over initialization $(\ba,\bW(0)),$ we have
\[
\|K^m-K\|_\infty
\lesssim
C_\sigma^L
\sqrt{\frac{d\log^2(m)}{m}}.
\]

Now, we give the estimate of $\big\|\bS_m g^{m,\text{GD}}_k - \bS g_k^{\text{GD}}\big\|_\rho $. 
\begin{lemma}\label{prop:gmk-gk}
    Let $\delta\in(0,1)$.
    Suppose Assumption~\ref{ass:activation} holds.
    Assume $m \gtrsim C_\sigma^L dL\log^2({m}/{\delta})$ and $\eta \le C_\sigma^{-1}$.
    With probability at least $1-L^2\exp(-C_\sigma \log^2(m)) - \delta$ over initialization $(\ba,\bW(0))$, for all \(k\in[T]\), it holds that
    \begin{align*}
        \big\|\bS_m g^{m,\text{GD}}_k - \bS g_k^{\text{GD}}\big\|_\rho \le \big\|g^{m,\text{GD}}_k - g_k^{\text{GD}}\big\|_\infty \le C_\sigma(\eta T)^2\sqrt{\frac{d\log^2(m)}{m}} 
    \end{align*}
    and
    \begin{align*}
        \big\|\bS_m g^{m,\text{GD}}_T - \bS g_T^{\text{GD}}\big\|_\rho^2   \le \frac{C_\sigma(\eta T)^4  d\log^2(m)}{m} .  
    \end{align*}
\end{lemma}

\begin{proof}
    By the uniform NTK concentration bound under the symmetric initialization stated above, the condition $m \gtrsim C_\sigma^L dL\log^2({m}/{\delta})$ and \eqref{eq:bound_K}, we know 
    \begin{align*}
        \big\|K^m\big\|_\infty \le \big\|K^m - K\big\|_\infty + \big\|K\big\|_\infty \le 1 + \|K\|_\infty \le C_\sigma.
    \end{align*}
    Then, combining the symmetric NTK concentration bound above
with Lemma~\ref{lem:two_kernel_GD} (with $\widehat{K} = K^m$, $\widetilde{K} = K$ and $F =  C_\sigma$) yields that
    \begin{align*}
        \big\|\bS_mg^{m,\text{GD}}_k - \bS g_k^{\text{GD}}\big\|_\rho \le \big\|g^{m,\text{GD}}_k - g^{\text{GD}}_k\big\|_\infty \le C_\sigma(\eta T)^2\sqrt{\frac{d\log^2(m)}{m}}.
    \end{align*}
    Taking $k=T$ completes the proof.
\end{proof}

\subsubsection{Proof of Theorem \ref{pro:f-g_T}}\label{app:GD-thm1}
\begin{proof}[Proof of Theorem \ref{pro:f-g_T}]
Note
\begin{align*}
        \big\|f_{\bW(T)}^{\text{GD}} - \bS g_T^{\text{GD}}\big\|_\rho^2 \lesssim \big\|f_{\bW(T)}^{\text{GD}} - f^{\text{lin,GD}}_{\bW(T)}\big\|_\rho^2 + \big\|f^{\text{lin,GD}}_{\bW(T)} - \bS_mg^{m,\text{GD}}_T\big\|_\rho^2 + \big\|\bS_mg^{m,\text{GD}}_T - \bS g_T^{\text{GD}}\big\|_\rho^2.
    \end{align*}
    Combining Lemmas \ref{prop:f-flin_gd}, \ref{prop:flin-gkm_gd} with $\delta$ replaced by $\delta/3$ and Lemma \ref{prop:gmk-gk} with $\delta$ replaced by $\delta/3$ yields the desired results.
\end{proof}

\subsection{Minimax-Optimal Bounds of GD}\label{app:minimax-GD}
We use the following  error decomposition:  
\begin{align*}
& \varepsilon_{risk}\big(f^{\text{GD}}_{\bW(T)}\big) \lesssim \big\|  f^{\text{GD}}_{\bW(T)}-  \bS g^{\text{GD}}_T\big\|_\rho^2 +  \big\|\bS g_T^{\text{GD}}   - f_{\rho}\big\|_\rho^2,
\end{align*}
Theorem \ref{pro:f-g_T} gives the estimate of the first term. 
To estimate $\big\|\bS g^{\text{GD}}_T - f_{\rho} \big\|_\rho^2$,  
we first introduce an intermediate population iteration $h^{\text{GD}}_k$ on $\H_K$: 
\begin{align}\label{eq:population_iteration}
    h^{\text{GD}}_{k+1} = h^{\text{GD}}_k - \eta \int_\Z \big(\langle h^{\text{GD}}_k, K_\bx\rangle_{\H_K} - y\big)K_\bx d\rho(\bz) \text{ \ with \ $h^{\text{GD}}_0 = 0$.}
\end{align}
If we regard the population risk $\L(\cdot)$ as a functional on $\H_K$, then the population iteration $h^{\text{GD}}_k$ can be interpreted as the GD on $\L(\cdot)$ initialized at $h^{\text{GD}}_0=0$.

\begin{lemma}\label{lemma:fh=frho}
    Let $\H$ be the closure of $\H_K$ in $\mathcal{L}^2_{\rho_\bx}$.
    Then, Assumption \ref{ass:frho_smth} implies $f_\rho\in\H$.
\end{lemma}
\begin{proof}
    Note that $\bL$ has the eigen-decomposition $\bL  f=\sum_{i=1}^\infty \lambda_i \langle f, \Phi_i\rangle_{\mathcal{L}^2_{\rho_\bx}} \Phi_i$.
    According to Assumption \ref{ass:frho_smth}, we know there exists a $g\in\mathcal{L}^2_{\rho_\bx}$ such that
    \begin{align*}
        f_\rho = \bL^\beta g = \sum_{i=1}^\infty \lambda_i^\beta\langle g,\Phi_i\rangle_{\mathcal{L}^2_{\rho_\bx}}\Phi_i = \sum_{i:\lambda_i\neq0}^\infty \lambda_i^\beta\langle g,\Phi_i\rangle_{\mathcal{L}^2_{\rho_\bx}}\Phi_i.
    \end{align*}
    Since for any $\lambda_i\neq0$, the associated eigenfunction $\Phi_i\in\H_K$ (see Chapter 4.5 in \cite{steinwart2008support}), we conclude that $f_\rho\in\H$.
\end{proof}
\begin{lemma}\label{lemma:gk-hk}
    Suppose Assumptions \ref{ass:activation}, \ref{ass:effec_dim} and \ref{ass:frho_smth} hold.
    Assume $\eta \le \|K\|_\infty^{-1}$.
    For any $\delta_1,\delta_2\in(0,1/2)$, assume $\eta T \le n(9\|K\|_\infty\log(n/\delta_2))^{-1}$. 
    Then, the following statements hold with probability at least $1-\delta_1-\delta_2$ over sampling.
    \begin{enumerate}[label=(\alph*), leftmargin=*]
        \item For the case $\beta \ge \frac{1}{2}$, there holds
        \begin{align*}
            \|\bS g^{\text{GD}}_T - \bS h^{\text{GD}}_T\|_\rho \le 4(B\|K\|_\infty^\beta + 1)(12 + 4\log(T) + \sqrt{2})\Big(\frac{\sqrt{\|K\|_\infty\eta T}}{n} + \sqrt{\frac{2c_\gamma(\eta T)^\gamma}{n}}\Big)\log\big(\frac{4}{\delta_1}\big).
        \end{align*}
        \item For the case $\beta \in(0,\frac{1}{2})$, there holds
        \begin{align*}
            \|\bS g^{\text{GD}}_T - \bS h^{\text{GD}}_T\|_\rho &\le (12 + 4\log(T) + \sqrt{2}\eta)\bigg(2(6+B\sqrt{\|K\|_\infty})\Big(\frac{\sqrt{\|K\|_\infty\eta T}}{n} + \sqrt{\frac{2c_\gamma(\eta T)^\gamma}{n}}\Big)\\
            &\qquad + \frac{4B\|K\|_\infty\big((\eta T)^{1-\beta} + 1\big)}{n} \bigg)\log\Big(\frac{3T}{\delta_1}\Big).
        \end{align*}
    \end{enumerate}
\end{lemma}

\begin{proof}
    The proof can be derived from Theorem 5 in \cite{lin2017optimal}, which estimate $\|\S_\rho\nu_{k+1} - \S_\rho\mu_{k+1}\|_\rho$ with two iteration sequences $\{\nu_{k+1}\}$ and $\{\mu_{k+1}\}$. 
    We first show that their assumptions are satisfied in our setting, and then apply their results with our Lemma \ref{lemma:<2} by showing that $\S_\rho\nu_{k+1} - \S_\rho\mu_{k+1}$ is equivalent to $\bS g^{\text{GD}}_k - \bS h^{\text{GD}}_k$.

    Note that we assume $|y|\le1$.
    Then, their Assumption 1 holds with $M=v=1$.
    Instead of using the notations $\bx, \langle \bx, \bx'\rangle_H $ and $ \S_\rho$ in \cite{lin2017optimal} for any $\bx,\bx'\in\X$, we use $K_\bx,$ $\langle K_\bx, K_{\bx'}\rangle_{\H_K} $ and $ \bS$ in our setting.
    Then, their $H_\rho$ is the same as our $\H_K$.
    Since $f_\H$ in \cite{lin2017optimal} is the projection of $f_\rho$ onto the closure of $H_\rho$ in $\mathcal{L}^2_{\rho_\bx}$, from Lemma \ref{lemma:fh=frho} we know their $f_\H$ is equivalent to our $f_\rho$.
    Hence, Assumption 2 in \cite{lin2017optimal} holds true with $\zeta = \beta$ and $R=B$ due to our Assumption \ref{ass:frho_smth}.
    Further, their Assumption 3 is guaranteed by Assumption \ref{ass:effec_dim}, their equation (3) holds true with $\kappa^2 = \|K\|_\infty$ due to $\langle K_\bx, K_{\bx'}\rangle_{\H_K} = K(\bx,\bx') \le \|K\|_\infty$. 
    Their equation (47) is guaranteed by Lemma \ref{lemma:<2} with $\kappa^2=\|K\|_\infty$, $\Gamma = n$, $\delta = \delta_2$, $\zeta_i = K_{\bx_i}$, $Q = \int_\X K_\bx\otimes K_\bx d\rho_{\bx}$.
    In addition,
    by taking the step-size $\eta_k = \eta$ for all $k\in[T]$, we know $\S_\rho\nu_{k+1} - \S_\rho\mu_{k+1}$ in \cite{lin2017optimal} is equivalent to our $\bS g^{\text{GD}}_k - \bS h^{\text{GD}}_k$.

    Then, combining above observations and Theorem 5 in \cite{lin2017optimal} with $\eta_k = \eta$, $\theta = 0$, $\lambda = (\eta T)^{-1}$, $\kappa^2 = \|K\|_\infty$, $M=v=1$, $R = B$, $\zeta = \beta$ and $m=n$, we get the desired results.
\end{proof}

\begin{lemma}\label{lemma:hk-frho}[Proposition 2 in \cite{lin2017optimal}]
    Suppose Assumption \ref{ass:frho_smth} holds.
    Let $\eta \in(0,\|K\|_\infty^{-1}]$ be the step size.
    For any $k\in\mathbb{N}$, there holds
    \begin{align*}
        \|\bS h^{\text{GD}}_k - f_\rho\|_\rho \le B\Big(\frac{\beta}{2\eta k}\Big)^\beta.
    \end{align*}
\end{lemma}
\begin{proof}
    In the proof of Lemma \ref{lemma:gk-hk}, we already showed that $\mathcal{S}_\rho\mu_{k+1}$ and $f_\H$ in \cite{lin2017optimal} are equivalent to our $\bS h_k$ and $f_\rho$.
    Then, by applying Proposition 2 in \cite{lin2017optimal} with $\eta_k = \eta$, $\kappa^2=\|K\|_\infty$, $R=B$ and $\zeta = \beta$, we get the desired results.
\end{proof}

Combining Lemma \ref{lemma:gk-hk} and Lemma \ref{lemma:hk-frho}, we give the proof of Lemma~\ref{prop:gT-frho}.  
\begin{lemma}\label{prop:gT-frho}
    Let $\delta\in(0,1)$.
    Suppose Assumptions \ref{ass:activation}, \ref{ass:effec_dim} and \ref{ass:frho_smth} hold.
    Assume $\eta \le \|K\|_\infty^{-1}$ and $\eta T \le n(9\|K\|_\infty\log(\frac{2n}{\delta}))^{-1}$. 
    Then, with probability at least $1-\delta$ over sampling, there holds   
        \begin{align*}
            \|\bS g^{\text{GD}}_T - f_\rho\|_\rho^2 \lesssim \|K\|_\infty^{\beta \vee 1}\big(   \eta Tn^{-2}  +  (\eta T)^\gamma n^{-1} +  (\eta T)^{1 - 2\beta}  n^{-1}\big) \log^4\big( {T}/{\delta}\big) + (\eta T)^{-2\beta} .
        \end{align*}
\end{lemma}

\begin{proof}
    Note that $\|\bS g^{\text{GD}}_T - f_\rho\|_\rho^2 \lesssim \|\bS g^{\text{GD}}_T - \bS h^{\text{GD}}_T\|_\rho^2 + \|\bS h^{\text{GD}}_T - f_\rho\|_\rho^2$.
    The desired results are obtained by combining Lemma \ref{lemma:gk-hk} with $\delta_1 = \delta_2 = \frac{\delta}{2}$ and Lemma \ref{lemma:hk-frho}.
\end{proof}

Now, we give the following excess risk bounds. 
\begin{theorem}\label{thm:f-frho_gd}
Let $\delta\in(0,1)$ and $T\in\mathbb{N}$.
Suppose Assumptions~\ref{ass:activation}, \ref{ass:effec_dim} and \ref{ass:frho_smth} hold.
Assume that $\eta \le C_{\sigma}^{-L}$, $1 \le \eta T \le n\big(C_\sigma^L\log(6n/\delta)\big)^{-1}$ and $m \gtrsim C_\sigma^L d(\eta T)^4\log^3({m}/{\delta})$.
    With probability at least $1-CL^2\exp(-C_\sigma \log^2(m)) - \delta$ over initialization $(\ba,\bW(0))$ and sampling, it holds that
    \begin{align*}
           \varepsilon_{risk}\big(f^{\text{GD}}_{\bW(T)}\big)  \lesssim  \frac{C_\sigma^L d(\eta T)^4\log^2(m)}{m} + C_\sigma^{L(\beta \vee 1)} \times 
       \big(   \frac{\eta T}{n^{2}}  + \frac{   (\eta T)^{\gamma \vee (1 - 2\beta)} }{n}  \big) \log^4 ( \frac{T}{\delta})  +  \frac{1}{(\eta T)^{ 2\beta}}.
    \end{align*} 
\end{theorem} 
\begin{proof} 
Note
\begin{align*}
    \big\|f^{\text{GD}}_{\bW(T)} - f_\rho\big\|_\rho^2 \lesssim \big\|f^{\text{GD}}_{\bW(T)}  - \bS g^{\text{GD}}_T\big\|_\rho^2 + \big\|\bS g^{\text{GD}}_T- f_\rho\big\|_\rho^2.
\end{align*}
    Combining \eqref{eq:bound_K}, Theorem \ref{pro:f-g_T} with $\delta$ replaced by $\delta/2$ and Lemma \ref{prop:gT-frho} with $\delta$ replaced by $\delta/2$ yields the desired results.
\end{proof}
Based on the above theorem, we present the proof of Theorem \ref{cor:f-frho_gd}. 
\begin{proof}[Proof of Theorem \ref{cor:f-frho_gd}]
The proof follows from Theorem \ref{thm:f-frho_gd} with $\delta$ replaced by $\delta/2$.
    We first show that the condition $n \ge \frac{12}{\delta}\big(\frac{18(2\beta+ \gamma)}{2\beta+\gamma-1}\big)^{\frac{4\beta + 2\gamma}{2\beta + \gamma - 1}}$ implies $\eta T \le n\big(C_\sigma^L\log(12n/\delta)\big)^{-1}$.
    Since $\eta T \le 2\eta n^{\frac{1}{2\beta + \gamma}}$ and $\eta \asymp C_\sigma^{-L}$, it suffices to show that $n^{\frac{2\beta+\gamma-1}{2\beta+\gamma}} \ge 18\log(12n/\delta)$, which is equivalent to proving
    $\big(\frac{12n}{\delta}\big)^{\frac{2\beta + \gamma -1}{2\beta+\gamma}} \ge \frac{18(2\beta+\gamma)}{2\beta + \gamma - 1}\big(\frac{12}{\delta}\big)^{\frac{2\beta + \gamma - 1}{2\beta+\gamma}}\log\big(\frac{12n}{\delta}\big)^{\frac{2\beta + \gamma - 1}{2\beta+\gamma}}.$
From (9.17) and (9.18) in \cite{gyorfi2006distribution} we know $u > 2c\log(c)$ implies $u> c\log(u)$ for any $c\ge e$.
Setting $u = \big(\frac{12n}{\delta}\big)^{\frac{2\beta + \gamma -1}{2\beta+\gamma}}$ and $c = \frac{18(2\beta+\gamma)}{2\beta + \gamma - 1}\big(\frac{12}{\delta}\big)^{\frac{2\beta + \gamma - 1}{2\beta+\gamma}}$ and solving $u\ge c^2$, the desired result is obtained by noting $u \ge c^2 > 2c\log(c)$ for all $c\ge e$.
Combining this with $T=\lceil n^{\frac{1}{2\beta+\gamma}}\rceil$, we know $n \ge \max\big\{ \frac{12}{\delta}\big(\frac{18(2\beta+\gamma)}{2\beta+\gamma-1}\big)^{\frac{4\beta + 2\gamma}{2\beta + \gamma - 1}}, \eta^{-(2\beta + \gamma)}\big\}$ implies $1\le \eta T\le n\big(C_\sigma^L\log(12n/\delta)\big)^{-1}$.

Since $\eta T \asymp n^{\frac{1}{2\beta+\gamma}}C_\sigma^{-L}$, we can show that $m\gtrsim C_\sigma^Ldn^{\frac{2\beta+4}{2\beta+\gamma}}\log^3\big(\frac{dLn}{\delta}\big)$ implies the condition $m \gtrsim C_\sigma^L d(\eta T)^4\log^3(\frac{m}{\delta})$ and ensures $\frac{C_\sigma^Ld(\eta T)^4\log^3(\frac{m}{\delta})}{m} \lesssim C_\sigma^L n^{-\frac{2\beta}{2\beta+\gamma}}$ in a similar way.

By further noting that $m\gtrsim C_\sigma^Ldn^{\frac{2\beta+4}{2\beta+\gamma}}\log^3\big(\frac{dLn}{\delta}\big)$ and $n \gtrsim \delta^{-1}$ imply $L^2\exp(-C_\sigma \log^2(m)) \le \delta/2$. 
Then, from Theorem~\ref{thm:f-frho_gd} and \eqref{eq:bound_K} we know with probability at least $1-\delta$ over initialization $(\ba,\bW(0))$ and sampling, there holds
\begin{align*}
    \varepsilon_{risk}\big(f^{\text{GD}}_{\bW(T)}\big)  \lesssim    C_\sigma^L n^{-\frac{2\beta}{2\beta+\gamma}}+ C_\sigma^{L(\beta\vee1)}\Big(\frac{ {\eta T}}{n^2} + \frac{(\eta T)^\gamma + (\eta T)^{1 - 2\beta}}{n}  \Big)\log^4 \big(\frac{T}{\delta}\big)\nonumber  + (\eta T)^{-2\beta}.
\end{align*}
In addition, since $2\beta+\gamma > 1$ and $\eta T \ge 1$, it holds that $(\eta T)^{1-2\beta} \le (\eta T)^\gamma$.
Plugging the choice of $\eta T \asymp n^{\frac{1}{2\beta+\gamma}}C_\sigma^{-L}$ back into the above inequality, we get 
\begin{align*}
   \varepsilon_{risk}\big(f^{\text{GD}}_{\bW(T)}\big)  \lesssim  C_\sigma^{L(\beta\vee1)}n^{-\frac{2\beta}{2\beta+\gamma}} \log^4\big(\frac{n}{\delta}\big).
\end{align*}
The proof is complete. 
\end{proof}

\section{Proofs of Stochastic Gradient Descent}
In this section, we prove the main results of SGD.
\subsection{A Kernel-to-Network Transfer Framework of SGD}\label{app:transfer-SGD}
Recall that \(\bW(k)\) denote the \(k\)-th SGD iterate generated by \eqref{eq:sgd-update}, and let \(f_{\bW(k)}^{\text{SGD}}\) and $\nabla f_{\bW(k)}^{\text{SGD}}$ be the corresponding network predictor and the gradient at iteration $k$.
Note that the superscript $\mathrm{SGD}$ is used only to indicate the optimization trajectory, while the underlying network and gradient remain $f_{\bW(k)}$ and $\nabla f_{\bW(k)}$, respectively.
At initialization, we simply write $f_{\bW(0)}$ and $\nabla f_{\bW(0)}$ instead of $f^{\text{SGD}}_{\bW(0)}$ and $\nabla f^{\text{SGD}}_{\bW(0)}$.

Recall that $g^{m,\text{GD}}_k$ and $g_k^{\text{GD}}$ denote the iteration of GD in $\H_m$ and $\H_K$, respectively.
We also introduce the iteration of SGD for the empirical risk $\L_S(\cdot)$ in $\H_m$ by
\begin{align}\label{eq:update-kernelSGD}
f^m_{k+1} = f^m_k - \eta( f^m_k(\bx_{i_k}) - y_{i_k})K^m_{\bx_{i_k}} \text{ with } \ f^m_0 = 0, 
\end{align}
where $i_k\sim \mathrm{Unif}\{1,\ldots,n\}.$
With these intermediate objects, we can further decompose the error term $\ebb [\|f_{\bW(T)}-  \bS g_T\|_\rho^2]$ into \begin{align*}
        \ebb \big[\big\|f^{\text{SGD}}_{\bW(T)} - \bS g_T^{ \text{GD}}\big\|_\rho^2\big] \lesssim &\,\, \ebb \big[\big\|f^{ \text{SGD}}_{\bW(T)} - f^{\text{lin,SGD}}_{\bW(T)}\big\|_\rho^2 + \big\|f^{\text{lin,SGD}}_{\bW(T)} - \bS_m f^m_T\big\|_\rho^2 + \big\|\bS_m\big(f^m_T - g^{m,\text{GD}}_T\big)\big\|_\rho^2\big]\\& + \big\|\bS_mg^{m,\text{GD}}_T - \bS g^{\text{GD}}_T\big\|_\rho^2.
    \end{align*}
    Here,  the expectation is taken over the algorithmic randomness of SGD. 
We will bound these four terms in the sequel.

\subsubsection{(1) Control of $\big\|f^{ \text{SGD}}_{\bW(T)} - f^{\text{lin,SGD}}_{\bW(T)}\big\|_\rho^2$}\label{app:SGD-term1}

\begin{lemma}\label{lem:o-o0}
    Suppose $m \gtrsim \max\{R^2, d\}$ and \eqref{eq:optr_nrm_W} holds.
    For any $\ell\in[L]$, it holds that
    \begin{align*}
        \sup_{\bx\in\X}\sup_{\|\bW-\bW(0)\|_2 \le R}\big\|o^{\ell}(\bx) - o^{\ell}_0(\bx)\big\|_2 \le\frac{C_\sigma^LR}{\sqrt{m}}.
    \end{align*}
\end{lemma}

\begin{proof}
    Applying Lemma~\ref{lem:tildeo-o} with $\tildeW = \bW(0)$, we know that for all $\ell\in[L]$, it holds that
    \begin{align*}
        \|o^\ell(\bx) - o^\ell_0(\bx)\|_2 \le   \frac{C_\sigma^\ell\sum_{k=1}^\ell\|\bW^k - \bW^k(0)\|_{op}}{\sqrt{m}} \le \frac{C_\sigma^L\ell R}{\sqrt{m}}\le \frac{C_\sigma^LR}{\sqrt{m}}.
    \end{align*}
    This completes the proof of the lemma.
\end{proof}
Recall that we only use the superscript $\mathrm{SGD}$ to indicate the optimization trajectory.

\begin{lemma}\label{lem:wk-wo_sgd}
Let $\delta\in(0,1).$
    Suppose Assumption~\ref{ass:activation} holds.
 Let $\{\bW(k)\}_{k=1}^T$ be produced by SGD iterations with   $T$ iterations based on $S$. 
 Assume $\eta \le C_\sigma^{-L}$, $\eta T\ge 1$ and $m\gtrsim C_\sigma^Ld(\eta T)^4\log(m/\delta)$. 
     Then, with probability at least $1 - CL\exp(-m/2) - \delta$ over initialization $(\ba, \bW(0))$, for any $k=0,\ldots,T$ it holds that \[\|\bW(k)-\bW(0) \|_2^2  \le 4\eta k \ \text{ and } \ \sup_{z\in\Z}|f^{\text{SGD}}_{\bW(k)}(\bx) - y| \le C_\sigma^L\sqrt{\eta k} + 1. \] 
\end{lemma}

\begin{proof}
    We prove this lemma by induction on $k$.
    It is evident that $\|\bW(k) - \bW(0)\|_2^2 \le 4\eta k$ and $\sup_{z\in\Z}|f^{\text{SGD}}_{\bW(k)}(\bx) - y| \le C_\sigma^L\sqrt{\eta k} + 1$ hold for $k=0$.
    We assume that $\|\bW(t) - \bW(0)\|_2^2 \le 4\eta t$ and $\sup_{z\in\Z}|f^{\text{SGD}}_{\bW(t)}(\bx) - y| \le C_\sigma^L\sqrt{\eta t} + 1$ hold for all $t \le k$ with some $k\in[T-1]$.
    We will prove that $\|\bW(k+1) - \bW(0)\|_2^2 \le 4\eta (k+1)$ and $\sup_{z\in\Z}|f^{\text{SGD}}_{\bW(k+1)}(\bx) - y| \le C_\sigma^L\sqrt{\eta (k+1)} + 1$.
    
    From the update rule \eqref{eq:sgd-update}, we have
    \begin{align*}
        &\big\|\bW(k+1) - \bW(0)  \big\|_2^2 = \big\|\bW(k) - \bW(0) - \bW(k) + \bW(k+1)  \big\|_2^2\nonumber\\
       &= \big\|\bW(k) - \bW(0)\big\|_2^2 + \big\|\bW(k+1) - \bW(k)\big\|_2^2 + 2\langle \bW(k) - \bW(0), \bW(k+1) - \bW(k) \big\rangle \nonumber\\
       &= \big\|\bW(k) - \bW(0)\big\|_2^2 + \eta^2|f^{\text{SGD}}_{\bW(k)}(\bx_{i_k}) - y_{i_k}|^2\big\|\nabla f^{\text{SGD}}_{\bW(k)}(\bx_{i_k})\big\|_2^2 + 2\eta \big\langle \bW(0) - \bW(k),  \nabla l(\bW(k);z_{i_k})  \big\rangle\\
       &= \big\|\bW(k) - \bW(0)\big\|_2^2 + 2\eta^2l(\bW(k);z_{i_k})\big\|\nabla f^{\text{SGD}}_{\bW(k)}(\bx_{i_k})\big\|_2^2 + 2\eta \big\langle \bW(0) - \bW(k),  \nabla l(\bW(k);z_{i_k})  \big\rangle.
    \end{align*}
    According to the induction assumption and \eqref{eq:almost_convex} with $R = 2\sqrt{\eta T}$, $\bW = \bW(k)$, $\widetilde{\bW} = \bW(0)$ and $\|\bW(k) - \bW(0)\|_2 \le 2\sqrt{\eta T}$, we know
\begin{align*}
    2\eta\big\langle \bW(0)- \bW(k),   \nabla l(\bW(k);z_{i_k}) \big\rangle 
    &\le 2\eta\big(l(\bW(0);z_{i_k}) - l(\bW(k);z_{i_k})\big) + 2\eta \epsilon|f^{\text{SGD}}_{\bW(k)}(\bx_{i_k}) - y_{i_k}|.
\end{align*}
with $\epsilon \le C_\sigma^L(R+1)\|\bW(k) - \bW(0)\|^2_2\sqrt{\log(m/\delta)/m} \le C_\sigma^L(\eta T)^{3/2}\sqrt{\log(m/\delta)/m}$.

Combining the above equality and the inequality, we have
\begin{align*}
    &\big\|\bW(k+1) - \bW(0)\big\|_2^2\\
    &\le \big\|\bW(k) - \bW(0)\big\|_2^2 +  2\eta^2l(\bW(k);z_{i_k})\big\|\nabla f^{\text{SGD}}_{\bW(k)}(\bx_{i_k})\big\|_2^2 + 2\eta\big(l(\bW(0);z_{i_k}) - l(\bW(k);z_{i_k})\big) \\& \quad + 2\eta \epsilon|f^{\text{SGD}}_{\bW(k)}(\bx_{i_k}) - y_{i_k}|\\
    &\le \big\|\bW(k) - \bW(0)\big\|_2^2 + C_\sigma^L\eta^2l(\bW(k);z_{i_k}) + 2\eta\big(l(\bW(0);z_{i_k}) - l(\bW(k);z_{i_k})\big) + 2\eta \epsilon|f^{\text{SGD}}_{\bW(k)}(\bx_{i_k}) - y_{i_k}|\\
    &\le \big\|\bW(k) - \bW(0)\big\|_2^2 + C_\sigma^L\eta^2l(\bW(k);z_{i_k}) + 2\eta\big(l(\bW(0);z_{i_k}) - l(\bW(k);z_{i_k})\big) + (C_\sigma^L\sqrt{\eta k}+1)\epsilon\eta\\
    &\le \big\|\bW(k) - \bW(0)\big\|_2^2 + (C_\sigma^L\eta^2 - 2\eta)l(\bW(k);z_{i_k}) + 2\eta l(\bW(0);z_{i_k}) + \eta\\
    &\le \big\|\bW(k) - \bW(0)\big\|_2^2 + \eta + \eta \le 4\eta k + 2\eta < 4\eta(k+1),
\end{align*}
where in the second inequality we have used Lemma~\ref{lem:Hessian} by noting $\|\bW(k) - \bW(0)\|_2 \le 2\sqrt{\eta k} \le R$ and the condition $m\gtrsim C_\sigma^L(\eta T)^3\log(m/\delta)$, and in the third inequality we have used the induction assumption $\sup_{z\in\Z}|f^{\text{SGD}}_{\bW(k)}(\bx) - y| \le C_\sigma^L \sqrt{\eta k}+1$, and in the fourth inequality we have used $(C_\sigma^L\sqrt{\eta k}+1) \epsilon \le C_\sigma^L(\eta T)^{2}\sqrt{\log(m/\delta)/m} \le 1$ by noting the condition $m \gtrsim C_\sigma^Ld(\eta T)^4\log(m/\delta)$, and in the third-to-last inequality we have used $\eta \le C_\sigma^{-L}$ and $2l(\bW(0);z_{i_k}) = |y_{i_k}|^2 \le 1$ and in the second-to-last inequality we have used induction assumption $\|\bW(k)-\bW(0)\|_2^2\le 4\eta k$.

Now, we turn to show that $|f^{\text{SGD}}_{\bW(k+1)}(\bx) - y| \le C_\sigma^L\sqrt{\eta (k+1)} + 1$ for any $z=(\bx,y)\in\Z$.
Combining Lemma~\ref{lemma:oprt_norm}, Lemma \ref{lem:o-o0} with $\bW = \bW(k+1)$ and the fact $\|\bW(k+1)- \bW(0)\|_2\le 2\sqrt{\eta (k+1)} \le R$, we know
\begin{align*}
    |f^{\text{SGD}}_{\bW(k+1)}(\bx) - y| &\le |f^{\text{SGD}}_{\bW(k+1)}(\bx) - f_{\bW(0)}(\bx)| + |f_{\bW(0)}(\bx) - y| \le \|\ba\|_2\|o^L_k(\bx) - o^L_0(\bx)\|_2 + 1\\
    &\le C_\sigma^L \sqrt{\eta (k+1)} + 1,
\end{align*}
which completes the second part of the lemma.
\end{proof}

Recall that $f^{\text{lin,SGD}}_{\bW(k)}(\bx) = f_{\bW(0)}(\bx) +  \langle \bW(k) - \bW(0), \nabla f_{\bW(0)}(\bx)\rangle$ denote  the first-order linear approximation around initialization $\bW(0)$ for SGD.
\begin{lemma}\label{prop:f-flin_sgd}
    Let $\delta\in(0,1).$
    Suppose Assumption~\ref{ass:activation} holds.
    Let $\{\bW(k)\}$ be produced by SGD iterations with $T$ iterations based on $S$. 
    Assume $\eta \le C_\sigma^{-L}$, $\eta T\ge1$ and $m\gtrsim C_\sigma^Ld(\eta T)^4\log(m/\delta)$. 
    Then, with probability at least $1 - CL\exp(-m/2) - \delta$ over initialization $(\ba, \bW(0))$, it holds that
     \begin{align*}
      \big\|f_{\bW(T)}^{\text{SGD}} - f^{\text{lin,SGD}}_{\bW(T)}\big\|_\rho^2 \le   \big\|f_{\bW(T)}^{\text{SGD}} - f^{\text{lin,SGD}}_{\bW(T)}\big\|_\infty^2 \le C_\sigma^L(\eta T)^3\frac{\log(m/\delta)}{m}. 
    \end{align*}
\end{lemma}
\begin{proof}
    From Lemma \ref{lem:wk-wo_sgd} we know $\|\bW(k) - \bW(0)\|_2 \le 2\sqrt{\eta k}$ for all $k\in[T]$.
    Setting $R = 2\sqrt{\eta T}$.
    The condition $m \gtrsim C_\sigma^L d(\eta T)^4\log(m/\delta)$ implies $m\gtrsim R^2$.
    Then, from the definition of $f^{\text{lin,SGD}}_\bW$ and \eqref{eq:almost_linear} in Lemma \ref{lem:almost-convex} with $R=2\sqrt{\eta T}$, $\widetilde{\bW} = \bW(k)$ and $\bW = \bW(0)$, it holds that
    \begin{align*}
        \big\|f_{\bW(k)}^{\text{SGD}} - f^{\text{lin,SGD}}_{\bW(k)}\big\|_\infty \le C_\sigma^L \sqrt{\frac{(\eta k)^3\log(m/\delta)}{m}}.
    \end{align*}
    Taking $k=T$ completes the proof of the lemma.
\end{proof}

\subsubsection{(2) Control of $\big\|f^{\text{lin,SGD}}_{\bW(T)} - \bS_m f^m_T\big\|_\rho^2$}\label{app:SGD-term2}

Recall that $f^m_k$  defined in \eqref{eq:update-kernelSGD} denotes the iteration of SGD in $\H_m$.

\begin{lemma}\label{prop:flin-kernel_sgd}
    Let $\delta\in(0,1).$
    Suppose Assumption \ref{ass:activation} holds.
    Assume $\eta \le C_\sigma^{-L}$, $\eta T\ge1$ and $m \gtrsim C_\sigma^Ld(\eta T)^4\log(m/\delta)$.
    With probability at least $1 - CL\exp(-m/2) - \delta$ over initialization $(\ba, \bW(0))$, it holds that
    \begin{align*}
    \big\|f^{\text{lin,SGD}}_{\bW(T) } - f_T^m\big\|_\rho^2 \le     \big\|f^{\text{lin,SGD}}_{\bW(T) } - f_T^m\big\|_\infty^2 \le C_\sigma^L  (\eta T)^{5}\frac{\log(m/\delta)}{m} 
    \end{align*}
\end{lemma}

\begin{proof}
    For all $k\in[T-1]$, we define $\epsilon_k = f_{\bW(k+1)}^{\text{lin,SGD}} - f_{\bW(k)}^{\text{lin,SGD}} + \eta(f_{\bW(k) }^{\text{lin,SGD}}(\bx_{i_k}) - y_{i_k})K^m_{\bx_{i_k}} \in\H_m$.
From the update rule of $ f_{k}^{m}$ (see \eqref{eq:update-kernelSGD}), it holds that
\begin{align}
    f^{\text{lin,SGD}}_{\bW(k+1)} - f^m_{k+1} &= \big(f^{\text{lin,SGD}}_{\bW(k)} -f^m_k\big) -\eta\big(f^{\text{lin,SGD}}_{\bW(k)}(\bx_{i_k})- f^m_k(\bx_{i_k})\big)K^m_{\bx_{i_k}} + \epsilon_k \nonumber\\
    &= \big(f^{\text{lin,SGD}}_{\bW(k)} -f^m_k\big) - \eta\big\langle f^{\text{lin,SGD}}_{\bW(k)} -f^m_k, K^m_{\bx_{i_k}}\big\rangle_{\H_m}K^m_{\bx_{i_k}} + \epsilon_k\nonumber\\
    &= \big(\bfI - \eta K^m_{\bx_{i_k}}\otimes K^m_{\bx_{i_k}}\big)\big(f^{\text{lin,SGD}}_{\bW(k)} -f^m_k\big) + \epsilon_k\nonumber, 
\end{align}
where in the second equality we have used $f^{\text{lin,SGD}}_{\bW(k) } - f_k^m \in \H_m$ and the reproducing kernel property $f_{\bW(k)}^{\text{lin,SGD}}(\bx_{i_k}) - f^m_k(\bx_{i_k}) = \langle f_{\bW(k)}^{\text{lin,SGD}} - f^m_k,K^m_{\bx_{i_k}}\rangle_{\H_m}.$

Applying the above equality recursively, we get
\begin{align*}
   f^{\text{lin,SGD}}_{\bW(k+1)} - f^m_{k+1}
    &= \sum_{s=0}^{k}\prod_{a=s+1}^k \big(\bfI - \eta K^m_{\bx_{i_a}}\otimes K^m_{\bx_{i_a}}\big)\epsilon_{s},
\end{align*}
where we used the conventional notation $\prod_{k+1}^k = \bfI$ for any $k\in\mathbb{N}$.
Note that for any $a\in[k]$ and $i_a$, the operator $\eta K^m_{\bx_{i_a}}\otimes K^m_{\bx_{i_a}}$ is self-adjoint and positive, and from Lemma~\ref{lem:Hessian} we know $\|K^m_{\bx_{i_a}}\otimes K^m_{\bx_{i_a}}\|_{op} = \|K^m_{\bx_{i_a}}\|_{\H_m}^2 = K^m(\bx_{i_a},\bx_{i_a}) \le \|K^m\|_\infty = \sup_{\bx\in\X}\|\nabla f_{\bW(0)}(\bx)\|_2^2 \le C_\sigma^L$.
Then, $\|\bfI - \eta K^m_{\bx_{i_a}}\otimes K^m_{\bx_{i_a}}\|_{op} \le 1$ due to $\eta \le C_\sigma^{-L}$.

According to the above inequality, we have
\begin{align}\label{eq:sgd-flin-fm-sum}
    &\big\|f^{\text{lin,SGD}}_{\bW(k+1)} - f^m_{k+1}\big\|_\infty = \sup_{\bx\in\X}\big|f^{\text{lin,SGD}}_{\bW(k+1)}(\bx) - f^m_{k+1}(\bx)\big| = \sup_{\bx\in\X} \big|\big\langle f^{\text{lin,SGD}}_{\bW(k+1)} - f^m_{k+1}, K^m_\bx\big\rangle_{\H_m}\big|\nonumber\\
    &\le \sup_{\bx\in\X} \big\|f^{\text{lin,SGD}}_{\bW(k+1)} - f^m_{k+1}\big\|_{\H_m} \|K^m_\bx\|_{\H_m}   \le  \big\|f^{\text{lin,SGD}}_{\bW(k+1)} - f^m_{k+1}\big\|_{\H_m} \sqrt{\|K^m\|_\infty}\nonumber\\
    &\le C_\sigma^L\sum_{s=0}^{k}\prod_{a=s+1}^k \big\|\bfI - \eta K^m_{\bx_{i_a}}\otimes K^m_{\bx_{i_a}}\big\|_{op} \big\|\epsilon_{s}\big\|_{\H_m} \le C_\sigma^L\sum_{s=0}^{k} \big\|\epsilon_{s}\big\|_{\H_m},
\end{align}
where in the second equality we have used the reproducing kernel property and the last second inequality follows from $\sqrt{\|K^m\|_\infty}\le \sup_{\bx\in\X}\|\nabla f_{\bW(0)}(\bx)\|_2 \le \sqrt{C_\sigma^L} \le C_\sigma^L$.

Now, we are in a position to estimate $\|\epsilon_k\|_{\H_m}$.
For any $k\in[T-1]$, from the definition of $f^{\text{lin,SGD}}_{\bW }$ and the update rule of SGD \eqref{eq:sgd-update}, there holds
    \begin{align}
        \epsilon_k(\bx) &= f^{\text{lin,SGD}}_{\bW(k+1) }(\bx)  - f^{\text{lin,SGD}}_{\bW(k)} (\bx) + \eta\big(f^{\text{lin,SGD}}_{\bW(k)} (\bx_{i_k}) - y_{i_k}\big)K^m_{\bx_{i_k}}(\bx)\nonumber\\
        &= \big\langle \nabla f_{\bW(0)}(\bx), \bW(k+1)- \bW(k) \big\rangle + \eta\big(f^{\text{lin,SGD}}_{\bW(k)} (\bx_{i_k}) - y_{i_k}\big)\big\langle \nabla f_{\bW(0)}(\bx_{i_k}), \nabla f_{\bW(0)}(\bx) \big\rangle \nonumber\\
        & = \eta\big[\big(y_{i_k} - f^{\text{SGD}}_{\bW(k)}(\bx_{i_k})\big)\big\langle \nabla f^{\text{SGD}}_{\bW(k)}(\bx_{i_k}), \nabla f_{\bW(0)}(\bx) \big\rangle  + \big(f^{\text{lin,SGD}}_{\bW(k)}(\bx_{i_k}) - y_{i_k}\big) \big\langle \nabla f_{\bW(0)}(\bx_{i_k}), \nabla f_{\bW(0)}(\bx) \big\rangle\big] \nonumber\\
        &=\Big\langle \eta  \big[\big(y_{i_k} - f^{\text{SGD}}_{\bW(k)}(\bx_{i_k})\big)\nabla f^{\text{SGD}}_{\bW(k)}(\bx_{i_k}) + \big(f^{\text{lin,SGD}}_{\bW(k)}(\bx_{i_k}) - y_{i_k}\big)\nabla f_{\bW(0)}(\bx_{i_k})\big] , \nabla f_{\bW(0)}(\bx)\Big\rangle \nonumber\\
        & =: \big\langle \Delta(k), \nabla f_{\bW(0)}(\bx) \big\rangle, \nonumber
    \end{align}
    where the second equality follows from $K^m_{\bx_{i_k}}(\bx) = K^m(\bx_{i_k},\bx) = \langle \nabla f_{\bW(0)}(\bx_{i_k}), \nabla f_{\bW(0)}(\bx) \rangle$, and the third equality is according to the update rule \eqref{eq:sgd-update}, and in the last equality we defined
    \begin{align*}
        \Delta(k) &:= \eta  \big[\big(y_{i_k} - f^{\text{SGD}}_{\bW(k)}(\bx_{i_k})\big)\nabla f^{\text{SGD}}_{\bW(k)}(\bx_{i_k}) + \big(f^{\text{lin,SGD}}_{\bW(k)}(\bx_{i_k}) - y_{i_k}\big)\nabla f_{\bW(0)}(\bx_{i_k})\big]\\
        &= \eta \big[\big(y_{i_k} - f^{\text{SGD}}_{\bW(k)}(\bx_{i_k})\big)\big(\nabla f^{\text{SGD}}_{\bW(k)}(\bx_{i_k}) - \nabla f_{\bW(0)}(\bx_{i_k})\big) + \big(f^{\text{lin,SGD}}_{\bW(k)}(\bx_{i_k}) - f^{\text{SGD}}_{\bW(k)}(\bx_{i_k})\big)\nabla f_{\bW(0)}(\bx_{i_k})\big].
    \end{align*}
    
    According to the definition  $\|f\|_{\H_m}  =  \inf  \{ \|\bW\|_2 :   \bW  \in \W  \text{ with } f(\bx) =  \langle \bW, \Phi_m(\bx) \rangle_2  \}$, we have
    \begin{align}\label{eq:epsilon2}
        &\|\epsilon_k\|_{H_m} \le \|\Delta(k)\|_2\nonumber\\
        &\le \eta \big[\big|y_{i_k} -  f^{\text{SGD}}_{\bW(k)}(\bx_{i_k})\big|\big\|\nabla f^{\text{SGD}}_{\bW(k)}(\bx_{i_k}) - \nabla f_{\bW(0)}(\bx_{i_k})\big\|_2  + \big|f^{\text{lin,SGD}}_{\bW(k)}(\bx_{i_k})  -  f^{\text{SGD}}_{\bW(k)}(\bx_{i_k})\big|  \big\|\nabla f_{\bW(0)}(\bx_{i_k})\big\|_2\big].
    \end{align}
    Setting $R = 2\sqrt{\eta T}$, from Lemma \ref{lem:wk-wo_sgd} we know $\sup_{z\in\Z}|f^{\text{SGD}}_{\bW(k)}(\bx) - y| \le C_\sigma^L\sqrt{\eta k} + 1$ and $\|\bW(k) -\bW(0)\|_2\le R = 2\sqrt{\eta T}$ for any $k\in[T]$.
    Combining this, Lemma \ref{lem:Hessian} and fundamental theorem of calculus, it holds that
    \begin{align*}
        &\sup_{z\in\Z}|f^{\text{SGD}}_{\bW(k)}(\bx)-y|\big\|\nabla f^{\text{SGD}}_{\bW(k)}(\bx)  -  \nabla f_{\bW(0)}(\bx)\big\|_2\\
        &\le C_\sigma^L\sqrt{\eta T}\sup_{\bx\in\X}\Big\|\int_0^1 \nabla^2 f_{\bW(0) + t(\bW(k) - \bW(0))}(\bx) \big(\bW(k) - \bW(0)\big) dt\Big\|_2\\
        &\le C_\sigma^L\sqrt{\eta T}\sup_{\|\bW-\bW(0)\|_2\le R}\sup_{\bx\in\X}\big\|\nabla^2 f_{\bW}(\bx)\big\|_{op}\big\|\bW(k)-\bW(0)\big\|_2\\
        &\le C_\sigma^L \eta T\big\|\bW(k)-\bW(0)\big\|_2\sqrt{\frac{\log(m/\delta)}{m}} \le C_\sigma^L (\eta T)^{\frac{3}{2}}\sqrt{\frac{\log(m/\delta)}{m}}.
    \end{align*}
    According to Lemma \ref{prop:f-flin_sgd}, we know $\|f^{\text{SGD}}_{\bW(k)} - f_{\bW(k)}^{\text{lin,SGD}}\|_\infty \le C_\sigma^L(\eta k)^{\frac{3}{2}}\sqrt{\log(m/\delta)/m}$.
    From Lemma~\ref{lem:Hessian} we know $\sup_{\bx\in\X} \|\nabla f_{\bW(0)}(\bx)\|_2 \le C_\sigma^L$.
    Then, $|f^{\text{lin,SGD}}_{\bW(k)}(\bx_{i_k})  -  f_{\bW(k)}(\bx_{i_k})|  \|\nabla f_{\bW(0)}(\bx_{i_k})\|_2 \le C_\sigma^L(\eta k)^{\frac{3}{2}}\sqrt{\log(m/\delta)/m}$.
    Plugging the above estimates into \eqref{eq:epsilon2}, we know with probability at least $1 - CL\exp(-m/2)-\delta$ over initialization $(\ba, \bW(0))$, it holds that
    \begin{align*}
        \|\epsilon_k\|_{\H_m}\le C_\sigma^L  \eta^{\frac{5}{2}} T^{\frac{3}{2}}\sqrt{\frac{\log(m/\delta)}{m}}.
    \end{align*}
Plugging the estimate of $\|\epsilon_s\|_{\H_m}$ back into \eqref{eq:sgd-flin-fm-sum}, with probability at least $1-CL\exp(-m/2) - \delta$ over initialization $(\ba, \bW(0))$, it holds that
\begin{align*}
    \|f^{\text{lin,SGD}}_{\bW(k) } - f_k^m\|_\infty \le C_\sigma^L\sum_{s=0}^k\|\epsilon_s\|_{\H_m} \le C_\sigma^L  (\eta T)^{\frac{5}{2}}\sqrt{\frac{\log(m/\delta)}{m}}.
\end{align*}
Taking $k=T$ completes the proof of the lemma.
\end{proof}

\subsubsection{(3) Control of $\E\big[\big\|\bS_m\big(f^m_T - g^{m,\text{GD}}_T\big)\big\|_\rho^2\big]$}\label{app:SGD-term3}

\begin{lemma}\label{lem:fmk-gmk}
    Suppose Assumption \ref{ass:activation} holds.
    Let $\delta \in (0,1)$ and $T\in\mathbb{N}$.
    Suppose 
    $
        0 < \eta \le \frac{1}{C_\sigma^L(\log (T) + 1)} \text{ and } \frac{1}{\eta T} \ge \frac{C_\sigma^L}{n}\log\big(\frac{n}{\delta}\big),
    $
    Then, with probability at least $1 - \delta$ over sampling, it holds that
    $$
        \ebb_\A\big[\big\|\bS_m(f^m_T - g^{m,\text{GD}}_T)\big\|^2_\rho\big] \le C_\sigma^L\eta(\log (T) \vee 1).
    $$
\end{lemma}

\begin{proof}
    We prove this lemma by applying Proposition 6 in \cite{lin2017optimal}, which provides upper bounds for $\ebb_\A[\|\S_\rho\omega_{T+1} - \S_\rho\nu_{T+1}\|_\rho^2]$.
    We first show that their assumptions are satisfied in our setting, and then apply their results with our Lemma \ref{lemma:<2} by showing that $\S_\rho\omega_{T+1} - \S_\rho\nu_{T+1}$ is equivalent to $\bS_m f^m_T - \bS_m g^m_T$.

    In the proof of Lemma \ref{lemma:gk-hk} we already showed that Assumption 1 in \cite{lin2017optimal} holds with $M=v=1$.
    Instead of using the notations $\bx, \langle \bx, \bx'\rangle_H $ and $ \S_\rho$ in \cite{lin2017optimal} for any $\bx,\bx'\in\X$, we use $K^m_\bx,$ $\langle K^m_\bx, K^m_{\bx'}\rangle_{\H_m} $ and $ \bS_m$ in our setting.
    Using Lemma \ref{lemma:<2} with $\kappa^2 = C_\sigma^L \ge \sup_{\bx\in\X}\|\nabla f_{\bW(0)}(\bx)\|_2^2 \ge \|Q\|_{op}$, $\Gamma = n$, $\zeta_i = K^m_{\bx_i}$, $Q = \int_\X K^m_\bx\otimes K^m_\bx d\rho_{\bx}$, and $\lambda = (\eta T)^{-1}$, we know their (47) holds with probability at least $1 - \delta$ over sampling
    Further, by taking $\eta_k = \eta$ for all $k\in[T]$, we know the term $\S_\rho\omega_{T+1} - \S_\rho\nu_{T+1}$ in \cite{lin2017optimal} is the same as our $\bS_m f^m_T - \bS_m g^m_T$.

    Then, applying Proposition 6 in \cite{lin2017optimal} with $\eta_k = \eta$, $b=1$, $M=v=1$, $\theta = 0$ and $\lambda = (\eta T)^{-1}$, we can obtain the desired results.
\end{proof}

\subsubsection{(4) Control of $\big\|\bS_mg^{m,\text{GD}}_T - \bS g^{\text{GD}}_T\big\|_\rho^2$ and the proof of Theorem \ref{pro:connection-sgd}}\label{app:SGD-term4}

The term $\big\|\bS_mg^{m,\text{GD}}_T - \bS g^{\text{GD}}_T\big\|_\rho^2$ is estimated in Lemma \ref{prop:gmk-gk}. 
Now, we give the proof of Theorem\ref{pro:connection-sgd}.
\begin{proof}[Proof of Theorem \ref{pro:connection-sgd}]
    Note
    \begin{align*}
        \ebb_\A\big[\big\|f_{\bW(T)}^{\text{SGD}} - \bS g_T^{\text{GD}}\big\|_\rho^2\big] \lesssim & \,\, \ebb_\A\big[\big\|f_{\bW(T)}^{\text{SGD}} - f^{\text{lin,SGD}}_{\bW(T)}\big\|_\rho^2 + \big\|f^{\text{lin,SGD}}_{\bW(T)} - \bS_mf^m_T\big\|_\rho^2  + \big\|\bS_m\big(f^m_T - g^{m,\text{GD}}_T\big)\big\|_\rho^2\big] \nonumber\\
        & + \big\|\bS_mg^{m,\text{GD}}_T - \bS g^{\text{GD}}_T\big\|_\rho^2.
    \end{align*}
    Combining Lemmas \ref{prop:f-flin_sgd}, \ref{prop:flin-kernel_sgd}, \ref{lem:fmk-gmk} with $\delta$ replaced by $\delta/4$ and Lemma \ref{prop:gmk-gk} with $\delta$ replaced by $\delta/4$ yields the desired results.
\end{proof}



\subsection{Minimax-Optimal Bounds of SGD}\label{app:minimax-sgd}
Finally, we give a general excess risk bound as follows. 
\begin{theorem}\label{thm:f-frho_sgd}
Let $\delta\in(0,1)$ and $T\in\mathbb{N}$.
Suppose Assumptions~\ref{ass:activation}, \ref{ass:effec_dim} and \ref{ass:frho_smth} hold.
Let $\bW(T)$ be the output of SGD with symmetric Gaussian initialization. 
Assume that $\eta \le 1/\big(C_\sigma^L(\log(T) + 1)\big)$, $1 \le \eta T \le n\big(C_\sigma^L\log(n/\delta)\big)^{-1}$ and $m \gtrsim C_\sigma^L d(\eta T)^5\log^2({m}/{\delta})$.
    With probability at least $1-L^2\exp(-C_\sigma d\log^2(m)) - \delta$ over initialization $(\ba,\bW(0))$ and sampling, it holds that 
    \begin{align*}
          \ebb_\A\big[\varepsilon_{risk}\big(f^{\text{SGD}}_{\bW(T)}\big)\big]  \lesssim \frac{C_\sigma^L d(\eta T)^5\log^2(m)}{m}  + C_\sigma^L\eta\log(T) 
       + C_\sigma^{L(\beta \vee 1)}  \Big(   \frac{\eta T}{n^{2}}   + \frac{(\eta T)^{\gamma\vee (1 - 2\beta)}}{n}    \Big) \log^4 ( \frac{T}{\delta})+ \frac{1}{(\eta T)^{ 2\beta}}.
    \end{align*}
\end{theorem}
\begin{proof} 
Note
\begin{align*}
    \ebb_\A\big[\big\|f^{\text{SGD}}_{\bW(T)} - f_\rho\big\|_\rho^2\big] \lesssim \ebb_\A\big[\big\|f^{\text{SGD}}_{\bW(T)} - \bS g_T^{\text{GD}}\big\|_\rho^2\big] + \big\|\bS g_T^{\text{GD}} - f_\rho\big\|_\rho^2.
\end{align*}
    Combining Theorem \ref{pro:connection-sgd} with $\delta$ replaced by $\delta/2$ and  Lemma \ref{prop:gT-frho} with $\delta$ replaced by $\delta/2$ yields the desired results.
\end{proof}
Based on Theorem~\ref{thm:f-frho_sgd}, we give  the proof of Theorem \ref{cor:f-frho_sgd}.
\begin{proof}[Proof of Theorem \ref{cor:f-frho_sgd}]
    It is easy to show that $1 \le \eta T \le n\big(C_\sigma^L\log(n/\delta)\big)^{-1}$ holds for $\eta=(C_\sigma^L\log(n/\delta))^{-1} n^{-\frac{2\beta}{2\beta+\gamma}} $,  $ T = \lceil n^{\frac{2\beta+1}{2\beta+\gamma}}\rceil$ and $n \ge (C_\sigma^{L}\log(n/\delta))^{2\beta+\gamma}$.
    Note that the choices of $\eta$ and $T$ and the assumption $2\beta+\gamma>1$ imply $C_\sigma^L\eta\log(T)
       + C_\sigma^{L(\beta \vee 1)}  \Big(   \frac{\eta T}{n^{2}}   + \frac{(\eta T)^{\gamma\vee 1 - 2\beta}}{n}    \Big) \log^4 ( \frac{T}{\delta})+ \frac{1}{(\eta T)^{ 2\beta}} \lesssim C_\sigma^L n^{-\frac{2\beta}{2\beta+\gamma}}\log^{2+2\beta}(\frac{n}{\delta})$. 
       
    Further, similar to the proof of Theorem \ref{cor:f-frho_gd}, one can also prove that $m\gtrsim C_\sigma^Ldn^{\frac{2\beta+5}{2\beta+\gamma}}\log^3\big(\frac{dLn}{\delta}\big)$ indicates $m \gtrsim C_\sigma^L d(\eta T)^5\log^3(m/\delta)$ and $\frac{C_\sigma^Ld(\eta T)^5\log^3(m/\delta)}{m} \lesssim C_\sigma^Ln^{-\frac{2\beta}{2\beta+\gamma}}$.
    In addition, note that $m \gtrsim C_\sigma^L d(\eta T)^5\log^3(m/\delta)$ implies $L^2\exp(-C_\sigma \log^2(m)) \le \delta/2$.
    Combining the above observations with Theorem \ref{thm:f-frho_sgd} with $\delta$ replaced by $\delta/2$ yields the desired results.
\end{proof}

\section{Uniform Concentration of the NTK under General Gaussian Initialization} \label{app:NTK}
In this section, we consider the standard Gaussian initialization,
which differs from the symmetric initialization adopted in the main text.
We show that the corresponding finite-width NTK concentrates uniformly
around a deterministic infinite-width kernel.

More precisely, for every $\ell \in [L]$, the entries of
$\mathbf{W}^{\ell}(0)$ and the output layer $\ba$ are independently drawn from
$\mathcal{N}(0,1)$.
Under this initialization, the layerwise gradients $\nabla_{\ell}f_{\mathbf{W}(0)}(\mathbf{x})$
are generally nonzero also for $\ell\in[L-1]$.
We define the finite-width NTK at each layer $\ell\in[L]$ as
\begin{align}
    K^{m,\ell}(\bx, \bx') &= \big\langle \nabla_\ell f_{\bW(0)}(\bx), \nabla_\ell f_{\bW(0)}(\bx')\big\rangle   = \langle o^{\ell-1}_0(\bx),o^{\ell-1}_0(\bx')\rangle \, \ba^\top \bV^\ell_{L,0}(\bx)^\top\bV^\ell_{L,0}(\bx')\ba .\nonumber
\end{align}
In the following theorem, we show that $K^{m,\ell}$ converges to an infinite-width NTK $K^\ell:\X\times\X\to\R$ uniformly over $\X\times\X$, which is given by
\begin{align*}
    K^\ell(\bx,\bx') = c_\sigma\ebb[\sigma(U^{\ell-1}(\bx))\sigma(U^{\ell-1}(\bx'))] \prod_{h=\ell}^{L}q^h(\bx,\bx'),
\end{align*}
where the definitions of $U^\ell(\bx)$ and $q^h(\bx)$ can be found in Appendix~\ref{appde:problem}.


\subsection{Concentration Theorem}
\begin{theorem}[Restatement of Theorem~\ref{thm:concentration}]\label{thm:re-concen} 
    Let $\delta\in(0,1)$.  
 Suppose Assumption~\ref{ass:activation} holds, and $m \gtrsim C_\sigma^L d\log({m}/{\delta})$. 
    With probability at least $1-L^2\exp(-C_\sigma \log^2(m)) - \delta$ over initialization $(\ba,\bW(0))$, for all $\ell\in[L]$, it holds that 
    \begin{align*}
        \|K^{m,\ell} - K^\ell \|_\infty \lesssim C_\sigma^L\sqrt{\frac{d\log^2(m)}{m}}.
    \end{align*}
\end{theorem}
The above theorem shows the uniform convergence of $K^{m,\ell}$ to $K^\ell$ as $m\to \infty$ for any layer $\ell\in[L]$,  and characterizes the rate of the convergence which is proportional to $1/\sqrt{m}$. This improvement helps us develop the desired excess risk rate.  


\smallskip

\noindent\textbf{Comparison with related works.} 
The initial work \cite{jacot2018neural} showed that the NTK converges in probability to an explicit deterministic limit, i.e., $K^{m} \rightarrow K$ in probability as $m\rightarrow \infty$.
The most related works are \cite{du2019gradient} and \cite{xu2024overparametrized}, both of which provided concentration of the NTK for $L$-layer fully connected neural networks. Specifically, \cite{du2019gradient} established  the concentration of the last layer $K^{m,L}$ for a bounded number of data points $S=\{\bz_i\}_{i=1}^n$ when the smooth and Lipschitz activation functions are considered.
They showed that $\|\bK^{m,L} -\bK^L\|_{op}\le \lambda_0/4$ with high probability if $m\gtrsim  {n^2}/{\lambda_0^2}$, where $\bK^{m,L}$ and $\bK^L$ are gram matrices with kernels $K^{m,L}$ and $K^L$, respectively, and $\lambda_0$ is the smallest eigenvalue of $\bK^L$. Instead, we provide an upper bound independent of $\lambda_0$ uniformly over any $\bx, \bx' \in\X$ and all $L$ layers under the condition $m\gtrsim  C_\sigma^L d$, which is a much stronger result than that of \cite{du2019gradient}.  Recently, \cite{xu2024overparametrized} focused on ReLU networks and proved that $\| K^{m,\ell}-K^{\ell} \|_\infty\lesssim \frac{C^{L}}{m^{1/6}}+ \sqrt{\frac{dL\log(m)}{m} }$ for all $\ell\in[L]$ if $m\gtrsim d^2 \exp(L^2)$. We significantly improve their convergence rate from being proportional to $m^{-1/6}$ to $m^{-1/2}$ by carefully using the smoothness property of the loss (see Lemma~\ref{lem:trace-prod}  for more details). 

\subsection{Proof of the Concentration Theorem}
Following the approach of \cite{xu2024overparametrized}, we bound $|K^{m,\ell}(\bx,\bx') - K^\ell(\bx,\bx')|$ by three core deviation terms $|\langle o^{\ell-1}_0(\bx),  o^{\ell-1}_0(\bx')\rangle - c_\sigma\ebb[\sigma(U^{\ell-1}(\bx))\sigma(U^{\ell-1}(\bx'))]|$, $|\ba^\top\bV^\ell_{L,0}(\bx)^\top\bV^\ell_{L,0}(\bx')\ba - tr(\bV^\ell_{L,0}(\bx)^\top\bV^\ell_{L,0}(\bx'))|$ and $|tr(\bV^\ell_{L,0}(\bx)^\top\bV^\ell_{L,0}(\bx')) - \prod_{h=\ell}^Lq^h(\bx,\bx')|$.
We will estimate them separately.

We first prove that $\langle o^{\ell-1}_0(\bx),o^{\ell-1}_0(\bx')\rangle$ converges to $c_\sigma\ebb[\sigma(U^{\ell-1}(\bx))\sigma(U^{\ell-1}(\bx'))]$ uniformly on $\X\times\X$.

The following lemma can be derived from Theorem E.1 in \cite{du2019gradient}.
Note that in the following lemma, the training dataset $S$ is replaced by any finite subset $\A$ of $\X$.
\begin{lemma}\label{lem:o-U_pointwise}
    Let $\A\subset \X$ with cardinality $|\A| = p$.
    With probability at least $1 - \delta$ over initialization $\bW(0)$, for all $\ell\in[L]$ and $\bx,\bx'\in\A$, it holds that
    \begin{align}
        \big|\langle &o^{\ell-1}_0(\bx),o^{\ell-1}_0(\bx')\rangle - c_\sigma\ebb[\sigma(U^{\ell-1}(\bx))\sigma(U^{\ell-1}(\bx'))]\big|\lesssim C^L\sqrt{\frac{\log(pL/\delta)}{m}}.
    \end{align}
\end{lemma}
\begin{proof}
For completeness, we provide a brief proof here. By applying Theorem E.1 in \cite{du2019gradient} with $h=\ell$, $H = L$, $\rho^{(h)}(\cdot) = \sqrt{c_\sigma}\sigma(\cdot)$, $p^{(h)} = 1$, $\D^{(h)} = 0$, $\mathcal{W}^{(h)}(t) = t$, $(\bW^{(h),(\alpha)}_{r})_{r=1}^m = \bw^\ell_\alpha(0)$, $\frac{1}{m}\sum_{r=1}^m(\bX^{(h),(\alpha)}_i)^\top\bX^{(h),(\alpha)}_i = \langle o^{\ell-1}_0(\bx),o^{\ell-1}_0(\bx')\rangle$ and $\bK^{(h)}_{i,j} = c_\sigma\ebb[\sigma(U^{\ell-1}(\bx))\sigma(U^{\ell-1}(\bx'))]$, we can get the desired results.
\end{proof}
The above lemma applies only to a finite subset $\mathcal{A}$. In the following lemma, we strengthen this result by extending it to the entire space. 
\begin{lemma}\label{lem:o-U_uniform}
    With probability at least $1-L\exp(-Cm)-\delta$ over initialization $\bW(0)$, for all $\ell\in[L]$, it holds that
    \begin{align*}
        \sup_{\bx,\bx'\in\X}\big|& \langle o^{\ell-1}_0(\bx),o^{\ell-1}_0(\bx')\rangle - c_\sigma\ebb[\sigma(U^{\ell-1}(\bx))\sigma(U^{\ell-1}(\bx'))]\big|\lesssim C_\sigma^L\sqrt{\frac{d\log(mL/\delta)}{m}}.
    \end{align*}
\end{lemma}
\begin{proof}
    Recall from the definitions that $\langle o^{0}_0(\bx),o^{0}_0(\bx')\rangle = c_\sigma\ebb[\sigma(U^{0}(\bx))\sigma(U^{0}(\bx'))] = \langle\bx, \bx'\rangle$ for all $\bx,\bx'\in\X$.
    Then, the case for $\ell=1$ is trivial.

    Now, we consider the case $\ell\ge 2$.
    Let $\A\subset\X$ be a $m^{-\frac{1}{2}}$-net of $\X$.
    We know for any $\bx,\bx'\in\X$, there exists $\bx_0,\bx_0'\in\A$ such that $\|\bx - \bx_0\|_2 \le \frac{1}{\sqrt{m}}$ and $\|\bx' - \bx_0'\|_2 \le \frac{1}{\sqrt{m}}$.
    Then, we have
    \begin{align}
        &\big|\langle o^{\ell-1}_0(\bx),o^{\ell-1}_0(\bx')\rangle - c_\sigma\ebb[\sigma(U^{\ell-1}(\bx))\sigma(U^{\ell-1}(\bx'))]\big|\nonumber\\
        &\le \big|\langle o^{\ell-1}_0(\bx),o^{\ell-1}_0(\bx')\rangle - \langle o^{\ell-1}_0(\bx_0),o^{\ell-1}_0(\bx_0')\rangle\big|  + \big|\langle o^{\ell-1}_0(\bx_0),o^{\ell-1}_0(\bx_0')\rangle - c_\sigma\ebb[\sigma(U^{\ell-1}(\bx_0))\sigma(U^{\ell-1}(\bx_0'))]\big|\nonumber\\
        &\quad + c_\sigma\big|\ebb[\sigma(U^{\ell-1}(\bx_0))\sigma(U^{\ell-1}(\bx_0'))]   - \ebb[\sigma(U^{\ell-1}(\bx))\sigma(U^{\ell-1}(\bx'))]\big|\nonumber\\
        &=: I_1 + I_2 + I_3.
    \end{align}
    The term $I_2$ can be directly  controlled by using Lemma \ref{lem:o-U_pointwise}. 
    To bound $I_1$, we first show that $o^\ell_0(\cdot)$ is Lipschitz continuous for all $\ell\in[L]$.
    For any $\bx, \bx'\in\X$, it holds that
    \begin{align*}
        &\|o^\ell_0(\bx) - o^\ell_0(\bx')\|_2 
         = \sqrt{\frac{c_\sigma}{m}}\big\|\sigma\big(\bW^\ell(0)o^{\ell-1}_0(\bx)\big) - \sigma\big(\bW^\ell(0)o^{\ell-1}_0(\bx')\big)\big\|_2\\
        &\le \sqrt{\frac{c_\sigma}{m}}B_{\sigma'}\big\|\bW^\ell(0)\big(o^{\ell-1}_0(\bx) - o^{\ell-1}_0(\bx')\big)\|_2 \le \sqrt{\frac{c_\sigma}{m}}B_{\sigma'}\big\|\bW^\ell(0)\big\|_{op}\|o^{\ell-1}_0(\bx) - o^{\ell-1}_0(\bx')\|_2\\
        &\le \sqrt{c_\sigma}B_{\sigma'}c_0\|o^{\ell-1}_0(\bx) - o^{\ell-1}_0(\bx')\|_2,
    \end{align*}
    where the first inequality used Assumption \ref{ass:activation} and the last inequality is according to Lemma \ref{lemma:oprt_norm}.
    Applying the above inequality recursively, we know
    \begin{align}\label{eq:lip_o}
        \|o^\ell_0(\bx) - o^\ell_0(\bx')\|_2 \le \big(\sqrt{c_\sigma}B_{\sigma'}c_0\big)^\ell \|\bx-\bx'\|_2.
    \end{align}
    Then, we can control $I_1$ as
    \begin{align*}
        I_1 &\le \big|\langle o^{\ell-1}_0(\bx),o^{\ell-1}_0(\bx')\rangle - \langle o^{\ell-1}_0(\bx_0),o^{\ell-1}_0(\bx')\rangle\Big|  + \big|\langle o^{\ell-1}_0(\bx_0),o^{\ell-1}_0(\bx')\rangle - \langle o^{\ell-1}_0(\bx_0),o^{\ell-1}_0(\bx_0')\rangle\big|\\
        &\le \|o^{\ell-1}_0(\bx) - o^{\ell-1}_0(\bx_0)\|_2\|o^{\ell-1}_0(\bx')\|_2 + \|o^{\ell-1}_0(\bx') - o^{\ell-1}_0(\bx_0')\|_2\|o^{\ell-1}_0(\bx_0)\|_2\\
        &\le \frac{\big(\sqrt{c_\sigma}B_{\sigma'}c_0\big)^{\ell-1}}{\sqrt{m}}\big(\|o^{\ell-1}_0(\bx')\|_2 + \|o^{\ell-1}_0(\bx_0)\|_2\big) \le \frac{2B_\sigma \sqrt{c_\sigma}\big(\sqrt{c_\sigma}B_{\sigma'}c_0\big)^{\ell-1}}{\sqrt{m}} \lesssim \frac{C_\sigma^L}{\sqrt{m}},
    \end{align*}
    where the third inequality used \eqref{eq:lip_o} and the fact $\|\bx-\bx_0\|_2, \|\bx'-\bx_0'\|_2 \le m^{-\frac{1}{2}}$, and the last-second inequality used the fact $\sup_{\bx\in\X}\|o^\ell_0(\bx)\|_2 \le \sqrt{c_\sigma} B_\sigma$ by noting $\sup_{a\in\R}|\sigma(a)| \le B_\sigma$ (see Assumption \ref{ass:activation}).

    Now, we turn to estimate $I_3$.
    Given \eqref{eq:EU=1}, applying Lemma G.4 in \cite{du2019gradient} with $a_1=b_1=a_2=b_2 = 1$, we have
    \begin{align*}
        I_3 &= c_\sigma\big|\ebb[\sigma(U^{\ell-1}(\bx_0))\sigma(U^{\ell-1}(\bx_0'))] - \ebb[\sigma(U^{\ell-1}(\bx))\sigma(U^{\ell-1}(\bx'))]\big|\\
        &\le 2c_\sigma^2\big|\ebb[\sigma(U^{\ell-2}(\bx_0))\sigma(U^{\ell-2}(\bx_0'))]  - \ebb[\sigma(U^{\ell-2}(\bx))\sigma(U^{\ell-2}(\bx'))]\big|.
    \end{align*}
    By applying the above inequality recursively, it holds that
    \begin{align*}
        I_3 &\le c_\sigma(2c_\sigma)^{\ell-2}\big|\ebb[\sigma(U^1(\bx_0))\sigma(U^1(\bx_0'))]  - \ebb[\sigma(U^1(\bx))\sigma(U^1(\bx'))]\big|\\
        &\le c_\sigma(2c_\sigma)^{\ell-1}\big|\langle \bx_0, \bx_0'\rangle - \langle \bx, \bx'\rangle\big|\le c_\sigma(2c_\sigma)^{\ell-1}\big( \|\bx_0 - \bx\|_2\|\bx_0'\|_2 + \|\bx_0' - \bx'\|_2\|\bx\|_2 \big) \le \frac{(2c_\sigma)^{\ell}}{\sqrt{m}} \lesssim \frac{C_\sigma^L}{\sqrt{m}}.
    \end{align*}

    Note that $\X = S^{d-1}$ is the unit sphere in $\R^d$ and $\A$ is a $m^{-\frac{1}{2}}$-net of $\X$.
    From Corollary 4.2.13 in \cite{vershynin2018high}, we know the covering number of $\X$ satisfies $|\A|\le (3\sqrt{m})^d.$
    Combining Lemma \ref{lem:o-U_pointwise} with $p = (3\sqrt{m})^d$ and the estimates for $I_1$ and $I_3$, we get the desired results.
\end{proof}

Recall that in the previous section we defined $q^\ell(\bx,\bx') =  c_\sigma\ebb[\sigma'(U^{\ell}(\bx))\sigma'(U^{\ell}(\bx'))]$ and \[\bV^\ell_{L,0}(\bx) = \big(\big[\prod_{h=\ell+1}^L\sqrt{\frac{c_\sigma}{m}}\bD^h_0(\bx)\bW^h(0)\big]\sqrt{\frac{c_\sigma}{m}}\bD^\ell_0(\bx)\big)^\top.\] 
Now, we are in a position to show that  $\ba^\top \bV^\ell_{L,0}(\bx)^\top\bV^\ell_{L,0}(\bx)\ba$ converges to $\prod_{h=\ell}^{L}q^h(\bx,\bx')$ uniformly.
First, we present the following lemma that provides an estimate of $ \|\bV^{\ell}_{\ell_1,0}(\bx)\|_{op}$.
\begin{lemma}\label{lem:bound_V0}
    With probability at least $1-L\exp(-Cm)$ over initialization $\bW(0)$, for any $1\le\ell \le \ell_1 \le L$, it holds that
    \begin{align*}       \sup_{\bx}\big\|\bV^{\ell}_{\ell_1,0}(\bx)\big\|_{op} \le \frac{(B_{\sigma'}\sqrt{c_\sigma})^{{\ell_1 - \ell}+1}c_0^{{\ell_1 - \ell}}}{\sqrt{m}} \lesssim \frac{C_\sigma^L}{\sqrt{m}}.
    \end{align*}
\end{lemma}
\begin{proof}
    For all $\bx\in\X$, from the definition of $\bV^{\ell}_{\ell_1,0}(\bx)$ we have
    \begin{align*}
        &\big\|\bV^{\ell}_{\ell_1,0}(\bx)\big\|_{op} 
        \le \bigg[\prod_{h=\ell+1}^{\ell_1}\sqrt{\frac{c_\sigma}{m}}\big\|\bD^h_0(\bx)\big\|_{op}\big\|\bW^h(0)\big\|_{op}\bigg]\sqrt{\frac{c_\sigma}{m}}\big\|\bD^{\ell}_0(\bx)\big\|_{op}\\
        &\le \big(\sup_{a\in\R}|\sigma'(a)|\big)^{\ell_1-\ell+1}\frac{(\sqrt{c_\sigma})^{{\ell_1 - \ell}+1}}{\sqrt{m}}\prod_{h=\ell+1}^{\ell_1}\frac{1}{\sqrt{m}}\big\|\bW^h(0)\big\|_{op} \le \frac{(B_{\sigma'}\sqrt{c_\sigma})^{{\ell_1 - \ell}+1}c_0^{{\ell_1 - \ell}}}{\sqrt{m}} \lesssim \frac{C_\sigma^L}{\sqrt{m}},
    \end{align*}
    where the third inequality used Assumption \ref{ass:activation} and Lemma \ref{lemma:oprt_norm}. The proof is completed. 
\end{proof}

The following lemma shows that $\bV^{\ell}_{\ell_1,0}(\cdot)$ is Lipschitz continuous with respect to the operator norm for all $1\le\ell\le\ell_1\le L$.
\begin{lemma}\label{lem:lip_V0}
    With probability at least $1 - Lm\exp(-Cm)$ over initialization $\bW(0)$, for all $1\le\ell\le\ell_1\le L$ and $\bx,\bx'\in\X$, we have
    \begin{align*}
       \big\|\bV^{\ell}_{\ell_1,0}(\bx) - \bV^{\ell}_{\ell_1,0}(\bx')\big\|_{op} \lesssim C_\sigma^L\|\bx-\bx'\|_2.
    \end{align*}
\end{lemma}
\begin{proof}
    We first show that $\bD^\ell_0(\cdot)$ is Lipschitz continuous for all $\ell\in[L]$.
    From Bernstein inequality (Eq (3.1) in \cite{vershynin2018high} with $X = \bw^{\ell_1+1}_r(0)$, $n=m$, $u=1$), we know with probability at least $1-Lm\exp(-Cm)$ over initialization $\bW(0)$, it holds that $\|\bw^{\ell}_r(0)\|_2 \le 2\sqrt{m}$ for all $r\in[m]$ and $\ell\in[L]$.
    Then, with this high probability, for any $\bx,\bx'\in\X$, we have
    \begin{align}\label{eq:lip_D0}
        &\big\|\bD^\ell_0(\bx) - \bD^\ell_0(\bx')\big\|_{op}  = \max_{r\in[m]}\big|\sigma'(\bw^\ell_r(0)^\top o^{\ell-1}_0(\bx)) - \sigma'(\bw^\ell_r(0)^\top o^{\ell-1}_0(\bx'))\big|\nonumber\\
        &\le \sup_{t\in\R,r\in[m]}|\sigma''(t)|\big|\bw^\ell_r(0)^\top \big(o^{\ell-1}_0(\bx) - o^{\ell-1}_0(\bx')\big)\big| \le B_{\sigma''}\max_{r\in[m]}\|\bw^\ell_r(0)\|_2\big\|o^{\ell-1}_0(\bx) - o^{\ell-1}_0(\bx')\big\|_2\nonumber\\
        &\le 2B_{\sigma''} \big(\sqrt{c_\sigma}B_{\sigma'}c_0\big)^{\ell-1}\sqrt{m}\|\bx-\bx'\|_2,
    \end{align}
    where in the last inequality we have used \eqref{eq:lip_o}.
    Then, for any $\bx,\bx'\in\X$, it holds that
    \begin{align*}
        &\big\|\bV^{\ell}_{\ell_1,0}(\bx) - \bV^{\ell}_{\ell_1,0}(\bx')\big\|_{op}\\
        &\le \Big\|\bV^{\ell}_{\ell_1,0}(\bx) - \sqrt{\frac{c_\sigma}{m}}\bD^{\ell_1}_0(\bx')\bW^{\ell_1}(0)\bV^{\ell}_{{\ell_1}-1,0}(\bx)\Big\|_{op} + \Big\|\sqrt{\frac{c_\sigma}{m}}\bD^{\ell_1}_0(\bx')\bW^{\ell_1}(0)\bV^{\ell}_{{\ell_1}-1,0}(\bx) - \bV^{\ell}_{{\ell_1},0}(\bx')\Big\|_{op}\\
        &= \Big\|\sqrt{\frac{c_\sigma}{m}}\big(\bD^{\ell_1}_0(\bx) - \bD^{\ell_1}_0(\bx')\big)\bW^{\ell_1}(0)\bV^{\ell}_{{\ell_1}-1,0}(\bx)\Big\|_{op} + \Big\|\sqrt{\frac{c_\sigma}{m}}\bD^{\ell_1}_0(\bx')\bW^{\ell_1}(0) \big(\bV^{\ell}_{{\ell_1}-1,0}(\bx) - \bV^{\ell}_{{\ell_1}-1,0}(\bx')\big)\Big\|_{op}\\
        &\le \sqrt{\frac{c_\sigma}{m}}\big\|\bD^{\ell_1}_0(\bx) \!-\! \bD^{\ell_1}_0(\bx')\big\|_{op}\big\|\bW^{\ell_1}(0)\big\|_{op}\big\|\bV^{\ell}_{{\ell_1}-1,0}(\bx)\big\|_{op} \!+\! \sqrt{\frac{c_\sigma}{m}}\big\|\bD^{\ell_1}_0(\bx')\big\|_{op}\big\|\bW^{\ell_1}(0)\big\|_{op}\big\|\bV^{\ell}_{{\ell_1}-1,0}(\bx) \\& \quad - \bV^{\ell}_{{\ell_1}-1,0}(\bx')\big\|_{op}\\
        &\le 2\sqrt{c_\sigma}c_0B_{\sigma''}\big(\sqrt{c_\sigma}B_{\sigma'}c_0\big)^{\ell_1-1}\sqrt{m}\|\bx-\bx'\|_2\big\|\bV^{\ell}_{{\ell_1}-1,0}(\bx)\big\|_{op}  + \sqrt{c_\sigma}c_0B_{\sigma'}\big\|\bV^{\ell}_{{\ell_1}-1,0}(\bx) - \bV^{\ell}_{{\ell_1-1},0}(\bx')\big\|_{op}\\
        &\le (\sqrt{c_\sigma})^{2\ell_1 - \ell+1}c_0^{2\ell_1 - \ell}B_{\sigma'}^{2{\ell_1}-\ell}B_{\sigma''}\|\bx - \bx'\|_2  + \sqrt{c_\sigma}c_0B_{\sigma'}\big\|\bV^{\ell}_{{\ell_1}-1,0}(\bx) - \bV^{\ell}_{{\ell_1}-1,0}(\bx')\big\|_{op}\\
        &=: A_{\ell_1}\|\bx-\bx'\|_2 + \Tilde{c}\big\|\bV^{\ell}_{{\ell_1}-1,0}(\bx) - \bV^{\ell}_{{\ell_1}-1,0}(\bx')\big\|_{op},
    \end{align*}
    where $A_{\ell_1} = (\sqrt{c_\sigma})^{2\ell_1-\ell+1}c_0^{2\ell_1-\ell}B_{\sigma'}^{2{\ell_1} - \ell}B_{\sigma''}$, $\Tilde{c} = \sqrt{c_\sigma}c_0B_{\sigma'}$, and in the third inequality we have used \eqref{eq:lip_D0}, Assumption \ref{ass:activation} and Lemma \ref{lemma:oprt_norm}, and in the last inequality we have used Lemma \ref{lem:bound_V0} and \eqref{eq:lip_o}.

    By applying the above inequality recursively, we get
    \begin{align*}
         \big\|\bV^{\ell}_{{\ell_1},0}(\bx) - \bV^{\ell}_{{\ell_1},0}(\bx')\big\|_{op}  
        & \le \Big(\sum_{h=0}^{\ell_1-\ell-1} \Tilde{c}^hA_{\ell_1-h} + \Tilde{c}^{\ell_1-\ell}\sqrt{c_\sigma}B_{\sigma''}\Big)\|\bx-\bx'\|_2 \lesssim LC_\sigma^L\|\bx-\bx'\|_2 \lesssim C_\sigma^L\|\bx-\bx'\|_2,
    \end{align*}
    where the last inequality used $L\lesssim C_{\sigma}^L$.
    This completes the proof of the lemma. 
\end{proof}

We introduce the Hanson-Wright inequality, which can be found in Theorem 6.21 of \cite{vershynin2018high}.
\begin{lemma}[Hanson-Wright inequality]\label{lem:hanson-wright}
    Let $\bA\in\R^{m\times m}$ be a matrix.
    Suppose $\bu\sim\N(0,\bfI_m)$.
    Then, with probability at least $1 - 2\exp\big(-C\frac{t^2}{\|\bA\|^2_2}\wedge \frac{t}{\|\bA\|_{op}}\big)$, it holds that
    \begin{align*}
        |\bu^\top \bA\bu - tr(\bA)| \le t.
    \end{align*}
\end{lemma}

\begin{lemma}\label{lem:quadratic-trace}
    With probability at least $1- L^2\exp\big(-C_\sigma d\log^2(m)\big)$ over initialization $(\ba, \bW(0))$, the following statements hold for all $1\le\ell\le\ell_1\le L-1$
    \[\sup_{r\in[m]}\sup_{\bx,\bx'\in\X}\big|\bw^{\ell_1+1}_r(0)^\top \bV^{\ell}_{\ell_1,0}(\bx)^\top\bV^{\ell}_{\ell_1,0}(\bx')\bw^{\ell_1+1}_r(0) -tr\big(\bV^{\ell}_{\ell_1,0}(\bx)^\top\bV^{\ell}_{\ell_1,0}(\bx')\big)\big| \lesssim C_\sigma^L\sqrt{\frac{d\log^2(m)}{m}},\]
        \[\sup_{\bx,\bx'\in\X}\big|\ba^\top \bV^{\ell}_{L,0}(\bx)^\top\bV^{\ell}_{L,0}(\bx')\ba -tr\big(\bV^{\ell}_{L,0}(\bx)^\top\bV^{\ell}_{L,0}(\bx')\big)\big| \lesssim C_\sigma^L\sqrt{\frac{d\log^2(m)}{m}}.\]
\end{lemma}
\begin{proof}
    We only prove the first inequality, the second inequality can be proved in a similar way.   Let $\A\subset\X$ be a $m^{-1}$-net of $\X$.
    Similarly to the proof of Lemma \ref{lem:o-U_uniform}, we can show that $|\A|\le (3m)^d.$
    Define $\bA^{\ell}_{\ell_1}(\bx,\bx'):= \bV^{\ell}_{{\ell_1},0}(\bx)^\top\bV^{\ell}_{{\ell_1},0}(\bx')$ for $\bx,\bx'\in\X$.
    We first estimate $\|\bA^{\ell}_{\ell_1}(\bx,\bx')\|_{op}$ and $\|\bA^{\ell}_{\ell_1}(\bx,\bx')\|_{2}$.
    From Lemma \ref{lem:bound_V0} we know for any $\bx,\bx'\in\X$, it holds that $\|\bA^{\ell}_{\ell_1}(\bx,\bx')\|_{op}\le \|\bV^{\ell}_{{\ell_1},0}(\bx)\|_{op}\|\bV^{\ell}_{{\ell_1},0}(\bx')\|_{op} \lesssim C_\sigma^Lm^{-1}$.
    Then, we have $\|\bA^{\ell}_{\ell_1}(\bx,\bx')\|_{2} \le \sqrt{m}\|\bA^{\ell}_{\ell_1}(\bx,\bx')\|_{op} \lesssim C_\sigma^Lm^{-\frac{1}{2}}.$

    Applying Lemma \ref{lem:hanson-wright} with $\bu = \bw^{\ell_1+1}_r(0)$, $\bA = \bA^{\ell}_{\ell_1}(\bx_0,\bx_0')$, and taking union bound over all $\bx_0,\bx_0'\in\A$ and all $r\in[m]$, $1\le\ell\le\ell_1\le L-1$ (note the number of all the probability events is $|\A|^2\cdot m \cdot (L-1)^2$ which can be bounded by $3^{2d}L^2m^{2d+1}$) and set $t = C_\sigma^L\log(m)\sqrt{d/m}$, we know with probability at least $1-3^{2d}L^2m^{2d+1}\exp(-C_\sigma d\log^2(m))$ over initialization $(\ba, \bW(0))$, it holds that for all $1\le\ell\le\ell_1\le L -1$
    \begin{align*}
        \sup_{r\in[m]}\sup_{\bx_0,\bx_0'\in\A}  \big|\bw^{\ell_1+1}_r(0)^\top \bA^{\ell}_{\ell_1}(\bx_0,\bx_0')\bw^{\ell_1+1}_r(0) - tr\big(\bA^{\ell}_{\ell_1}(\bx_0,\bx_0')\big)\big| \lesssim C_\sigma^L\sqrt{\frac{d\log^2(m)}{m}}.
    \end{align*}
    For any $\bx,\bx'\in\X$, $r\in[m]$ and $1\le\ell\le\ell_1\le L-1$, from the above inequality we have
    \begin{align*}
        &\big|\bw^{\ell_1+1}_r(0)^\top \bA^{\ell}_{\ell_1}(\bx,\bx')\bw^{\ell_1+1}_r(0) - tr\big(\bA^{\ell}_{\ell_1}(\bx,\bx')\big)\big|\\
        &\le \big|\bw^{\ell_1+1}_r(0)^\top \bA^{\ell}_{\ell_1}(\bx,\bx')\bw^{\ell_1+1}_r(0) - \bw^{\ell_1+1}_r(0)^\top \bA^{\ell}_{\ell_1}(\bx_0,\bx_0')\bw^{\ell_1+1}_r(0)\big|\\
        &\quad + \big|\bw^{\ell_1+1}_r(0)^\top \bA^{\ell}_{\ell_1}(\bx_0,\bx_0')\bw^{\ell_1+1}_r(0) -  tr\big(\bA^{\ell}_{\ell_1}(\bx_0,\bx_0')\big)\big|  + \big|tr\big(\bA^{\ell}_{\ell_1}(\bx_0,\bx_0')\big) - tr\big(\bA^{\ell}_{\ell_1}(\bx,\bx')\big)\big|\\
        &\lesssim I_1 + C_\sigma^L\sqrt{\frac{d\log^2(m)}{m}} + I_3,
    \end{align*}
    where we define $I_1:= |\bw^{\ell_1+1}_r(0)^\top\! \bA^{\ell}_{\ell_1}(\bx,\bx')\bw^{\ell_1+1}_r(0) - \bw^{\ell_1+1}_r(0)^\top\! \bA^{\ell}_{\ell_1}(\bx_0,\bx_0')\bw^{\ell_1+1}_r(0)|$ and $I_3:= |tr\big(\bA^{\ell}_{\ell_1}(\bx_0,\bx_0')\big) $ $- tr\big(\bA^{\ell}_{\ell_1}(\bx,\bx')\big)|$.
    Now, we turn to estimate $I_1$ and $I_3$, separately.
    According to Bernstein inequality (Eq (3.1) in \cite{vershynin2018high} with $X = \bw^{\ell_1+1}_r(0)$, $n=m$, $u=1$), we know with probability at least $1-Lm\exp(-Cm)$ over initialization $\bW(0)$, it holds that $\|\bw^{\ell_1+1}_r(0)\|_2 \le 2\sqrt{m}$ for all $r\in[m]$ and $\ell_1\in[L-1]$.
    Then, we have
    \begin{align*}
        I_1 &\le \|\bw^{\ell_1+1}_r(0)\|_2\big\|\bA^{\ell}_{\ell_1}(\bx,\bx') - \bA^{\ell}_{\ell_1}(\bx_0,\bx_0')\big\|_{op}\|\bw^{\ell_1+1}_r(0)\|_2 \le 4m \big\|\bA^{\ell}_{\ell_1}(\bx,\bx') - \bA^{\ell}_{\ell_1}(\bx_0,\bx_0')\big\|_{op}\\
        &\le 4m \big(\big\|\bA^{\ell}_{\ell_1}(\bx,\bx') - \bV^{\ell}_{{\ell_1},0}(\bx_0)^\top\bV^{\ell}_{{\ell_1},0}(\bx')\big\|_{op} + \big\|\bV^{\ell}_{{\ell_1},0}(\bx_0)^\top\bV^{\ell}_{{\ell_1},0}(\bx') - \bA^{\ell}_{\ell_1}(\bx_0,\bx_0')\big\|_{op}\big)\\
        &= 4m \big(\big\|\big(\bV^{\ell}_{{\ell_1},0}(\bx)^\top - \bV^{\ell}_{{\ell_1},0}(\bx_0)^\top\big)\bV^{\ell}_{{\ell_1},0}(\bx')\big\|_{op} + \big\|\bV^{\ell}_{{\ell_1},0}(\bx_0)^\top \big(\bV^{\ell}_{{\ell_1},0}(\bx') - \bV^{\ell}_{{\ell_1},0}(\bx_0')\big)\big\|_{op}\big)\\
        &\le 4m \big(\big\|\bV^{\ell}_{{\ell_1},0}(\bx)^\top - \bV^{\ell}_{{\ell_1},0}(\bx_0)^\top\big\|_{op}\big\|\bV^{\ell}_{{\ell_1},0}(\bx')\big\|_{op} + \big\|\bV^{\ell}_{{\ell_1},0}(\bx_0) \big\|_{op}\big\|\bV^{\ell}_{{\ell_1},0}(\bx') - \bV^{\ell}_{{\ell_1},0}(\bx_0')\big)\big\|_{op}\big)\\
        &\lesssim C_\sigma^L\sqrt{m}(\|\bx - \bx_0\|_2 + \|\bx' - \bx_0'\|_2)\lesssim \frac{C_\sigma^L}{\sqrt{m}},
    \end{align*}
    where the last second inequality used Lemmas \ref{lem:bound_V0} and \ref{lem:lip_V0} and the last inequality used $\|\bx - \bx_0\|_2, \|\bx' - \bx_0'\|_2 \le m^{-1}$ since $\A$ is an $m^{-1}$-net of $\bX$.

    Note that $\ebb_{\bW^{\ell_1+1}(0)}[\bw^{\ell_1+1}_r(0)^\top\bA^{\ell}_{\ell_1}(\bx,\bx')\bw^{\ell_1+1}_r(0)] = tr(\bA^{\ell}_{\ell_1}(\bx,\bx'))$.
    Then, we can show $I_3\lesssim C_\sigma^Lm^{-\frac{1}{2}}$ in a similar way.

    We know the failure probabilities satisfy $3^{2d}L^2m^{2d+1}\!\exp(-C_\sigma d\log^2(m)) + Lm\exp(-Cm) \asymp L^2\exp(-C_\sigma d\log^2(m))$.
    Combining the above estimates yields the desired results.
\end{proof}

Recall for any $\ell\in[L]$, $tr(\bV^{\ell}_{\ell,0}(\bx)^\top\bV^{\ell}_{\ell,0}(\bx')) = \frac{c_\sigma}{m}\sum_{r=1}^m\sigma'(\bw^{\ell}_r(0)^\top o^{\ell-1}_0(\bx))\sigma'(\bw^{\ell}_r(0)^\top o^{\ell-1}_0(\bx'))$.
The following lemma shows that $tr(\bV^{\ell}_{\ell,0}(\bx)^\top\bV^{\ell}_{\ell,0}(\bx'))$ converges to its expectation with respect to $\bW^{\ell}(0)$ uniformly on $\X\times\X$.
\begin{lemma}\label{lem:sigma'-E[sigma']}
    With probability at least $1 - Lm\exp(-C_\sigma\log^2(m))$ over initialization $\bW(0)$, we have for all $\ell\in[L]$,
    \begin{align*}
        \sup_{\bx,\bx'\in\X}&c_\sigma\Big|\frac{1}{m}\sum_{r=1}^m\sigma'\big(\bw^{\ell}_r(0)^\top o^{\ell-1}_0(\bx)\big)\sigma'\big(\bw^{\ell}_r(0)^\top o^{\ell-1}_0(\bx')\big) - \ebb_{\bw\sim\N(0,\bfI)}\big[\sigma'\big(\bw^\top o^{\ell-1}_0(\bx)\big)\sigma'\big(\bw^\top o^{\ell-1}_0(\bx')\big)\big]\Big|\\
        &\lesssim C_\sigma \sqrt{\frac{dL\log^2(m)}{m}}.
    \end{align*}
\end{lemma}
\begin{proof}
    Let $\F = \{f(\cdot) = c_\sigma\sigma'((\cdot)^\top o^{\ell-1}_0(\bx))\sigma'((\cdot)^\top o^{\ell-1}_0(\bx')) : \bx, \bx'\in\X\}$.
    Then, the error term to be bounded can be written as $\sup_{f\in\F}|\frac{1}{m}\sum_{r=1}^mf(\bw^\ell_r(0)) - \ebb[f(\bw)]|$.
    Note this term can be regarded as a function on $(\bw^\ell_1(0),\ldots,\bw^\ell_m(0))$.
    According to Assumption~\ref{ass:activation}, one can check that the value of this function can change by at most $2B_{\sigma'}^2/m$ under an arbitrary change of the $r$-th coordinate.
    Then, by McDiarmid's inequality, we know with probability at least $1-\exp(-C_\sigma\log^2(m))$ over $\bW^\ell(0)$, there holds
    \begin{align*}
        \sup_{f\in\F}\Big|\frac{1}{m}\sum_{r=1}^mf(\bw^\ell_r(0)) - \ebb_{\bw\sim\N(0,\bfI)}\big[f(\bw)\big]\Big| &\!\le \!C_\sigma\!\sqrt{\frac{\log^2(m)}{m}} \!+\! \ebb_{\bW^\ell(0)}\Big[\sup_{f\in\F}\Big|\frac{1}{m}\sum_{r=1}^mf(\bw^\ell_r(0)) \!-\! \ebb_{\bw\sim\N(0,\bfI)}\big[f(\bw)\big]\Big|\Big] \\
        &\le C_\sigma\sqrt{\frac{\log^2(m)}{m}} + 2\ebb_{\bW^\ell(0), \boldsymbol{\epsilon}}\Big[\sup_{f\in\F}\Big|\frac{1}{m}\sum_{r=1}^m\epsilon_rf(\bw^\ell_r(0))\Big|\Big],
    \end{align*}
    where $\boldsymbol{\epsilon}$ is the Rademacher vector independent of $\bW^\ell(0)$ and the last inequality used the standard symmetrization methods in statistical learning theory (see \cite{vershynin2018high,wainwright2019high}).

    According to (5.48) in \cite{wainwright2019high} with $n = m$, $x_i = \bw^\ell_r(0)$ and $b = B_{\sigma'}^2$, the second term in the right hand side of the above inequality can be bounded as
    \begin{align*}
        2\ebb_{\bW^\ell(0), \boldsymbol{\epsilon}}\Big[\sup_{f\in\F}\Big|\frac{1}{m}\sum_{r=1}^m\epsilon_rf(\bw^\ell_r(0))\Big|\Big] \le \ebb_{\bW^\ell(0)}\Big[\frac{48}{\sqrt{m}}\int_0^{2B_{\sigma'}^2}\sqrt{\log(\N(\F,\|\cdot\|_n, t))}dt\Big],
    \end{align*}
    where the norm $\|\cdot\|_n$ is defined by $\|f\|_n^2 = \frac{1}{m}\sum_{r=1}^mf(\bw^\ell_r(0))^2$ for $f\in\F$ and $\N(\F,\|\cdot\|_n, t)$ is the covering number of $\F$ with norm $\|\cdot\|_n$ and radius $t$.
    
    It remains to estimate $\N(\F,\|\cdot\|_n, t)$.
    According to Bernstein inequality (Eq (3.1) in \cite{vershynin2018high} with $X = \bw^\ell_r(0)$, $n=m$, $u=1$), we know with probability at least $1-Lm\exp(-Cm)$ over initialization $\bW(0)$, it holds that $\|\bw^\ell_r(0)\|_2 \le 2\sqrt{m}$ for all $r\in[m]$ and $\ell\in[L]$.
    For any $f, g\in\F$ indexed by $(\bx,\bx')$ and $(\bx_0,\bx_0')$, respectively,
    it holds that
    \begin{align*}
        \|f - g\|_n^2 &= \frac{1}{m}\sum_{r=1}^m\big(f(\bw^\ell_r(0)) - g(\bw^\ell_r(0))\big)^2 \\
        &= \frac{1}{m}\sum_{r=1}^m\big|\sigma'\big(\bw^\ell_r(0)^\top o^{\ell-1}_0(\bx)\big)\sigma'\big(\bw^\ell_r(0)^\top o^{\ell-1}_0(\bx')\big) - \sigma'\big(\bw^\ell_r(0)^\top o^{\ell-1}_0(\bx_0)\big)\sigma'\big(\bw^\ell_r(0)^\top o^{\ell-1}_0(\bx_0')\big)\big|^2\\
        &\le \frac{1}{m}\sum_{r=1}^m \Big(\big|\sigma'\big(\bw^\ell_r(0)^\top o^{\ell-1}_0(\bx)\big) - \sigma'\big(\bw^\ell_r(0)^\top o^{\ell-1}_0(\bx_0)\big)\big|\cdot\big|\sigma'\big(\bw^\ell_r(0)^\top o^{\ell-1}_0(\bx')\big)\big|\\
        &\qquad \ \ + \big|\sigma'\big(\bw^\ell_r(0)^\top o^{\ell-1}_0(\bx_0)\big)\big|\cdot\big|\sigma'\big(\bw^\ell_r(0)^\top o^{\ell-1}_0(\bx')\big) - \sigma'\big(\bw^\ell_r(0)^\top o^{\ell-1}_0(\bx_0')\big)\big|\Big)^2\\
        &\le B_{\sigma'}^2B_{\sigma''}^2 \frac{1}{m}\sum_{r=1}^m \big(\big|\bw^\ell_r(0)^\top\big(o^{\ell-1}_0(\bx) - o^{\ell-1}_0(\bx_0)\big)\big| + \big|\bw^\ell_r(0)^\top\big(o^{\ell-1}_0(\bx') - o^{\ell-1}_0(\bx_0')\big)\big|\big)^2\\
        &\le B_{\sigma'}^2B_{\sigma''}m \big(\|o^{\ell-1}_0(\bx) - o^{\ell-1}_0(\bx_0)\|_2 + \|o^{\ell-1}_0(\bx') - o^{\ell-1}_0(\bx'_0)\|_2\big)^2\\
        &\le B_{\sigma'}^2B_{\sigma''}^2\big(\sqrt{c_\sigma}B_{\sigma'}c_0\big)^{2\ell-2}m\big(\|\bx - \bx_0\|_2 + \|\bx' - \bx_0'\|_2\big)^2,
    \end{align*}
    where the second inequality used Assumption~\ref{ass:activation} and mean value theorem, the last second inequality used $\|\bw^\ell_r(0)\|_2\le\sqrt{m}$ and the last inequality used \eqref{eq:lip_o}.
    Define the norm $\|\cdot\|_\otimes$ on $\X\times\X$ by $\|(\bx,\bx') - (\bx_0,\bx_0')\|_\otimes =C_\sigma^L\sqrt{m}(\|\bx - \bx_0\|_2 + \|\bx' - \bx_0'\|_2)$.
    The above inequality implies that the norm $\|\cdot\|_n$ on $\F$ is dominated by the norm $\|\cdot\|_\otimes$ on $\X\times\X$.
    Further, suppose $\A_1$ and $\A_2$ be $t/2$-net of a set $A$ equipped with norm $\|\cdot\|$, one can easily verify that $\A_1\times\A_2$ is a $t$-net of the product set $A\times A$ equipped with the norm $\|(a_1,a_2)\|:=\|a_1\|+\|a_2\|$.
    Then, it holds that
    \begin{align*}
        \N(\F,\|\cdot\|_n, t) \le \N(\X\times\X, \|\cdot\|_\otimes, t) \le \big(\N(\X, C_\sigma^L\sqrt{m}\|\cdot\|_2, t/2)\big)^2 \le \Big(\frac{C_\sigma^L\sqrt{m}}{t}\Big)^{2d},
    \end{align*}
    where the last inequality used Corollary 4.2.13 in \cite{vershynin2018high}.
    Combining the above estimates yields that
    \begin{align*}
       & \sup_{f\in\F}\Big|\frac{1}{m}\sum_{r=1}^mf(\bw^\ell_r(0)) - \ebb_{\bw\sim\N(0,\bfI)}\big[f(\bw)\big]\Big|  \le C_\sigma\sqrt{\frac{\log^2(m)}{m}} + \ebb_{\bW^\ell(0)}\Big[\frac{48}{\sqrt{m}}\int_0^{2B_{\sigma'}^2}\sqrt{\log(\N(\F,\|\cdot\|_n, t))}dt\Big]\\
        &\lesssim C_\sigma \Big(\sqrt{\frac{\log^2(m)}{m}} + \frac{1}{\sqrt{m}}\int_0^{2B_{\sigma'}^2} \sqrt{dL\log\Big(\frac{m}{t}\Big)}dt\Big) \lesssim C_\sigma \sqrt{\frac{dL\log^2(m)}{m}}.
    \end{align*}
    This completes the proof of the lemma.
\end{proof}

\begin{lemma}\label{lem:trace-prod}
    Suppose $m \gtrsim C_\sigma^L dL\log(m)$.
    With probability at least $1 - L^2\exp(-C_\sigma \log^2(m))$ over initialization $\bW(0)$, for all $\bx, \bx'\in\X$ and $\ell\in[L]$, we have
    \begin{align*}
        \Big|tr\big(\bV^\ell_{L,0}(\bx)^\top\bV^\ell_{L,0}(\bx')\big) - \prod_{h=\ell}^Lq^h(\bx,\bx')\Big| \lesssim C_\sigma^L\sqrt{\frac{d\log^2(m)}{m}}.
    \end{align*}
\end{lemma}
\begin{proof}
    Define $\bA^\ell_L(\bx,\bx'):= \bV^\ell_{L,0}(\bx)^\top\bV^\ell_{L,0}(\bx')$ for all $\ell\in[L]$ and $\bx,\bx'\in\X$.
    We first consider the case $\ell=L$.
    For any $\bx,\bx'\in\X$, from Lemma \ref{lem:sigma'-E[sigma']} we have
    \begin{align*}
        \big|&tr\big(\bA^L_L(\bx,\bx')\big) - q^L(\bx,\bx')\big| \!=\! c_\sigma\Big|\frac{1}{m}\sum_{r=1}^m\sigma'\big(\bw^L_r(0)^\top o^{L-1}_0(\bx)\big)\sigma'\big(\bw^L_r(0)^\top o^{L-1}_0(\bx')\big) - \ebb\big[\sigma'(U^{L}(\bx))\sigma'(U^{L}(\bx'))\big]\Big|\\
        &\le c_\sigma\Big|\frac{1}{m}\sum_{r=1}^m\sigma'\big(\bw^L_r(0)^\top o^{L-1}_0(\bx)\big)\sigma'\big(\bw^L_r(0)^\top o^{L-1}_0(\bx')\big) - \ebb_{\bw\sim\N(0,\bfI)}\big[\sigma'\big(\bw^\top o^{L-1}_0(\bx)\big)\sigma'\big(\bw^\top o^{L-1}_0(\bx')\big)\big]\Big|\\
        &\ \ \ + c_\sigma\Big|\ebb_{\bw\sim\N(0,\bfI)}\big[\sigma'\big(\bw^\top o^{L-1}_0(\bx)\big)\sigma'\big(\bw^\top o^{L-1}_0(\bx')\big)\big] - \ebb\big[\sigma'(U^{L}(\bx))\sigma'(U^{L}(\bx'))\big]\Big|\\
        &\lesssim C_\sigma \sqrt{\frac{dL\log^2(m)}{m}} + c_\sigma\Big|\ebb_{\bw\sim\N(0,\bfI)}\big[\sigma'\big(\bw^\top o^{L-1}_0(\bx)\big)\sigma'\big(\bw^\top o^{L-1}_0(\bx')\big)\big] - \ebb\big[\sigma'(U^{L}(\bx))\sigma'(U^{L}(\bx'))\big]\Big|.
    \end{align*}
    We turn to estimate the last term on the right hand side of the above inequality.
    Note $(\bw^\top o^{L-1}_0(\bx), \bw^\top o^{L-1}_0(\bx'))$ and $(U^{L}(\bx), U^{L}(\bx'))$ are bivariate normal variables with covariance matrices
    $$\begin{pmatrix}
        \|o^{L-1}_0(\bx)\|_2^2 & \langle o^{L-1}_0(\bx), o^{L-1}_0(\bx')\rangle\\
        \langle o^{L-1}_0(\bx), o^{L-1}_0(\bx')\rangle & \|o^{L-1}_0(\bx')\|_2^2
    \end{pmatrix}, c_\sigma\!\begin{pmatrix}
        \ebb[\sigma^2(U^{L-1}(\bx))] & \ebb[\sigma(U^{L-1}(\bx))\sigma(U^{L-1}(\bx'))]\\
        \ebb[\sigma(U^{L-1}(\bx))\sigma(U^{L-1}(\bx'))] & \ebb[\sigma^2(U^{L-1}(\bx'))]
    \end{pmatrix}.$$
    Lemma \ref{lem:o-U_uniform}, \eqref{eq:EU=1} and the condition $m\gtrsim dL\log(m)$ together imply that $\|o^{L-1}_0(\bx)\|_2^2 \ge c_\sigma\ebb[\sigma^2(U^{L-1}(\bx))] - |\|o^{L-1}_0(\bx)\|_2^2 - c_\sigma\ebb[\sigma^2(U^{L-1}(\bx))]| \ge 1 - 1/2 = 1/2$.
    Then, applying Lemma G.4 in \cite{du2019gradient} with $\sigma$ replaced by $\sigma'$ and $a_1=b_1 = 1/2$ and $a_2=b_2=1$, we have
    \begin{align*}
        &c_\sigma\Big|\ebb_{\bw\sim\N(0,\bfI)}\big[\sigma'\big(\bw^\top o^{L-1}_0(\bx)\big)\sigma'\big(\bw^\top o^{L-1}_0(\bx')\big)\big] - \ebb\big[\sigma'(U^{L}(\bx))\sigma'(U^{L}(\bx'))\big]\Big|\\
        &\le C_\sigma\sup_{\bx,\bx'\in\X}\big|\langle o^{L-1}_0(\bx), o^{L-1}_0(\bx')\rangle - \ebb[\sigma(U^{L-1}(\bx))\sigma(U^{L-1}(\bx'))]\big|\lesssim C_\sigma^L\sqrt{\frac{d\log^2(m)}{m}},
    \end{align*}
    where the last inequality used Lemma \ref{lem:o-U_uniform} with $\delta = \exp(-C_\sigma\log^2(m))$.
    Combining the above estimates yields
    \begin{align}\label{eq:trace-qL}
        \big|&tr\big(\bA^L_L(\bx,\bx')\big) - q^L(\bx,\bx')\big| \lesssim C_\sigma^L\sqrt{\frac{d\log^2(m)}{m}}.
    \end{align}
    Note that this inequality holds true when $L$ is replaced by $\ell$ for any $\ell\in[L]$.

    Now, we consider the case $\ell<L$.
    From  definitions we know
    \begin{align}\label{eq:trace-prod}
        &\Big|tr\big(\bA^\ell_L(\bx,\bx')\big) - \prod_{h=\ell}^Lq^h(\bx,\bx')\Big|\nonumber\\
        &= \Big|\frac{c_\sigma}{m}\sum_{r=1}^m\sigma'\big(\bw^L_r(0)^\top o^{L-1}_0(\bx)\big)\sigma'\big(\bw^L_r(0)^\top o^{L-1}_0(\bx')\big)\bw^L_r(0)^\top \bA^\ell_{L-1}(\bx,\bx')\bw^L_r(0) - \prod_{h=\ell}^Lq^h(\bx,\bx')\Big|\nonumber\\
        &\le \Big|\frac{c_\sigma}{m}\sum_{r=1}^m \sigma'\big(\bw^L_r(0)^\top o^{L-1}_0(\bx)\big)\sigma'\big(\bw^L_r(0)^\top o^{L-1}_0(\bx')\big)\Big(\bw^L_r(0)^\top \bA^\ell_{L-1}(\bx,\bx')\bw^L_r(0) - tr\big(\bA^\ell_{L-1}(\bx,\bx')\big)\Big)\Big|\nonumber\\
        &+ \Big|\Big(\frac{c_\sigma}{m}\sum_{r=1}^m \sigma'\big(\bw^L_r(0)^\top o^{L-1}_0(\bx)\big)\sigma'\big(\bw^L_r(0)^\top o^{L-1}_0(\bx')\big) - q^L(\bx,\bx')\Big)tr\big(\bA^\ell_{L-1}(\bx,\bx')\big)\Big|\nonumber\\
        & + \Big|q^L(\bx,\bx')\Big(tr\big(\bA^\ell_{L-1}(\bx,\bx')\big) - \prod_{h=\ell}^{L-1}q^h(\bx,\bx')\Big)\Big| \nonumber\\
        &=: I_1 + I_2 + I_3, 
    \end{align}
    where we denote $I_1,I_2,I_3$ as the three terms on the right hand side of the above inequality.

    According to Assumption~\ref{ass:activation} and Lemma \ref{lem:quadratic-trace}, it holds that
    \begin{align*}
        I_1 \le c_\sigma B_{\sigma'}^2\sup_{r\in[m]}\Big|\bw^L_r(0)^\top \bA^\ell_{L-1}(\bx,\bx')\bw^L_r(0) - tr\big(\bA^\ell_{L-1}(\bx,\bx')\big)\Big| \lesssim C_\sigma^L\sqrt{\frac{d\log^2(m)}{m}}.
    \end{align*}
    
    From Lemma \ref{lem:bound_V0} and \eqref{eq:trace-qL} we know
    \begin{align*}
        I_2 &= \big|tr\big(\bA^L_L(\bx,\bx')\big) - q^L(\bx,\bx')\big|\cdot\big|tr\big(\bA^\ell_{L-1}(\bx,\bx')\big)\big|\\& = \big|tr\big(\bA^L_L(\bx,\bx')\big) - q^L(\bx,\bx')\big|\cdot\big|\ebb_{\bw\sim\N(0,\bfI)}\big[\bw^\top\bA^\ell_{L-1}(\bx,\bx')\bw\big]\big|\\
        &\lesssim C_\sigma^L\sqrt{\frac{d\log^2(m)}{m}}\big\|\bA^\ell_{L-1}(\bx,\bx')\big\|_{op}\ebb_{\bw\sim\N(0,\bfI)}\big[\|\bw\|^2_2\big] \le  C_\sigma^L\sqrt{\frac{d\log^2(m)}{m}}\big\|\bV^{\ell}_{L-1,0}(\bx)\big\|_{op}\big\|\bV^{\ell}_{L-1,0}(\bx')\big\|_{op}m\\
        &\lesssim C_\sigma^L\sqrt{\frac{d\log^2(m)}{m}}.
    \end{align*}
    According to Assumption~\ref{ass:activation} and the definition of $q^L(\bx,\bx')$ (see Appendix~\ref{appde:problem}), we know $I_3\le c_\sigma B_{\sigma'}^2|tr(\bA^\ell_{L-1}(\bx,\bx')) - \prod_{h=\ell}^{L-1}q^h(\bx,\bx')|$.
    Plugging the estimates of $I_1,I_2,I_3$ back into \eqref{eq:trace-prod}, it holds that
    \begin{align*}
        \Big|tr\big(\bA^\ell_L(\bx,\bx')\big) - \prod_{h=\ell}^Lq^h(\bx,\bx')\Big| \lesssim C_\sigma^L\sqrt{\frac{d\log^2(m)}{m}} + c_\sigma B_{\sigma'}^2\Big|tr\big(\bA^\ell_{L-1}(\bx,\bx')\big) - \prod_{h=\ell}^{L-1}q^h(\bx,\bx')\Big|.
    \end{align*}
    Applying the above inequality recursively, we know
    \begin{align*}
        \Big|tr\big(\bA^\ell_L(\bx,\bx')\big) - \prod_{h=\ell}^Lq^h(\bx,\bx')\Big| &\lesssim \Big[\sum_{k=0}^{L-\ell-1}\big(c_\sigma B_{\sigma'}^2\big)^k\Big]C_{\sigma}^L\sqrt{\frac{d\log^2(m)}{m}} + \big(c_\sigma B_{\sigma'}^2\big)^{L-\ell}\big|tr\big(\bA^\ell_\ell(\bx,\bx')\big) - q^\ell(\bx,\bx')\big|\\
        &\lesssim C_{\sigma}^L\sqrt{\frac{d\log^2(m)}{m}},
    \end{align*}
    where the last inequality used \eqref{eq:trace-qL} with $L$ replaced by $\ell$.
    This completes the proof of the lemma.
\end{proof}

Now, we give the proof of Theorem~\ref{thm:re-concen} (i.e., Theorem~\ref{thm:concentration}) as follows.

\begin{proof}[Proof of Theorem~\ref{thm:re-concen}]
For all $\ell\in[L]$ and $\bx,\bx'\in\X$, from Assumption~\ref{ass:activation} we know \[|\langle o^{\ell-1}_0(\bx),  o^{\ell-1}_0(\bx')\rangle| \le c_\sigma\sup_{t\in\R}\sigma^2(t) \le c_\sigma B_\sigma^2.\]
According to Lemma \ref{lem:bound_V0}, we know \[tr(\bV^\ell_{L,0}(\bx)^\top\bV^\ell_{L,0}(\bx')) = \ebb_\ba[\ba^\top\bV^\ell_{L,0}(\bx)^\top\bV^\ell_{L,0}(\bx')\ba] \le \|\bV^{\ell}_{L,0}(\bx)\|_{op}\|\bV^{\ell}_{L,0}(\bx')\|_{op}\ebb_\ba[\|\ba\|_2^2] \le C_\sigma^L.\]
Applying \eqref{eq:EU=1} and Cauchy-Schwarz inequality, it holds that \[c_\sigma\ebb[\sigma(U^{\ell-1}(\bx))\sigma(U^{\ell-1}(\bx'))] \le \sqrt{c_\sigma\ebb[\sigma^2(U^{\ell-1}(\bx))}\sqrt{c_\sigma\ebb[\sigma^2(U^{\ell-1}(\bx'))} = 1.\]

Combining the above estimates, \eqref{eq:EU=1}, and Lemmas \ref{lem:o-U_uniform}, \ref{lem:quadratic-trace}, \ref{lem:trace-prod} yields that
\begin{align*}
    &\big\|K^{m,\ell} - K^\ell\big\|_\infty\\
    &= \sup_{\bx,\bx'\in\X}\Big|\langle o^{\ell-1}_0(\bx),  o^{\ell-1}_0(\bx')\rangle\ba^\top\bV^\ell_{L,0}(\bx)^\top\bV^\ell_{L,0}(\bx')\ba - c_\sigma\ebb[\sigma(U^{\ell-1}(\bx))\sigma(U^{\ell-1}(\bx'))] \prod_{h=\ell}^Lq^h(\bx,\bx') \Big| \nonumber\\
    &\le \sup_{\bx,\bx'\in\X}\big|\langle o^{\ell-1}_0(\bx), o^{\ell-1}_0(\bx')\rangle\big| \cdot \Big|\ba^\top\bV^\ell_{L,0}(\bx)^\top\bV^\ell_{L,0}(\bx')\ba - tr\big(\bV^\ell_{L,0}(\bx)^\top\bV^\ell_{L,0}(\bx')\big)\Big|\nonumber\\
    & + \sup_{\bx,\bx'\in\X}\big|\langle o^{\ell-1}_0(\bx),  o^{\ell-1}_0(\bx')\rangle - c_\sigma\ebb[\sigma(U^{\ell-1}(\bx))\sigma(U^{\ell-1}(\bx'))]\big| \cdot \big|tr\big(\bV^\ell_{L,0}(\bx)^\top\bV^\ell_{L,0}(\bx')\big)\big|\nonumber\\
    &  + \sup_{\bx,\bx'\in\X}\big|c_\sigma\ebb[\sigma(U^{\ell-1}(\bx))\sigma(U^{\ell-1}(\bx'))]\big| \cdot \Big|tr\big(\bV^\ell_{L,0}(\bx)^\top\bV^\ell_{L,0}(\bx')\big) - \prod_{h=\ell}^Lq^h(\bx,\bx')\Big|\nonumber\\
    &\lesssim   C_\sigma\sup_{\bx,\bx'\in\X} \Big|\ba^\top\bV^\ell_{L,0}(\bx)^\top\bV^\ell_{L,0}(\bx')\ba - tr\big(\bV^\ell_{L,0}(\bx)^\top\bV^\ell_{L,0}(\bx')\big)\Big|\nonumber\\
    &  +  C_\sigma^L\!\!\sup_{\bx,\bx'\in\X}\big|\langle o^{\ell-1}_0(\bx), o^{\ell-1}_0(\bx')\rangle - \ebb[\sigma(U^{\ell-1}(\bx))\sigma(U^{\ell-1}(\bx'))]\big| + \!\!\sup_{\bx,\bx'\in\X} \Big|tr\big(\bV^\ell_{L,0}(\bx)^\top\bV^\ell_{L,0}(\bx')\big) - \prod_{h=\ell}^Lq^h(\bx,\bx')\Big|\nonumber\\
    &\lesssim  C_\sigma^L\sqrt{\frac{d\log^2(m)}{m}}.
\end{align*}
This completes the proof of this theorem. 
\end{proof}

\end{document}